\documentclass[journal]{article}

\usepackage[a4paper,lmargin=2.9cm,rmargin=2.9cm]{geometry}

\usepackage{authblk}
\usepackage{amsmath,amsthm,mathtools}
\newtheorem{theorem}{Theorem}
\newtheorem{definition}[theorem]{Definition}
\newtheorem{lemma}[theorem]{Lemma}
\newtheorem{proposition}[theorem]{Proposition}

\newcommand\ignore[1]{}

\usepackage[utf8]{inputenc}
\usepackage[ruled,vlined,linesnumbered,nofillcomment]{algorithm2e}

\SetCommentSty{mycommfont}
\SetKw{Wprob}{with probability}
\SetKw{Do}{do}
\SetKwInOut{Input}{input}
\usepackage{xspace}
\usepackage{todonotes}

\usepackage{tikz}
\usetikzlibrary{calc,patterns,positioning,snakes,arrows,shapes,arrows.meta,graphs,graphs.standard}
\usepackage{pgfplots}
\usepackage{pgfplotstable}
\usepgfplotslibrary{fillbetween,statistics}
\pgfplotsset{compat=1.16}

\usepackage{caption,subcaption}
\usepackage{url}

\usepackage{dsfont}
\def\R{\mathds{R}}
\def\N{\mathds{N}}
\DeclareMathOperator{\E}{E}
\DeclareMathOperator*{\argmax}{arg\,max}
\DeclareMathOperator{\poly}{poly}

\newcommand\ea[1]{(1+1)~EA$^{\textsc{j}+\textsc{r}}_{#1}$\xspace}
\def\jar{\mbox{\textsc{JumpAndRepair}}\xspace}
\def\jarvc{\mbox{\textsc{JumpAndRepairVC}}\xspace}
\def\jarfvst{\mbox{\textsc{JumpAndRepairFVST}}\xspace}
\def\jaroct{\mbox{\textsc{JumpAndRepairOCT}}\xspace}
\def\topo{\mbox{\textsc{Sort}}\xspace}

\tikzstyle{vertex}=[circle,minimum size=6mm,fill=white,draw]
\tikzstyle{smallvertex}=[vertex,minimum size=2mm]
\tikzstyle{arc}=[-{Latex[width=1.5mm]}]

\pgfplotsset{
    discard if/.style 2 args={
        x filter/.code={
            \edef\tempa{\thisrow{#1}}
            \edef\tempb{#2}
            \ifx\tempa\tempb
              \def\pgfmathresult{inf}
              \typeout{XXX}
            \fi
        }
    }
}
  
\allowdisplaybreaks[3]

\title{Focused Jump-and-Repair Constraint Handling for Fixed-Parameter Tractable Graph Problems Closed Under Induced Subgraphs\footnote{A preliminary version
      of this paper appeared in the proceedings of FOGA~2021~\cite{BransonSutton2021jar}.}}

\author{Luke Branson}
\author{Andrew M.\ Sutton}
\affil{Department of Computer Science\\
   University of Minnesota Duluth}
 \date{}
 
\begin{document}
\maketitle

\begin{abstract}
   Repair operators are often used for constraint handling in constrained combinatorial optimization. We investigate the (1+1)~EA equipped with a tailored jump-and-repair operation that can be used to probabilistically repair infeasible offspring in graph problems. Instead of evolving candidate solutions to the entire graph, we expand the genotype to allow the (1+1)~EA to develop in parallel a feasible solution together with a growing subset of the instance (an induced subgraph). With this approach, we prove that the EA is able to probabilistically simulate an \emph{iterative compression} process used in classical fixed-parameter algorithmics to obtain a randomized FPT performance guarantee on three $\mathsf{NP}$-hard graph problems. For $k$-\textsc{VertexCover}, we prove that the (1+1) EA using focused jump-and-repair can find a $k$-vertex cover (if one exists) in $O(2^k n^2\log n)$ iterations in expectation. This leads to an exponential (in $k$) improvement over the best-known parameterized bound for evolutionary algorithms on \textsc{VertexCover}. For the $k$-\textsc{FeedbackVertexSet} problem in tournaments, we prove that the EA finds a feasible feedback set in $O(2^kk!n^2\log n)$ iterations in expectation, and for \textsc{OddCycleTransversal}, we prove the optimization time for the EA is $O(3^k k m n^2\log n)$. For the latter two problems, this constitutes the first parameterized result for any evolutionary algorithm. We discuss how to generalize the framework to other parameterized graph problems closed under induced subgraphs and report experimental results that illustrate the behavior of the algorithm on a concrete instance class.
\end{abstract}

\section{Introduction}
In many constrained combinatorial optimization problems, infeasible solutions are easy to correct by using a problem-tailored \emph{repair} operator~\cite{Coello2002survey}. Such an operator can be useful in the context of evolutionary computation, as an EA can be free to generate infeasible offspring which can promptly be made feasible again at comparatively low computational cost. Repair operators are one of the main approaches to handling infeasibility in constrained problems~\cite{Coello2002survey,ZydallisLamont2001repair,Salcedo_Sanz_2009}, even though they demand a degree of problem dependence.

In this paper we consider a simple evolutionary algorithm, the (1+1)~EA, equipped with an extra operation to probabilistically repair infeasible offspring. In particular, when mutation produces an infeasible offspring, we attempt to repair it by executing a small, focused ``jump'' in the space by deleting some elements of the solution and then calling a polynomial-time repair mechanism.

We apply this evolutionary algorithm to certain types of $\mathsf{NP}$-hard graph problems
that are closed under induced subgraphs, in particular,
$k$-\textsc{VertexCover} and $k$-\textsc{OddCycle\-Transversal} in undirected graphs and
$k$-\textsc{FeedbackVertexSet} in tournaments. Rather than considering
a candidate solution to the entire graph, we allow the EA to build up
both a solution and an induced subgraph in parallel. Using this
approach, we prove that the EA is able to probabilistically simulate
an \emph{iterative compression} process, which comes from classical
fixed-parameter algorithmics, in order to obtain a randomized FPT
performance guarantee on these three problems. For
$k$-\textsc{VertexCover}, we prove that the (1+1) EA using focused
jump-and-repair can find a $k$-vertex cover (if one exists) in
$O(2^k n^2\log n)$ iterations in expectation. When ignoring polynomial
factors, this leads to an exponential (in $k$) improvement over the
best-known parameterized bound for evolutionary algorithms on
\textsc{VertexCover} given by Kratsch and
Neumann~\cite{KratschNeumann2012}. Our presented EA also does not
require the optimization of multiple objectives. For the
$k$-\textsc{FeedbackVertexSet} problem in tournaments, we prove that the EA
finds a feasible feedback set in $O(2^kk!n^2\log n)$ iterations in
expectation, and for $k$-\textsc{OddCycleTransversal}, we prove the EA requires at most $O(3^k k m n^2 \log n)$ time. To our knowledge, this is the first run time analysis of an evolutionary algorithm on both \textsc{FeedbackVertexSet} and \textsc{OddCycleTransversal}.

\subsection{Background}
Techniques for constraint-handling in evolutionary algorithms include penalty functions, multi-objective optimization in both fitness and constraint space, special operators (such as random keys) and repair mechanisms~\cite{Coello2002survey}. Repair mechanisms take an infeasible solution produced by mutation or crossover and apply some kind of transformation to produce a new feasible solution.

Early work on repair mechanisms in the context of evolutionary computation attempted to reconcile the use of standard bit string representations for more restrictive combinatorial structures. For example, Hamiltonian cycles in a graph can be easily encoded as bit strings, but classical mutation and crossover operators are very unlikely to result in offspring that are legal Hamiltonian cycles. This was an obstacle for traditional genetic algorithms tasked to solve problems over permutations such as the Traveling Salesperson Problem (TSP), and one proposed expedient was to repair infeasible offspring by applying a greedy algorithm~\cite{Lidd1991tsp}. Repair-based crossover operators that work directly on permutation encodings have also been investigated~\cite{goldberg1985alleles,Muehlenbein1992parallel,GeneRepair2003}. For cardinality constraints on bit strings, the \emph{genetic fix} and \emph{restricted search} approaches ensure that offspring always have fixed Hamming weight after crossover~\cite{Salcedo_Sanz_2009}.

Repair techniques for minimum vertex cover were examined empirically for hierarchical Bayesian optimization (hBOA) and the simple genetic algorithm (SGA)~\cite{PelikanKH07}. The repair mechanism employed was a local search algorithm that transformed infeasible offspring into valid vertex covers. The authors found that these approaches (along with simulated annealing) discovered optimal vertex covers on Erd\H{o}s-R\'{e}nyi random graphs significantly faster than complete branch-and-bound search.

One drawback to repair mechanisms (including the one we present in
this paper) is that they are often problem-dependent, and require
domain-specific knowledge~\cite{Coello2002survey}. Of course, it
should be no surprise that extra domain knowledge can positively
influence the efficiency of an evolutionary algorithm. For example, in
the context of \textsc{VertexCover}, He, Yao and Li~\cite{HeYaoLi2005}
empirically demonstrated that allowing the mutation probability for
each bit to depend on the corresponding vertex degree resulted in a
significant performance gain over a pure black-box EA (i.e., access to
the instance only via the fitness function).  High-performance local
search algorithms such as NuMVC~\cite{DBLP:journals/jair/CaiSLS13} and
FastVC~\cite{DBLP:conf/ijcai/Cai15} also take into account this kind
of domain knowledge to iteratively discard vertices from a candidate
cover and probabilistically repair the resulting infeasible solution.

Rather than explicitly repairing infeasible covers, Khuri and B\"{a}ck designed a fitness function for the minimum vertex cover problem that incorporated a penalty term~\cite{Khuri94anevolutionary}. This fitness function guarantees that infeasible covers always have an inferior fitness to the worst feasible covers. Using this approach, the authors found that an evolutionary algorithm significantly outperforms the standard 2-approximation algorithm for vertex cover based on maximal matchings (often called \emph{vercov} in the EC community) both on
Erd\H{o}s-R\'{e}nyi random graphs and a class of structured graphs introduced by Papadimitriou and Steiglitz~\cite{PapadimitriouS82}.

This superior empirical performance of an EA on the minimum vertex cover problem prompted Oliveto, He and Yao to conduct a theoretical investigation~\cite{OlivetoHe2009vc}. They proved that when the maximum degree of the graph is bounded above by two, the (1+1)~EA can find the minimum vertex cover in expected time $O(n^4)$. However, the (1+1)~EA can easily get trapped on relatively simple instances. Friedrich et al.~\cite{FHNHW2007approx} proved that the expected time until the (1+1)~EA can find even a $(1-\epsilon)/\epsilon$ approximation on the complete bipartite graph $K_{\epsilon n, (1-\epsilon)n}$ is exponential. Oliveto, He and Yao~\cite{OlivetoHe2009vc} refined this picture slightly by identifying a tail bound: on $K_{\epsilon n, (1-\epsilon)n}$ (and the above-mentioned Papadimitriou-Steiglitz graphs), the running time is $O(n \log n)$ with asymptotically constant probability. This means that the pathological local optima contained in these instances can be overcome with a simple restart strategy.
However, the authors also prove that on
instances constructed as a chain of several copies of a complete
bipartite graph, the (1+1)~EA, even when equipped with a restart
strategy, cannot find a solution that lies within an approximation
factor slightly less than two. This constructive proof provides a
lower bound on the worst-case polynomial time approximation ratio for
the (1+1)~EA using vertex-based representations.
Minimum vertex-cover is likely hard to approximate below a
$2-\epsilon$ factor~\cite{KhotRegev2008Vertexcovermight}, and there
are currently no corresponding upper bounds on the approximation ratio
for the (1+1)~EA with vertex-based representations. An approximation
ratio of exactly two \emph{can} be guaranteed, however, using
edge-based representations~\cite{DBLP:conf/foga/JansenOZ13}.

Another method of constraint-handling for EAs is to incorporate the objective function and the constraint penalty into a multiobjective optimization problem~\cite{Coello2002survey}. This approach was applied to \textsc{VertexCover} by Friedrich et al.~\cite{FHNHW2007approx} who employed the fitness function $f\colon \{0,1\}^n \to \N^2 = x \mapsto (|x|,u(x))$ where $u(x)$ counts the number of edges in the graph that are not covered by the set selected by $x$. The goal is to simultaneously minimize both objectives. They proved that the multiobjective algorithm Global SEMO can optimize $K_{\epsilon n, (1-\epsilon)n}$ in $O(n^2 \log n)$ expected time.

This multiobjective optimization approach for \textsc{VertexCover} was also employed by Kratsch and Neumann~\cite{KratschNeumann2012} who presented the first fixed-parameter tractable evolutionary algorithm for a combinatorial optimization problem. Using the above biobjective fitness function along with a tailored mutation operator that concentrates mutation probability on endpoints of uncovered edges, they proved that Global SEMO has expected optimization time $O(OPT\cdot n^4 + n\cdot 2^{OPT^2+OPT})$ on any graph $G$ where $OPT$ is the size of the optimal vertex cover of $G$. When the count $u(x)$ of uncovered edges is replaced with the cost of an optimal fractional vertex cover (e.g., obtained via linear programming), the performance is improved to $O(n^2\log n + OPT\cdot n^2 + 4^{OPT}n)$.

\section{Focused Jump-and-Repair}
We consider parameterized graph problems in which we are given a (directed or undirected) graph $G = (V,E)$ and a natural number $k$, and the goal is to find a set $S \subseteq V$ with $|S| \le k$ such that $S$ is in some sense \emph{feasible} with respect to $G$. Examples of this kind of problem are when we insist each edge of $E$ is incident to at least one vertex in $S$ ($k$-\textsc{VertexCover}), that each vertex is adjacent to at least one vertex in $S$ ($k$-\textsc{DominatingSet}), or when the graph obtained by deleting $S$ is cycle-free ($k$-\textsc{FeedbackVertexSet}) or bipartite ($k$-\textsc{OddCycleTransversal}).

Given a subset $V' \subseteq V$, the \emph{induced subgraph} of $G$ with respect to $V'$, denoted $G[V']$, is the graph obtained from $G$ by including only vertices from $V'$ and edges from $E$ with both endpoints in $V'$. Note that feasible solutions of the above listed problems are closed under induced subgraphs, i.e., if $S$ is feasible with respect to $G$, then for any $V' \subseteq V$, $S \cap V'$ is feasible with respect to $G[V']$. The (open) neighborhood of a vertex $v \in V$ in $G$ is defined as the set $N(v) = \{u : (v,u) \in E\}$. The neighborhood of a set $S \subseteq V$ is denoted by $N(S) = \bigcup_{v \in S} N(v)$.

Evolutionary algorithms tasked with solving problems like \textsc{VertexCover} typically operate by representing a solution as a bitstring in $\{0,1\}^n$ where $|V|=n$, which selects elements from $V$ to include in $S$. The task is then to \emph{minimize} the Hamming weight of the bitstring (and hence the cardinality of the chosen set) subject to the constraint that it must be a feasible solution, e.g., a valid vertex cover. The hope is that the algorithm would eventually find a feasible solution of Hamming weight $k$.

In this paper, we will take a different perspective to parameterized graph problems by two new insights
\begin{enumerate}
   \item\label{item:1} We expand the search space to both candidate solutions $S$ and induced subgraphs of $G$.
   \item\label{item:2} We employ a focused jump-and-repair step that has the potential to efficiently repair an offspring made infeasible by mutation.
\end{enumerate}

Thus, instead of searching for $S$-sets in $G$ where $|S| \le k$, we design the EA to search for induced subgraphs of $G$ such that $|S| \le k$ and $S$ is \emph{already a feasible solution} for the subgraph. The fitness of a solution is the size of the induced subgraph. In this way, we systematically build up induced subgraphs together with already-feasible sets until the entire graph is recovered. We argue that this approach is useful on parameterized graph problems that are closed under induced subgraphs, such as the ones we investigate in this paper.

To characterize both a set $S$ and an induced subgraph of $G$, we define the search space over $\{0,1\}^{2n}$ so that a candidate length-$2n$ bit string $x$ has a length-$n$ prefix corresponding to vertices selected for the solution set and a length-$n$ suffix corresponding to vertices selected for the subgraph. It is convenient from an analytical perspective to factor such a string into an ordered pair $x \coloneqq (x_S,x_V)$ where $x_S$ corresponds to a candidate solution set in the induced subgraph $G[x_V]$\footnote{Here, and throughout the paper, we will often abuse notation by directly interpreting length-$n$ bitstrings as sets on $n$ elements and vice-versa.}. Ignoring vertices outside the induced subgraph, given a parameterized graph problem, a candidate string $x = (x_S,x_V)$ has two feasibility constraints:
\begin{description}
   \item[Solution constraint:] $x_S$ must be a valid solution for $G[x_V]$.
   \item[Cardinality constraint:] $|x_S | \le k$.
\end{description}
It is therefore convenient to use the following definitions.
\begin{definition}
  \label{def:feasibility}
   We say a string $x = (x_S,x_V)$ is \emph{solution feasible} for $G$
   when $x_S$ is a valid solution for the induced subgraph $G[x_V]$,
   that is, the vertex set $S = x_S \cap x_V$ is feasible in the
   induced subgraph. We say $x$ is \emph{cardinality feasible} when
   $|x_S| \le k$. We say a $x$ is \emph{infeasible} if at least one of these properties is violated.
\end{definition}

We define a general fitness function for a parameterized graph problems as follows:
\begin{equation}
   \label{eq:fitness}
   f_k(x) = \begin{cases}
      -(|x_S|+|x_V|) & \text{if $x$ is infeasible,} \\
      |x_V|          & \text{otherwise.}
   \end{cases}
\end{equation}
The function $f_k$ is designed to penalize infeasible solutions by attaining a negative value that pressures solutions toward smaller sets and graphs, while feasible solutions are rewarded for their size. We define the (1+1)~EA equipped with an additional \textsc{JumpAndRepair} operation as \ea{k} in Algorithm~\ref{alg:eak}.

\begin{algorithm}
   \SetKwInOut{Input}{input}
   \SetKwIF{Wp}{ElseIf}{Else}{with probability}{do}{}{else}{}
   \SetKwFor{RepeatFor}{repeat}{}{}
   \Input{A graph $G$ on $n$ vertices}
   \BlankLine
   choose $x$ from $\{0,1\}^{2n}$ uniformly at random\;   
   \While{$f_k(x) < n$}{      
      Obtain $y$ from $x$ by flipping each bit with probability $\frac{1}{2n}$\;
      \If{$y$ is solution feasible but not cardinality feasible}{
        $y'\gets \jar(y,G)$\label{li:ea-jar}\;
        $y \gets \argmax_{\{y,y'\}} f_k$\label{li:protect}\;
      }
      \lIf{$f_k(y) \ge f_k(x)$}{$x \gets y$}
   }
   \caption{\label{alg:eak} \protect\ea{k}}
\end{algorithm}

This outlines a rather general framework for applying an evolutionary algorithm to a parameterized graph problem. Apart from the fitness function, there is only one additional module that requires problem-specific knowledge to attempt to repair solutions. Moreover, there is also a just-in-time appeal to this approach, as at any point during the execution of the \ea{k}, a feasible solution is a valid $k$-solution to some induced subgraph.

It remains to define the focused jump-and-repair operation, which is the main innovation of the paper. We start by giving a general framework for the operation in Algorithm~\ref{alg:jump-and-repair}.

\begin{algorithm}
   \Input{A pair of strings $x = (x_S,x_V)$ and a graph $G$}
   \BlankLine
   \tcc{Jump}
   $S' \gets \emptyset$\;
   \For{$i \in \{ 1,\ldots,n : x_S[i] = 1 \}$}{
   \Wprob $1/2$ \Do $S' \gets S' \cup \{i\}$\;\label{li:jump}
   }
   \lIf{$S'$ is solution feasible for $G[x_V]$}{\Return $(x_{S'},x_V)$}
   \tcc{Repair}
   Find a minimal set $T \subseteq x_V \setminus x_S$ s.t. $S' \cup T$ is solution feasible for $G[x_V]$\;\label{li:repair}
   \Return $(x_{S' \cup T},x_V)$\;
   \caption{\label{alg:jump-and-repair} \jar}
\end{algorithm}

The jump-and-repair operation takes a solution-feasible offspring (which is not necessarily cardinality feasible) and executes a focused jump by selecting a random subset of the elements of $x_S$ (in line~\ref{li:jump}) and then calling a problem-specific repair procedure if the set of selected elements is not solution feasible (in line~\ref{li:repair}). A jump-and-repair operation can be successful by kicking out enough elements so that the repaired string is both cardinality and solution feasible as illustrated in Figure~\ref{fig:feasible}.

\begin{figure}
   \centering
   \begin{tikzpicture}[scale=0.85]
      \draw[fill=blue!20,fill opacity=0.2] plot[smooth cycle,tension=0.7]
      coordinates{(-4,2.5) (-3,3) (-2,2.8) (-0.8,2.5) (-0.5,1.5) (0.5,0) (0,-2)(-1.5,-2.5) (-4,-2) (-3.5,-0.5) (-5,1)};
      \draw[fill=red!20,fill opacity=0.2] plot[smooth cycle,tension=0.7]
      coordinates {(-1.2,1) (0,2.5) (0.8,1.7)  (2.5,1.9) (3.5,-1) (3,-2.5) (0,-2.2) (-1.5,-1.5) (-1.6,0)};

      \begin{scope}
         \clip plot[smooth cycle,tension=0.7]
         coordinates{(-4,2.5) (-3,3) (-2,2.8) (-0.8,2.5) (-0.5,1.5) (0.5,0) (0,-2)(-1.5,-2.5) (-4,-2) (-3.5,-0.5) (-5,1)};
         \draw[fill=yellow!20] plot[smooth cycle,tension=0.7]
         coordinates {(-1.2,1) (0,2.5) (0.8,1.7)  (2.5,1.9) (3.5,-1) (3,-2.5) (0,-2.2) (-1.5,-1.5) (-1.6,0)};
      \end{scope}
      \coordinate (solfeas) at (-3.6,1);
      \coordinate (feasible) at (-0.2,-1);
      \coordinate (crdfeas) at (2.2,1.4);

      \node[circle,inner sep=2pt,fill] (feasible-node) at (feasible) {};
      \node[circle,inner sep=2pt,fill,label={below:offspring}] (solfeas-node) at (solfeas) {~};
      \node[circle,inner sep=2pt,fill] (crdfeas-node) at (crdfeas) {~};

      \draw[-{Latex[width=2mm]},very thick,dashed] (solfeas-node) to[out=70,in=120] node[above,midway] {\it focused jump} (crdfeas-node);

      \draw[-{Latex[width=2mm]}] (crdfeas-node) to[out=-90,in=0] node[midway,below right] {\it repair} (feasible-node);

      \node[align=center,anchor=north] at (-2.5,-2.8) {\bf solution-feasible\\ \bf region};
      \node[align=center,anchor=north] at (1.8,-2.8) {\bf cardinality-feasible\\ \bf region};
      \node[align=center] at (-0.5,0) {\bf feasible\\\bf region};

   \end{tikzpicture}
   \caption[]{\label{fig:feasible} A successful jump-and-repair operation transforms a cardinality-infeasible offspring by focusing on a random subset of $x_S$ and applying a repair operation on the result to make it solution feasible again.}
\end{figure}
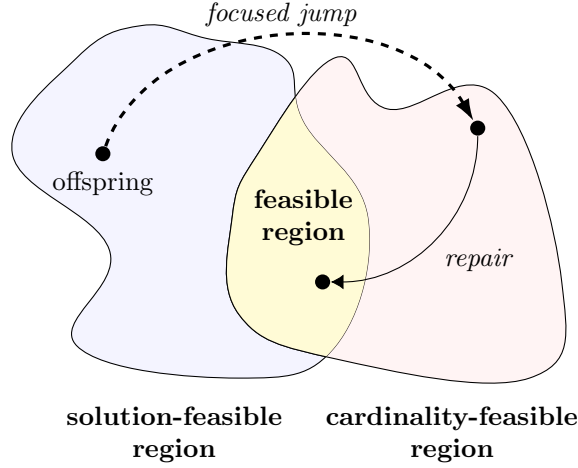

The jump operation is somewhat inspired by the \emph{alternative mutation} operator employed for Global SEMO on the minimum \textsc{VertexCover} problem~\cite{KratschNeumann2012} which flips, with probability $1/2$, all vertices incident to an uncovered edge. The goal is to focus the search on components of the solution that need attention. In the case of the \ea{k}, rather than focusing on elements that have not yet been included in the solution, our focused jump tries to find promising subsets of the solution which can then be repaired.

The actual repair operation must be problem-tailored, and we will later define explicit probabilistic repair procedures for different parameterized graph problems. We stress here that the repair procedure must run in time polynomial in $n$ for a meaningful FPT result (though, strictly speaking, it could more generally run in FPT time).

We will also make use of the following well-known result, stated here for completeness.
\begin{theorem}[Multiplicative Drift~\cite{DoerrJW2012,DoerrGoldberg2010}]
   \label{thm:multidrift}
   Let $(X_t)_{t \ge 0}$ be a sequence of nonnegative random variables with a finite state space $\mathcal{S} \subseteq \R^+_0$ such that $0 \in \mathcal{S}$. Let $s_{\rm min} \coloneqq \min(\mathcal{S}\setminus\{0\})$, let $T \coloneqq \inf\{t \ge 0 \mid X_t = 0\}$. If there exists $\delta > 0$ such that for all $s \in \mathcal{S} \setminus \{0\} \text{~and~} t \in \N$, $\E[X_t - X_{t+1} \mid X_t = s] \ge \delta s$, then
   \[
      \E[T] \le \frac{1 + \E[\ln (X_0/s_{\rm min})]}{\delta}.
   \]
   Moreover, for all $r \ge 0$, if $X_0 = s_0$,
   \[
      \Pr\left(T > \left\lceil \frac{r+\ln(s_0/s_{\rm min})}{\delta}\right\rceil\right) \le e^{-r}.
   \]
\end{theorem}

According to Equation~\eqref{eq:fitness}, any infeasible solution has a negative fitness, whereas every feasible solution has nonnegative fitness. This immediately yields the following lemma.
\begin{lemma}
   \label{lem:infeasible}
   If the underlying parameterized graph problem is closed under
   induced subgraphs, then the \ea{k} produces a feasible solution
   after $O(n \log n)$ steps in expectation. After this point, it does not accept infeasible solutions.
\end{lemma}
\begin{proof}
  We consider the potential function
  \[
    \phi(x) \coloneqq \max\{-f_k(x),0\}
  \]
  and let $(X_t)_{t \ge 0}$ be the stochastic process corresponding to
  the potential of the solution generated by \ea{k} in the $t$-th
  iteration.
  The elitist nature of the \ea{k} ensures that this potential is
  nonincreasing, and since feasible points always attain nonnegative
  $f_k$-values, $\phi(x) = 0 \iff x \text{~is feasible}$. Thus it
  suffices to bound the drift $\E[X_t - X_{t+1}\mid X_t = s]$ and
  apply Theorem~\ref{thm:multidrift}.

  Assume that the point $x$ in the $t$-th iteration is infeasible. It
  follows that $X_t > 0$, and in fact, $X_t = |x|$ by
  Equation~\eqref{eq:fitness}. We argue that there is a reasonably
  good chance that the \ea{k} can reduce the potential by at least
  one. For any index $i$ such that $x[i] = 1$, the probability of
  producing an offspring $y$ by flipping only the $i$-th bit of $x$ to
  zero in the mutation step, leaving the remaining bits unchanged, is
  at least
    \[
      \frac{1}{2n}\left(1-\frac{1}{2n}\right)^{2n-1} \ge
      \frac{1}{2en}.
    \]
    After this, if the offspring $y$ happens to be
    solution feasible but not cardinality feasible, then \jar is
    called (line~\ref{li:ea-jar} of Algorithm~\ref{alg:eak}) to
    produce an intermediate point $y'$.

    We now argue that $y'$ is copied back to $y$ only if the potential
    of $y'$ is no larger than the potential of $y$.
    Note that $y'$ is guaranteed to be solution feasible, since the
    repair procedure always returns a cover.    
    In the case that $y'$ is still not cardinality feasible, then in
    line~\ref{li:protect}, $y'$ is copied back to $y$ only if
    $f_k(y') \ge f_k(y)$ (and hence $\phi(y') \le \phi(y) =
    X_t-1$).
    
    Otherwise, if $y'$ is also cardinality feasible, then it must be a
    feasible point. Since all feasible points attain nonnegative
    $f_k$-values, we have $f_k(y') > f_k(x)$ and $y'$ becomes the new
    offspring in line~\ref{li:protect}. We would then have
    $X_{t+1} = 0 < X_t$.
    In any case, under this mutation event, we have
    $X_{t+1} \le X_t - 1$, and the mutation event occurs with
    probability at least $1/(2en)$.
    
   Summing over all one-bits in $x$, we have $\E[X_t - X_{t+1} \mid X_t = s] \ge \frac{s}{2en}$, and applying Theorem~\ref{thm:multidrift}, the expected time until a feasible solution $z$ is first generated is $O(n \log n)$. At this point, since $f_k(z) \ge 0$ and every infeasible solution has a negative fitness, no infeasible solution is subsequently accepted.
\end{proof}

\subsection{Fixed-parameter tractable EAs}
The theory of parameterized complexity~\cite{DowneyFellows1999,Flum2006parameterized} allows the analysis of the running time of an algorithm to be decomposed into multiple parameters of the input. The motivation is that many large intractable problems can be solved in practice because real-world problem instances usually exhibit some kind of restriction over their structure. Parameterized complexity aims to distill the source of hardness in a problem class by isolating the superpolynomial contribution to the running time to a parameter independent of the problem size.

Formally, a parameterized problem is a language $L \subseteq \Sigma^{*}\times \N$ for a finite alphabet $\Sigma$. A problem $L$ is \emph{fixed-parameter tractable} if $(x,k) \in L$ can be decided in time $g(k) |x|^{O(1)}$ for some function $g$ that depends only on $k$. The complexity class of fixed-parameter tractable problems is $\mathsf{FPT}$. A problem $L$ is \emph{slice-wise polynomial} if $(x,k) \in L$ can be decided in time $|x|^{g(k)}$. The complexity class of slice-wise polynomial problems is denoted $\mathsf{XP}$. Note that for a problem in $\mathsf{FPT}$, each fixed parameter value determines (via $g$) the size of a \emph{leading constant} of a polynomial running time, whereas a slice-wise polynomial problem also has polynomial running time for each fixed-parameter value, but each fixed parameter value governs the \emph{degree} of the polynomial.

Applying parameterized complexity analysis to the run time analysis of evolutionary algorithms is useful when one would like to gain direct insight into how problem instance structure influences run time on $\mathsf{NP}$-hard problems~\cite{NeumannSutton2020,DBLP:journals/algorithmica/Sutton21}. For a randomized search heuristic, the \emph{optimization time} is characterized as a random variable $T$ that measures the number of fitness function evaluations until an optimal solution is first visited. With a suitable fitness function, it is possible to apply randomized search heuristics for optimization problems to decision problems as follows. An algorithm is a \emph{Monte~Carlo FPT~algorithm} for a parameterized problem $L$ if it accepts $(x,k) \in L$ with probability at least $1/2$ in time $g(k) |x|^{O(1)}$ and accepts $x \not \in L$ with probability zero. Any randomized search heuristic with a bound $\E[T] \leq g(k) |x|^{O(1)}$ on $L$ can be trivially transformed into a Monte~Carlo FPT~algorithm by stopping its execution after $2g(k) |x|^{O(1)}$ iterations. It is therefore convenient to say a randomized search heuristic runs in \emph{randomized FPT time} on a parameterized problem of size $n$ when $\E[T] \le g(k) n^{O(1)}$ (similarly for \emph{randomized XP time}).

\subsection{Iterative compression}
The technique of iterative compression was first presented by Reed, Smith and Vetta for finding odd cycle transversals in a graph~\cite{ReedSmithVetta2004transversals}. The main idea is to employ a \emph{compression routine} that takes as input a problem instance together with a solution $S$ and calculates a smaller solution or proves $S$ is already of minimum size.

The running time of the compression routine depends exponentially on the size $|S|$ of the solution to compress, so a good way to utilize it is to build a graph up one vertex at a time while always keeping the minimum possible solution. In other words, starting with an empty graph $V' = \emptyset$ and $S = \emptyset$, since the problems we consider are closed under induced subgraphs, $S$ would be a valid solution of $G[V']$. We then iterate over the vertex set $V$ adding each vertex $v$ to both $V'$ and $S$ (see Algorithm~\ref{alg:ic}). In each step, since $S$ would again be a solution for $G[V']$, we call the compression routine on $G[V']$ and $S$ to compute a smaller solution for $G[V']$, or certify none exists. Thus if the compression routine runs in $O(g(k)\cdot n^c)$ time for a constant $c > 0$, iterative compression correctly solves the parameterized problem in $O(g(k)\cdot n^{c+1})$ time.

\begin{algorithm}
   \SetKwInOut{Input}{input}
   \Input{A graph $G = (V,E)$}
   \BlankLine
   $V' \gets \emptyset$\;
   $S \gets \emptyset$\;
   \For{$v \in V$}{
      $V' \gets V' \cup \{v\}$\;
      $S \gets S \cup \{v\}$\;
      $S \gets \textsc{Compress}(G[V'],S)$\;
   }
   \Return $S$\;
   \caption{\label{alg:ic}\textsc{IterativeCompression}}
\end{algorithm}

The iterative compression technique motivates the \ea{k} framework presented in this paper. In the remainder of the paper, we will prove that the jump-and-repair approach enables the (1+1)~EA to simulate iterative compression, resulting in the solution of certain $\mathsf{NP}$-hard graph problems in randomized FPT time.

\section{$k$-Vertex Cover}
\label{sec:k-vertex-cover}
Given an undirected graph $G = (V,E)$ with $|V| = n$ vertices and $|E| = m$ edges, and a natural number $k$, the $k$-\textsc{VertexCover} problem is the problem of finding a set $S \subseteq V$ such that $|S| \le k$ and, for each $\{u,v\} \in E$, $\{u,v\} \cap S \ne \emptyset$, i.e., at least one endpoint of every edge in $E$ is in $S$. The set $S$ is called a vertex cover of $G$. Note that vertex covers are closed under induced subgraphs, so our framework is applicable here.

The following lemma shows that there is a compression routine for $k$-\textsc{VertexCover}.
\begin{lemma}[Compression for $k$-\textsc{VertexCover}]
   \label{lem:vc-ic}
   Suppose that $S$ is a vertex cover of a graph $G = (V,E)$ with $|S|
   > k$, and that $G$ has a $k$-vertex cover. Then a cover of size at most $k$ can be found by removing some subset $R \subseteq S$ and repairing each uncovered edge by adding $N(u)$ for each $u \in R$.
\end{lemma}
\begin{proof}
   Assume there is a $k$-vertex cover $S^{*}$ in $G$ and let $R \coloneqq (S \setminus S^{*} )\subseteq S$. Suppose removing the vertices of $R$ leaves a set $F \subseteq E$ edges uncovered. Let $\{u,v\} \in F$ be an arbitrary uncovered edge, and w.l.o.g., assume $u \in R$. It follows that $v \in S^{*}$ since $S^{*}$ is a cover for $G$. Thus, $N(R) \cup (S \setminus R) = S^{*}$, and thus adding the vertices $N(u)$ for each $u \in R$ obtains a $k$-vertex cover for $G$.
\end{proof}

Using the result of Lemma~\ref{lem:vc-ic}, we are able to design a jump-and-repair operator for vertex covers that takes a set selected by $x_S$ and executes a focused jump by flipping each element in $\{i : x_S[i] = 1\}$ with probability $1/2$, hence choosing a random subset $S'$ of the set selected by $x_S$. If $S'$ is already a vertex cover, this is returned. Otherwise, the solution $S'$ is repaired to be solution feasible by covering all of the uncovered edges with the neighbors of the ``removed'' elements from $x_S \setminus S'$. The resulting set is guaranteed to be a vertex cover, which is then returned by the operator. The specific jump-and-repair for vertex cover is listed in Algorithm~\ref{alg:jump-and-repair-vc}.

\begin{algorithm}
   \Input{A pair of strings $x = (x_S,x_V)$ and a graph $G$}
   \BlankLine
   \tcc{Jump}
   $S' \gets \emptyset$\;
   \For{$i \in \{ 1,\ldots,n : x_S[i] = 1 \}$}{
   \Wprob $1/2$ \Do $S' \gets S' \cup \{i\}$\;
   }
   \lIf{$S'$ is a vertex cover for $G[x_V]$}{\Return $(x_{S'},x_V)$}
   \tcc{Repair}
   $T \gets \emptyset$\;
   \For{$v \in x_S \setminus S'$}{$T \gets T \cup N(v)$}
   \Return $(x_{S' \cup T},x_V)$\;
   \caption{\label{alg:jump-and-repair-vc} \jarvc}
\end{algorithm}

\begin{theorem}
   \label{thm:vc} Let $G = (V,E)$ be a graph with a $k$-vertex cover. Then the expected optimization time of the \ea{k} applied to $G$ is bounded by $O(2^k n^2\log n)$.
\end{theorem}
\begin{proof}
  By Lemma~\ref{lem:infeasible}, a feasible solution is generated
  after $O(n \log n)$ iterations in expectation and infeasible
  solutions are not accepted thereafter. It remains to bound the time
  spent on feasible solutions. We consider the potential function
  $\phi(x) = n - f_k(x)$ and consider the stochastic process
  $(X_t)_{t \ge 0}$ on $\{0,\ldots,n\}$ corresponding to the potential
  $\phi$ of the solution generated by the execution of the \ea{k}
  after the first feasible solution is found.

  Let $x = (x_S,x_V)$ be the solution in the $t$-th feasible iteration
  of \ea{k}. Since we assume $x$ is feasible, $|x_S| \le k$ and
  corresponds to a cover on $G[x_V]$. Denote by $\mathcal{E}_i$ the
  event that mutation changes $x_V[i]$ from 0 to 1 and that, after
  mutation, $x_S[i] = 1$. We argue that
  $\Pr(\mathcal{E}_i) \ge (4e n^{2})^{-1}$.
  In particular, we pessimistically assume that $x_S[i] = 0$, and thus
  $x_V[i]$ and $x_S[i]$ both must flip under mutation. Mutation
  changes only these bits with probability $\frac{1}{4n^2}\left(1 - \frac{1}{2n}\right)^{2n-2} \ge
  \frac{1}{4en^2}$, and this is sufficient for event $\mathcal{E}_i$.

  Conditioning on $\mathcal{E}_i$, the intermediate offspring
  corresponds to a new graph $G[x_{V} \cup \{i\}]$ and
  $x_{S} \cup \{i\}$ must be a cover, since only vertex $i$ was
  introduced, and all edges in $G[x_{V} \cup \{i\}]$ incident to $i$
  are covered as $i$ is in the set corresponding to
  $x_{S} \cup \{i\}$. Therefore, this intermediate string is solution
  feasible (but not necessarily cardinality feasible).
  
  If the intermediate offspring is not
  cardinality feasible, this triggers a call to \jarvc in
  line~\ref{li:ea-jar} of Algorithm~\ref{alg:jump-and-repair}.
  If \jarvc also fails to produce a cardinality feasible solution,
  then we may discard the event $\mathcal{E}_i$.  Otherwise, \jarvc
  returns a vertex cover of size at most $k$ of the selected subgraph,
  and this point would be fitter than any cardinality-infeasible
  solution so it is accepted as $y$ in line~\ref{li:protect} of
  Algorithm~\ref{alg:jump-and-repair}.
  
   Let $\mathcal{J}$ denote the event that the subsequent call to
   \jarvc results in an offspring $y$ that is a valid $k$-vertex cover
   of $G[y_{V}]$. We seek to bound the probability of
   $\mathcal{J}$. Since $x$ is feasible, then $|x_S| \le k$, and it
   holds that $|x_{S} \cup \{i\}| \le k+1$. By Lemma~\ref{lem:vc-ic},
   there is a $k$-vertex cover that can be obtained by removing a
   subset of elements in $x_S$ and ensuring that any uncovered edges
   are covered by neighbors of $x_S$. \jarvc selects exactly this
   subset with probability $2^{-|S|}$ and so we have
   $\Pr(\mathcal{J}\mid \mathcal{E}_i) \ge 2^{-|S|} \ge 2^{-(k+1)}$.

   We now argue that the joint occurrence of events $\mathcal{E}_i$ and $\mathcal{J}$ results in a feasible offspring $y$ with a strictly larger fitness value. In particular, $|x_{V} \cup \{i\}| = |x_V| + 1$ and so $f(y) = f(x) + 1$ since $\mathcal{J}$ ensures that $y$ is feasible. Therefore, the drift conditioned on $\mathcal{J} \cap \mathcal{E}_i$ is 1. We can bound the total drift as follows.
   \begin{align*}
      \E[X_{t}-X_{t+1}\mid X_t = s] & \ge \sum_{i : x_V[i] = 0} \Pr(\mathcal{J}\cap\mathcal{E}_i)
      = \sum_{i : x_V[i] = 0} \Pr(\mathcal{J} \mid \mathcal{E}_i)\Pr(\mathcal{E}_i)                                           \\
                                    & \ge \frac{|\{i : x_V[i]=0\}|}{4e\cdot 2^{k+1}n^2}  = \Omega{\left(\frac{s}{2^kn^2}\right)}.
   \end{align*}
   Applying Theorem~\ref{thm:multidrift} completes the proof.
\end{proof}

It follows that the \ea{k} can be characterized as a Monte Carlo FPT algorithm for the $k$-\textsc{VertexCover} problem on graphs. This also allows for one to develop a strategy to finding the \emph{minimum} vertex cover of a graph in FPT time. In particular, consider the restart strategy for the \ea{k} listed in Algorithm~\ref{alg:restart-eak}.

\begin{algorithm}
   \caption{\label{alg:restart-eak} Restart framework}
   \SetKwInOut{Input}{input}
   \Input{A graph $G = (V,E)$}
   \BlankLine
   $k \gets 1$\;
   \While{a vertex cover for $G$ has not been found}{
      Run the \ea{k} on $G$ for $13e^22^kn^2\ln n$ steps\;
      $k \gets k+1$\;
   }
\end{algorithm}

\begin{theorem}
   \label{thm:restart-eak}
   Let $G = (V,E)$ be a graph with an optimal vertex cover of size $OPT$. Then with probability $1 - o(1)$, the restart framework of the \ea{k} in Algorithm~\ref{alg:restart-eak} finds an optimal vertex cover for $G$ within $O(2^{OPT}n^2\log n)$ function calls.
\end{theorem}
\begin{proof}
   Since there are no vertex covers for $G$ of size $k < OPT$, Algorithm~\ref{alg:restart-eak} is successful when the run in which $k=OPT$ finds a vertex cover for $G$. Thus, it suffices to bound the probability of success in this run. Observe that this is equivalent to the probability that an unbounded run would be successful before the cutoff time. Let $T$ denote the random variable that measures the optimization time on an unbounded run of \ea{k} with $k=OPT$, and let $T = T_1 + T_2$ where $T_1$ measures the steps spent on infeasible solutions and $T_2$ the steps spent on feasible solutions.
   We thus seek to bound $\Pr(T_1 + T_2 < 13 e\cdot 2^{OPT} n^2 \ln n)$.

   In the proof of Theorem~\ref{thm:vc}, we have bounded the drift
   factor in the feasible phase from below as $\delta \ge 1/(4
   e\cdot 2^{k+1}n^2) = 1/(8 e\cdot 2^k n^2)$. Thus, setting $r =
   \frac{1}{2}\ln n$ in the tail bound of
   Theorem~\ref{thm:multidrift}, we get
   \[
     \Pr(T_2 > 12e\cdot 2^{OPT}n^2\ln n) < e^{-r} = \frac{1}{\sqrt{n}}.
   \]
   Similarly, the drift factor in the infeasible phase, obtained in
   the proof of Lemma~\ref{lem:infeasible}, is $\delta = \Omega(1/n)$,
   and so
   \[\Pr(T_1 > e \cdot 2^{OPT}n^2\ln n) < e^{-\Omega(n)},
   \]
   and it follows that
   \[\Pr(T_1 + T_2 < 13e\cdot 2^{OPT} n^2 \ln n) = 1 - o(1).
     \]

   If Algorithm~\ref{alg:restart-eak} is successful in its $k$-th run where $k=OPT$, it spends only a constant fraction of time on too-small vertex covers, and the total optimization time is
   \[
      \sum_{k=1}^{OPT} 13 e\cdot 2^kn^2 \ln n = O(2^{OPT}n^2\log n),
   \]
   which completes the proof.
\end{proof}

Ignoring polynomial factors, this constitutes an upper bound that is smaller by an exponential factor (in $OPT$) than the bound presented for Global SEMO solving minimum \textsc{VertexCover}, which requires
$O(n^2\log n + OPT\cdot n^2 + 4^{OPT}n)$ fitness evaluations in expectation~\cite{KratschNeumann2012}.

\section{Feedback Vertex Sets in Tournaments}
\label{sec:feedback-vertex-sets}
A \emph{tournament} is an orientation of a complete graph, that is, a directed graph $G = (V,E)$ such that for every distinct $u,v \in V$, either $(u,v) \in E$ or $(v,u) \in E$. A tournament $G$ is \emph{transitive} when
\[
   (u,v),(v,w) \in E \implies (u,w) \in E.
\]
These structures, illustrated in Figure~\ref{fig:tournament}, have important applications, for example in social choice and voting theory~\cite{mcgarvey1953}. A tournament is transitive if and only if it contains no directed cycles. Moreover, all cycles contain a directed triangle, yielding the following basic proposition.
\begin{proposition}
   \label{prp:triangle}
   A tournament $G$ is transitive if and only if it contains no directed triangles.
\end{proposition}
\begin{proof}
   If $G$ contains a directed triangle, then it is not transitive. If $G$ is not transitive, it must contain some $u,v,w \in V$ such that $(u,v), (v,w) \in E$ but $(u,w) \not \in E$ (otherwise it is transitive). Since $G$ is a tournament, it follows that $(w,u) \in E$ so it contains the directed triangle $u \to v \to w \to u$.
\end{proof}

\begin{figure}
   \centering
   \begin{tikzpicture}
      \foreach \v in {1,...,8} {
            \node[vertex] (\v) at ({360/8 * (\v+1)}:2cm) {};
         }
      % permutation: 1,5,2,3,8,6,4,7

      \draw[arc] (1) edge (5);
      \draw[arc,bend right=10] (1) edge (2);
      \draw[arc,bend left=10] (1) edge (3);
      \draw[arc,bend left=10] (1) edge (8);
      \draw[arc,bend right=10] (1) edge (6);
      \draw[arc,bend left=10] (1) edge (4);
      \draw[arc,bend right=10] (1) edge (7);

      \draw[arc,bend right=10] (5) edge (2);
      \draw[arc,bend left=10] (5) edge (3);
      \draw[arc,bend left=10] (5) edge (8);
      \draw[arc,bend right=10] (5) edge (6);
      \draw[arc,bend left=10] (5) edge (4);
      \draw[arc,bend right=10] (5) edge (7);

      \draw[arc,bend right=10] (2) edge (3);
      \draw[arc,bend right=10] (2) edge (8);
      \draw[arc] (2) edge (6);
      \draw[arc,bend left=10] (2) edge (4);
      \draw[arc,bend right=10] (2) edge (7);

      \draw[arc,bend right=10] (3) edge (8);
      \draw[arc,bend left=10] (3) edge (6);
      \draw[arc,bend right=10] (3) edge (4);
      \draw[arc] (3) edge (7);

      \draw[arc,bend right=10] (8) edge (6);
      \draw[arc] (8) edge (4);
      \draw[arc,bend left=10] (8) edge (7);

      \draw[arc,bend right=10] (6) edge (4);
      \draw[arc,bend right=10] (6) edge (7);

      \draw[arc,bend left=10] (4) edge (7);

   \end{tikzpicture}

   \caption[]{\label{fig:tournament} A transitive tournament on 8 vertices.}
\end{figure}
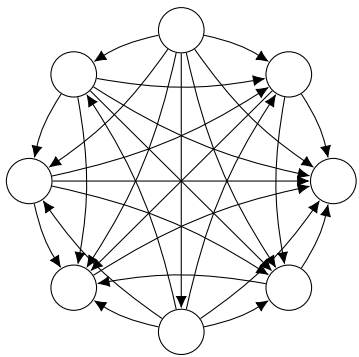

A \emph{feedback vertex set} is a set $S \subseteq V$ that intersects every directed cycle of $G$ (and thus $G - S$ is transitive). Finding a feedback vertex set of minimum size corresponds to the problem of making an antisymmetric relation transitive by eliminating the smallest possible number of elements. The problem is $\mathsf{NP}$-hard~\cite{Speckenmeyer1990feedback,CaiDengZang2001fvst}, however the parameterized version is in $\mathsf{FPT}$~\cite{DomGHNTfptfeedback}. In the parameterized version of the problem, $k$-\textsc{FVST}, we are given a tournament $G$, and the task is to find a feedback vertex set of size at most $k$.

Every transitive tournament has a unique topological sort by starting from the vertex with zero in-degree, removing that vertex and recursively sorting the remainder of the graph. Thus we have the following proposition.
\begin{proposition}
   \label{prp:topo-sort}
   If $G = (V,E)$ is a tournament and $S \subseteq V$ is a feedback vertex set of $G$, then there is a unique sequence on $V\setminus S$
      $\topo(G,S) \coloneqq (v_1,v_2,\ldots,v_{|V\setminus S|})$,
   where $(v_i,v_j) \in E \iff i < j$.
\end{proposition}

Similar to $k$-\textsc{VertexCover}, we can construct the following compression result for $k$-\textsc{FVST}.

\begin{lemma}[Compression for $k$-\textsc{FVST}]
   \label{lem:fvst-ic}
   Let $G = (V,E)$ be a tournament. Suppose that $S \subseteq V$ is a feedback vertex set of $G$ (that is, $G - S$ is transitive), and $|S| > k$. Suppose there exists a feedback vertex set $S^{*}$ of $G$ with $|S^{*}| \le k$. Then a feedback vertex set of size at most $k$ can be found by
   \begin{enumerate}
      \item\label{item:5} removing a set $R$ of vertices from $S$, and
      \item\label{item:6} guessing the correct relative ordering on $R$ in $\topo(G,S^{*})$.
   \end{enumerate}
\end{lemma}
\begin{proof}
   Set $S' = S \cap S^{*}$ and $R = S \setminus S'$. Let $T$ denote the set of vertices $v \not\in S$ that appear in a directed triangle $u,v,w$ where $u,w \in R$. It follows that $T \subseteq S^{*}$ because otherwise $G - S^{*}$ would not be cycle-free.

   We seek to find the smallest set that extends $S' \cup T$ but does not overlap with $R$ (note that $S^{*}$ is one such extension).

   Pick an arbitrary $W$ where $W \supseteq (S \cup T)$. Note that $W$ is a feedback vertex set since $G-S$ is transitive and $W \supseteq S$. We construct a sequence $\sigma_{W}$ on $V\setminus W$. Define a labeling $p : V \setminus W \to |\pi|$ as follows. For each $v \in V \setminus W$, $p(v) = \min\Big(\{i : (v,\pi_i) \in E\} \cup\{|\pi|+1\}\Big)$. Then for all $v,w \in V \setminus W$ we set
   \[
      v \prec_{\sigma_{W}} w
      \iff \begin{cases}
         p(v) < p(w) \text{~or,}                   \\
         p(v)=p(w), \topo(G,W)[v] < \topo(G,W)[w]. \\
      \end{cases}
   \]
   We claim that $\sigma_W = \topo(G,W)$ if and only if $W \setminus R$ is a feedback vertex set for $G$.

   First, assume $\sigma_W = \topo(G,W)$. Suppose for contradiction that there is a cycle in $G - (W\setminus R)$. Then by Proposition~\ref{prp:triangle}, there must be a directed triangle $u,v,w$ in $G - (W\setminus R)$. Note that $G - W$ is triangle-free (since $W$ is a feedback vertex set), so there must be at least one vertex of this triangle in $R$. Furthermore, there can be \emph{at most} one vertex of this triangle in $R$ since if there were two vertices appearing in $R$, then the third would be in $S' \cup T$ which do not appear in $G - (W\setminus R)$. Obviously, all three triangle endpoints cannot be in $R$ since $S^{*}$ is a feedback vertex set. Without loss of generality, suppose $u \in R$ and $v,w \not\in R$ with $(v,w) \in E$. Note that $u = \pi_i$ for some $i$ and the triangle is completed by edges $(w,\pi_i) \in E$ and $(\pi_i,v) \in E$. Since $(w,\pi_i) \in E$, it follows that $i \ge p(w)$, and since $(\pi_i,v) \in E$, it follows that $i < p(v)$ and we have $p(w) < p(v)$ and so $w$ appears before $v$ in $\sigma_{W}$. However, $(v,w) \in E$ so $v$ appears before $w$ in $\topo(G,W)$, which contradicts the assumption the $\sigma_W = \topo(G,W)$.

   Now suppose $\sigma_W \ne \topo(G,W)$. Then there must be a pair $v,w \in V \setminus W$ with $(v,w) \in E$ but $p(v) > p(w)$. This means for $i = p(w)$, we have $(w,\pi_i) \in E$ and $(v,\pi_i) \not\in E$. But since $G$ is a tournament, $(v,\pi_i) \not\in E \implies (\pi_i,v) \in E$. Since $\pi_i \in R$, there is a triangle $v,w,\pi_i$ in $G - (W\setminus R)$.

   Thus, given $S'$ and $\pi$, it suffices to find the minimal such $W$ so that $\sigma_W$ is equal to $\topo(G,W)$. This can be done by computing the sequences for $\sigma_{S\cup T}$ and $\topo(G,S\cup T)$ and adding to $S \cup T$ the vertices not in a longest common subsequence. Since we know that there exists at least one feedback vertex set $S^{*}$ disjoint from $R$ with size at most $k$, it follows that this method computes a feedback vertex cover with size at most $k$.
\end{proof}

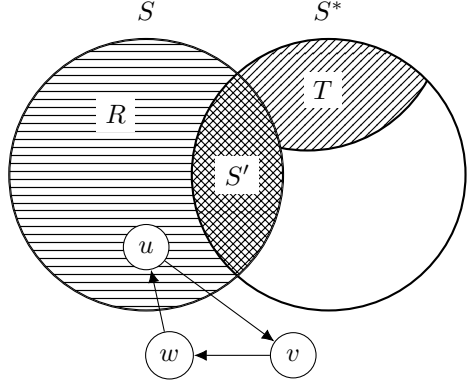
\begin{figure}
   \centering
   \begin{tikzpicture}[scale=1.2]
      \draw[thick] (0,0) circle (1.5cm) node (S) {} (0:2cm) circle (1.5cm) node (Sstar) {};
      \draw[fill=white,draw=none] (0,0) circle (1.5cm);

      \begin{scope}[even odd rule]
         \clip (0:2cm) circle (1.5cm);
         \draw[pattern={north east lines},thick] (45:2.5cm) circle (1.5cm) node[yshift=-10mm,xshift=2mm,fill=white] {$T$};
         \draw[preaction={fill=white},pattern={crosshatch}, draw=none,even odd rule] (0,0) circle (1.5cm) node[xshift=1.2cm,fill=white] {$S'$};
      \end{scope}

      \begin{scope}
         \clip (0,0) circle (1.5cm);
         \draw[pattern={horizontal lines},thick, even odd rule] (0,0) circle (1.5cm) node[xshift=-4mm,yshift=8mm,fill=white]{$R$} (0:2cm) circle (1.5cm);
      \end{scope}

      \node[above=1.8cm of S] {$S$};
      \node[above=1.8cm of Sstar] {$S^{*}$};

      \node[vertex] (u) at (0,-0.8) {$u$};
      \node[vertex,below right=1cm and 1.5cm of u] (v) {$v$};
      \node[vertex,left=1cm of v] (w) {$w$};

      \draw[arc] (u) -- (v);
      \draw[arc] (v) -- (w);
      \draw[arc] (w) -- (u);

   \end{tikzpicture}

   \caption[]{\label{fig:fvst-proof} Illustration of possible directed triangles in the proof of the compression lemma for $k$-\textsc{FVST} (Lemma~\ref{lem:fvst-ic}).}
\end{figure}

Lemma~\ref{lem:fvst-ic} allows us to design a focused jump-and-repair operation for the $k$-\textsc{FVST} problem. We outline this operation is Algorithm~\ref{alg:jump-and-repair-fvst}.

\begin{algorithm}
   \Input{A pair of strings $x = (x_S,x_V)$ and a tournament $G$}
   \BlankLine
   \tcc{Jump}
   $S' \gets \emptyset$\;
   \For{$i \in \{ 1,\ldots,n : x_S[i] = 1\}$}{
   \Wprob $1/2$ \Do $S' \gets S' \cup \{i\}$\;
   }
   $R \gets x_S \setminus S'$\;
   $\pi \gets $ a random permutation of $R$\;\label{li:randperm}
   \lIf{$G[x_V \setminus S']$ is transitive}{\Return $(x_{S'},x_V)$}
   \tcc{Repair}
   $T \gets \emptyset$\;
   \While{$\exists$ a directed triangle $u,v,w$ s.t.\ $u,w \in R$ and $v \not\in T$}
   {$T \gets T \cup \{v\}$}
   \For{$v \in x_V \setminus (x_S \cup T)$}{$p[v] = \min(\{i:(v,\pi_i]) \in E\} \cup \{|\pi|+1\})$}
   $\sigma \gets \topo(G,x_S \cup T)$\;
   $\sigma' \gets $ the sequence obtained by sorting $x_V \setminus (x_S \cup T)$ by $p$ and breaking ties by $\sigma$\;
   $Z \gets $ vertices in longest common subsequence of $\sigma$ and $\sigma'$\;
   \Return $(x_{S' \cup T \cup (x_V \setminus Z)},x_V)$\;
   \caption{\label{alg:jump-and-repair-fvst} \jarfvst}
\end{algorithm}

The focused jump-and-repair operation for $k$-\textsc{FVST} is slightly more complicated than the operation for $k$-\textsc{VertexCover}, since we also guess the correct ordering on the vertices removed from the current feedback set, and it is necessary to find the longest subsequence common the two arrangements of the remaining vertices. The former slows the optimization time by a factor of $k!$. The latter is not directly reflected in the optimization time, but would incur an extra factor of $O(n^2)$ in each repair operation for solving the \textsc{LongestCommonSubsequence} problem using, e.g., dynamic programming~\cite{CLRS}.

\begin{theorem}
   \label{thm:fvst}
   Let $G = (V,E)$ be a tournament with a feedback vertex set of size at most $k$. Then the expected optimization time of the \ea{k} applied to $G$ is bounded by $O(2^kk! n^2 \log n)$.
\end{theorem}
\begin{proof}
   We focus only on the phase of execution that consists of feasible solutions, which occurs after $O(n \log n)$ iterations in expectation by Lemma~\ref{lem:infeasible}.

   As with the proof of Theorem~\ref{thm:vc}, we bound the drift of the potential function $\phi(x) = n - f_k(x)$ modeled by the stochastic process $(X_t)_{t \ge 0}$ which starts after the first feasible iterations, and hits the absorbing state when a feedback vertex set of size at most $k$ on the entire graph $G$ is discovered.

   Again, assume that $x = (x_S,x_V)$ is the candidate solution in the $t$-th feasible iteration of the \ea{k}, and denote $\mathcal{E}_i$ as the event that changes $x_V[i]$ from 0 to 1 and $x_S[i] = 1$ after mutation. The resulting string is solution feasible, but not necessarily cardinality feasible. Let $\mathcal{J}$ be the event conditioned on $\mathcal{E}_i$ that the subsequent call to \jarfvst results in a solution $y$ that is a feedback vertex set of size at most $k$ for the induced subgraph $G[y_V]$.

   Since $G$ has a $k$ feedback vertex set, the induced subgraph $G[x_{V} \cup \{i\}]$ must also have a $k$ feedback vertex set $S^{*} \subseteq x_V \cup \{i\}.$ With probability $2^{-(|x_S|+1)}$ the jump phase of Algorithm~\ref{alg:jump-and-repair-fvst} selects $S' = (x_S \cup \{i\}) \cap S^{*}$, and with probability $\frac{1}{(|x_S| + 1 - |S'|)!}$, it selects the permutation $\pi$ that corresponds to the order of the vertices in $x_S \setminus S'$ in $\topo(G[x_V \cup \{i\},S^{*}]$.

   By Lemma~\ref{lem:fvst-ic}, if this jump operation was successful, the subsequent repair step produces a feedback vertex set for $G[x_v \cup \{i\}]$ with size at most $k$. Since $|x_S| \le k$, and $|S'| > 0$, we have
   \[
      \Pr(\mathcal{J}) \ge \frac{2^{-|x_S|}}{(|x_S| + 1 - |S'|)!}
      \ge \frac{1}{2^{k+1}k!}.
   \]
   As the resulting offspring $y$ is feasible and $|y_V| = |x_V| + 1$, it follows that $f(y) > f(x)$.

   By the same argumentation as in the proof of Theorem~\ref{thm:vc}, $\Pr(\mathcal{E}_i) \ge (2e n)^{-2}$, and we bound the total drift as
   \begin{align*}
      \E[X_{t}-X_{t+1}\mid X_t = s] & \ge \sum_{i : x_V[i] = 0} \Pr(\mathcal{J}\cap\mathcal{E}_i)                                 
                                     = \sum_{i : x_V[i] = 0} \Pr(\mathcal{J} \mid \mathcal{E}_i)\Pr(\mathcal{E}_i)               \\
                                    & \ge \frac{|\{i : x_V[i]=0\}|}{4e^22^{k+1}k!n^2}  = \Omega{\left(\frac{s}{2^kk!n^2}\right)},
   \end{align*}
   and the claim is proved by applying Theorem~\ref{thm:multidrift}.
\end{proof}

\section{Odd Cycle Transversals}
\label{sec:odd-cycle-transv}

Let $G = (V,E)$ be an undirected graph with $n$ vertices and $m$ edges. An \emph{odd cycle transversal} of $G$ is a set $S \subseteq V$ such that $G[V \setminus S]$ is bipartite. Finding an odd cycle transversal is a so-called \emph{node-deletion problem} for a nontrivial hereditary property, and thus is NP-complete~\cite{LewisYannakakis1980NodeDeletionProblem}. Finding a transversal of size at most $k$, however, is also fixed-parameter tractable, and was the first problem on which the iterative compression technique was developed in the seminal paper of Reed, Smith and Vetta~\cite{ReedSmithVetta2004transversals}. In that paper, the authors presented an $O(4^kkmn)$ iterative compression algorithm for solving $k$-odd cycle transversal, and this analysis has been improved to $O(3^kkmn)$~\cite{Hueffner2009AlgorithmEngineeringOptimal} as well as further simplified~\cite{LokshtanovEtAl2009SimplerParameterizedAlgorithm}. The current fastest parameterized algorithm for the problem is $O(2.3146^k \cdot\poly(n))$ and is due to Lokshtanov et al.~\cite{LokshtanovEtAl2014FasterParameterizedAlgorithms}.

In order to derive a repair procedure, we will take advantage of an earlier result of Lokshtanov et al.~\cite{LokshtanovEtAl2009SimplerParameterizedAlgorithm} that relates two odd cycle transversals in a graph to vertex cut separating certain sets. Given a solution feasible odd cycle transversal that is not cardinality feasible, the jump procedure randomly removes some vertices from the transversal, and then guesses a bipartition of these removed vertices. The repair routine, outlined in Algorithm~\ref{alg:oct-repair}, attempts to judiciously remove any odd cycles exposed by calculating a vertex cut in the remainder of the graph. The following lemma demonstrates the repair performed by Algorithm~\ref{alg:oct-repair} is successful if the jump is successful. The proof is similar to the proofs of Lemmas 3.2 and 3.3 of Lokshtanov et al.~\cite{LokshtanovEtAl2009SimplerParameterizedAlgorithm}.

\begin{algorithm}
   \Input{A graph $G = (V,E)$, an odd cycle transversal $S$ of $G$,
a set $R \subseteq S$ of vertices to remove from $S$ together with a partition $A,B$ of $R$}
   \BlankLine

   Let $C$ and $D$ be the bipartition of $G[V \setminus S]$\;
   Find a minimal vertex cut $T$ in $G[V \setminus S]$ separating $(C\cap N(A)) \cup (D \cap N(B))$ and $(C\cap N(B)) \cup (D \cap N(A))$\;
   \Return $T$\;
   \caption{\label{alg:oct-repair} Repair odd cycles that are exposed by removing $R$ from $S$}
\end{algorithm}

 \begin{figure}
   \centering
   \begin{tikzpicture}[scale=1.2]
      \draw[thick] (0,0) circle (1.5cm) node (S) {} (0:2cm) circle (1.5cm) node (Sstar) {};
      \draw[fill=white,draw=none] (0,0) circle (1.5cm);
      \begin{scope}
         \clip (0,0) circle (1.5cm);
         \draw[thick, even odd rule] (0,0) circle (1.5cm)  (0:2cm) circle (1.5cm);
         \draw[pattern={north east lines}] (-2,0) rectangle node [fill=white] (A) {$A$}(2,2);
         \draw[pattern={dots}] (-2,0) rectangle node[fill=white] (B) {$B$} (2,-2);
      \end{scope}

      \begin{scope}[even odd rule]
         \clip (0:2cm) circle (1.5cm);
         \draw[preaction={fill=white},pattern={crosshatch}, draw=none,even odd rule] (0,0) circle (1.5cm) node[xshift=1.2cm,fill=white] {$S'$};
      \end{scope}

      %%%%%%%  
      %% C %%  
      %%%%%%%  
      \draw[thick,rounded corners] (-5,2.5) rectangle node (C) {}(-3,0.5);
      \begin{scope}
         \clip[rounded corners] (-5,2.5) rectangle (-3,0.5);
         \draw[draw=none,pattern={dots}] (-5,1.5) rectangle (-1,2.5);
         \draw[draw=none,pattern={north east lines}] (-5,1.5) rectangle (-1,-1);
      \end{scope}

      \draw[thick,dashed] (-5,1.5) -- node[fill=white,above=3mm] (CA) {$C \cap N(A)$} node[fill=white,below=3mm] (CB) {$C \cap N(B)$}(-3,1.5);

      %%%%%%%
      %% D %%
      %%%%%%%
      \draw[thick,rounded corners] (-5,-2.5) rectangle node (D) {}(-3,-0.5);
      \begin{scope}
         \clip[rounded corners] (-5,-2.5) rectangle (-3,-0.5);
         \draw[draw=none,pattern={dots}] (-5,-1.5) rectangle (-2,0);
         \draw[draw=none,pattern={north east lines}] (-5,-1.5) rectangle (-2,-2.5);

      \end{scope}
      \draw[thick,dashed] (-5,-1.5) --node[fill=white,above=3mm] (DA) {$D \cap N(A)$}  node[fill=white,below=3mm] (DB) {$D \cap N(B)$}(-3,-1.5);

      %%%%%%%%%%%%%%%%  
      %% Set Labels %%
      %%%%%%%%%%%%%%%%
      \node[above=1.8cm of S] {$S$};
      \node[above=1.8cm of Sstar] {$S^{*}$};
      \node[right=2mm of Sstar] {$T$};
      \node[above=1cm of C] {$C$};
      \node[above=1cm of D] {$D$};

      %%%%%%%%%%%%%  
      %% "Edges" %%  
      %%%%%%%%%%%%%  
      \draw[shorten <= 0.2cm, shorten >= 0.5cm] (CA) -- (A);
      \draw[shorten <= 0.2cm, shorten >= 0.2cm,transform canvas={yshift=2mm}] (CA) -- (A);

      \draw[shorten <= 0.2cm, shorten >= 0.5cm] (CB) -- (B);
      \draw[shorten <= 0.5cm, shorten >= 0.5cm,transform canvas={yshift=2mm}] (CB) -- (B);

      \draw[shorten <= 0.2cm, shorten >= 0.5cm] (DA) -- (A);
      \draw[shorten <= 0.5cm, shorten >= 0.5cm,transform canvas={yshift=-2mm}] (DA) -- (A);

      \draw[shorten <= 0.2cm, shorten >= 0.5cm] (DB) -- (B);
      \draw[shorten <= 0.2cm, shorten >= 0.2cm,transform canvas={yshift=-2mm}] (DB) -- (B);
   \end{tikzpicture}
   \caption[]{\label{fig:oct-proof}
      Illustration of partitions of $G = (V,E)$ in the proof of the compression lemma for $k$-\textsc{OddCycleTransversal} (Lemma~\ref{lem:oct-ic}).

   }
\end{figure}
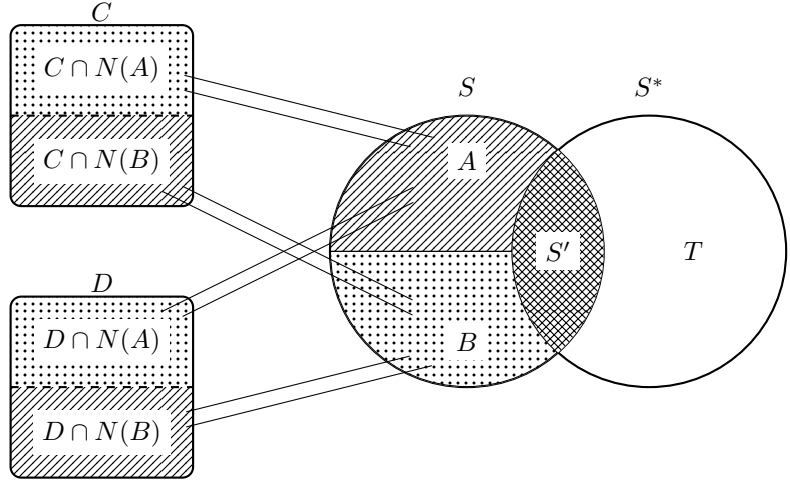

\begin{lemma}[Compression for $k$-\textsc{OddCycleTransversal}]
   \label{lem:oct-ic}
   Let $S$ be an odd cycle transversal of a graph $G = (V,E)$ with $|S| > k$, and suppose there exists an odd cycle transversal $S^{*}$ of $G$ with $|S^{*}| < k$. Then an odd cycle transversal of size at most $k$ can be found by
   \begin{enumerate}
      \item\label{item:3} removing a set $R$ of vertices from $S$,
      \item\label{item:4} guessing the correct partition $A \uplus B = R$ that splits $R$ correctly into the bipartition of $G$ induced by $S^{*}$, and
      \item\label{item:7} running Algorithm~\ref{alg:oct-repair} on these sets to repair any odd cycles exposed by removing $R$.
   \end{enumerate}
 \end{lemma}
\begin{proof}
   Let $S^{*}$ be an odd cycle transversal of size at most $k$ in $G$. Then $G[V\setminus S^{*}]$ is bipartite with bipartition $V_A,V_B$. We prove that if we remove the vertex set $R \coloneqq S\setminus S^{*}$, then calling Algorithm~\ref{alg:oct-repair} with sets $R$,
   $A \coloneqq R \cap V_A$ and $B \coloneqq R \cap V_B$ results in an odd cycle transversal of size at most $k$.
   Let $S' = S \cap S^{*}$ and $C,D \subseteq V \setminus S$ be the bipartition formed by $S$ (see Figure~\ref{fig:oct-proof} for an illustration).

   We first argue that $T \coloneqq S^{*} \setminus S$ is a vertex cut that separates $(C \cap N(A)) \cup (D \cap N(B))$ and $(C \cap N(B)) \cup (D \cap N(A))$ in the graph $G[V \setminus S]$.

   Let $P$ be a path in $G[V \setminus (S \cup T)]$. We show that $P$ cannot connect $(C \cap N(A)) \cup (D \cap N(B))$ and $(C \cap N(B)) \cup (D \cap N(A))$ in $G[V \setminus S]$ by ruling out the possible endpoints of $P$. Let $(u,\ldots,v)$ be the sequence of vertices associated with $P$.
   \begin{description}
      \item[Case 1:] Suppose $u \in C \cap N(A)$ and $v \in C \cap N(B)$. Since $u,v \in C$ and $C$ is an independent set in $G[V \setminus (S \cup T)]$, $P$ must contain an odd number of vertices. Furthermore, $u$ has a neighbor $a \in A$ and $v$ has a neighbor $b \in B$. Thus replacing the vertices in $S \setminus S^{*}$ would allow us to construct the path $P' = (a,u,\ldots,v,b)$ in $G[V \setminus S^{*}]$, also with an odd number of vertices. Since $P'$ also has an odd vertex count, there is no proper 2-coloring of $P'$ that places its endpoints in opposite color classes. But $a \in A \subseteq V_A$ and $b \in B \subseteq V_B$, so this contradicts the fact that $G[V \setminus S^{*}]$ is bipartite with bipartition $V_A$ and $V_B$.
      \item[Case 2:] Suppose $u \in C \cap N(A)$, $v \in D \cap N(A)$. Since $u \in C$ and $v \in D$, then $P$ has an even count of vertices. Note that both $u$ and $v$ have neighbors $a,a' \in A$ (respectively), so replacing the vertices in $S \setminus S^{*}$ would allow us to construct an even-length path $P' = (a,u,\ldots,v,a')$ in $G[V \setminus S^{*}]$. Since $P'$ has an even vertex count, there is no proper 2-coloring of $P'$ that places its endpoints in the same color class. Again, this contradicts the fact that $S^{*}$ is an odd-cycle transversal as described above.

      \item[Case 3:] If $u \in D \cap N(B)$, $v \in C \cap N(B)$, then $P$ cannot exist by a symmetric argument to Case 2, swapping $C$ with $D$ and $A$ with $B$.

      \item[Case 4:] If $u \in D \cap N(B)$, $v \in D \cap N(A)$, then $P$ cannot exist by a symmetric argument to Case 1, swapping $C$ with $D$ and $A$ with $B$.
   \end{description}
   Thus, no path exists in $G[V \setminus (S \cup T)]$ between the sets $(C \cap N(A)) \cup (D \cap N(B))$ and $(C \cap N(B)) \cup (D \cap N(A))$ and therefore $T$ is a vertex cut separating these two sets in $G[V \setminus S]$.

   Finally, we argue that if $T'$ is \emph{any} vertex cut separating
   $(C \cap N(A)) \cup (D \cap N(B))$ and $(C \cap N(B)) \cup (D \cap N(A))$ in $G[V \setminus S]$, then $T' \cup S'$ is an odd cycle transversal in $G$.

   Let $Q$ be an arbitrary cycle in $G[V \setminus (T' \cup S')]$ and denote as $(e_1,e_2,\ldots,e_q)$ the sequence of edges in the order they appear in $Q$. It suffices to show that $q$ must be even.
   If $Q$ has no vertices in $A \cup B$, then $q$ must be even since $S = A \cup B \cup S'$ is an odd cycle transversal. On the other hand, suppose $Q$ has at least one point in $A \cup B$.

   We call an edge $e_i$ \emph{internal} if $e_i \cap (A\cup B) = e_i$,
   and \emph{external} if $e_i \cap (A\cup B) = \emptyset$. Similarly, we
   call a path an \emph{internal path} when it is comprised only of
   internal edges. An \emph{external path} is an edge sequence
   $(e_i, \ldots, e_j)$ where
   $|e_i \cap (A\cup B)| = |e_j \cap (A\cup B)| = 1$, and the remaining
   edges are external. Note that $Q$ can always be decomposed into a
   sequence of internal and external paths where $|e_1 \cap e_q| = 1$.

   Since $G[A]$ and $G[B]$ are independent sets, every internal path with both endpoints in $A$ or both endpoints in $B$ has an even number of edges, and every internal path with one endpoint in $A$ and one endpoint in $B$ has an odd number of edges.

   Now consider an external path $P$ in $G[V \setminus (S \cup T')]$ with endpoints $u$ and $v$. If $u,v \in A$ then the vertex sequence for $P$ is $(u,w,\ldots,y,v)$ for some $w,y \in N(A)$. But $T'$ is a vertex cut separating $C \cap N(A)$ and $D \cap N(A)$, so there are no paths between these sets, so either $w,y \in C$ or $w,y \in D$. In either case, the subpath from $w$ to $y$ has an even count of edges in $G[V \setminus (S \cup T')]$ since $C$ and $D$ is a bipartition induced by the odd cycle transversal $S$. A similar argument shows that external paths with both endpoints in $B$ must have even length, and one endpoint in $A$ and one in $B$ must have odd length.

   Now let $Q = (P_1,P_2,\ldots,P_{\ell})$ be a decomposition of $Q$ into internal and external paths. Each path $P_i$ with endpoints in $A$ and $B$ must have a return path $P_i$ with endpoints in $B$ and $A$. Therefore, there must be an even number of paths $P_i$ with an odd count of edges, and so $Q$ must contain an even count of edges.

   To complete the proof, we see that Algorithm~\ref{alg:oct-repair} finds a minimal vertex cut $T'$ that separates $(C \cap N(A)) \cup (D \cap N(B))$ and $(C \cap N(B)) \cup (D \cap N(A))$, and $|T'| \le |T|$ since $T$ must also be such a vertex cut. As $|S' \cup T| \le k$, it follows that $S' \cup T'$ is an odd cycle transversal of $G$ with size at most $k$.
\end{proof}

\begin{algorithm}
   \Input{A pair of strings $x = (x_S,x_V)$ and a graph $G$}
   \BlankLine
   \tcc{Jump}
   $S' \gets \emptyset$\;
   \For{$i \in \{ 1,\ldots,n : x_S[i] = 1 \}$}{
   \Wprob $1/3$ \Do $S' \gets S' \cup \{i\}$\label{li:jump-prob}\;
   }
   \lIf{$S'$ is an odd cycle transversal for $G[x_V]$}{\Return $(x_{S'},x_V)$}
   \tcc{Repair}
   $R \gets x_S \setminus S'$\;
   Select a bipartition $A,B$ of $R$ uniformly at random\;
   Calculate $T$ using repair Algorithm~\ref{alg:oct-repair} called with inputs $G[x_V]$, $S$,$R$, $A$, $B$\;
   \Return $(x_{S' \cup T},x_V)$\;
   \caption{\label{alg:jump-and-repair-oct} \jaroct}
\end{algorithm}

The repair procedure of Algorithm~\ref{alg:oct-repair} finds a minimum vertex separator between two vertex sets, which can be done in $O(nm)$ time by computing the maximum flow in the appropriate network, e.g., by the Ford-Fulkerson method~\cite{CLRS}. If we insist on a vertex separator on size at most $k$, the running time bound can be improved to $O(km)$, as the existence of a separator of this size can be decided, and subsequently constructed, within this bound. In the case that a vertex separator of size at most $k$ cannot be found, the repair procedure can return an arbitrary set (e.g., $R$), as the repair would have failed.

We also point out a small deviation from the general \jar procedure listed in Algorithm~\ref{alg:jump-and-repair}. In line~\ref{li:jump-prob} of Algorithm~\ref{alg:jump-and-repair-oct}, we only select a vertex for $S'$ with probability $1/3$ rather than $1/2$, which ultimately improves the bound by a $(1.33)^k$ factor.

\begin{theorem}
   \label{thm:oct} Let $G = (V,E)$ be a graph where $|V| = n$, $|E| = m$ and $G$ contains an odd cycle transversal of size $k$. Then the expected optimization time of the \ea{k} applied to $G$ is bounded by $O(3^k k m n^2\log n)$.
\end{theorem}
\begin{proof}
   Again by Lemma~\ref{lem:infeasible}, after $O(n \log n)$ iterations in expectation, all subsequent solutions are feasible. Let $x = (x_S,x_V)$ be a feasible solution, that is, $|x_S| \le k$ and corresponds to an odd cycle transversal of $G[x_V]$. Let $i \in \{j : x_V[j] = 0\}$ and let $\mathcal{E}_i$ denote the event that after mutation $x_S[i] = x_V[i] = 1$ and no other bits have changed.

   Clearly, $x_S \cup \{i\}$ is also an odd cycle transversal of $G[x_V \cup \{i\}]$ so the resulting offspring conditioned on $\mathcal{E}_i$ must be solution feasible, but not necessarily cardinality feasible. In the latter case, note that we have assumed that $G$ admits an odd cycle transversal $S^{*} \subseteq V$ where $|S^{*}| \le k$. Since odd cycle transversals are closed under induced subgraphs, $(x_V \cup \{i\}) \cap S^{*}$ is an odd cycle transversal of $G[x_V \cup \{i\}]$.

   Let $\mathcal{J}_1$ be the event that in the jump phase of Algorithm~\ref{alg:jump-and-repair-oct}, the process chooses to flip to zero the set of bits $R = (x_S \cup\{i\}) \setminus {S^{*}}$, hence keeping the set $S' = ((x_V \cup \{i\}) \cap S^{*}) \cap (x_S \cup \{i\})$ set to $1$. Let $\mathcal{J}_2$ be the event that $R$ is partitioned exactly into the sets $A,B$ which are each respectively a subset of a bipartition induced by $(x_V \cup \{i\}) \cap S^{*}$. Then $\Pr(\mathcal{J}_1) = (2/3)^{|R|}(1/3)^{|S'|}$ and $\Pr(\mathcal{J}_2 \mid \mathcal{J}_1) = (1/2)^{|R|}$, so we have $\Pr(\mathcal{J}_1 \cap \mathcal{J}_2) = (1/3)^{|R|+|S'|} \ge 3^{-(k+1)}$. Under this joint event, by Lemma~\ref{lem:oct-ic}, Algorithm~\ref{alg:oct-repair} must return an odd cycle transversal for $x_V \cup \{i\}$ of size at most $k$.

   Arguing in the same way as with the proofs of Theorems~\ref{thm:vc} and \ref{thm:fvst}, we may set up a potential function with multiplicative drift
   \[
      \E[X_{t}-X_{t+1}\mid X_t = s] \ge \sum_{i : x_V[i] = 0} \Pr(\mathcal{J}_1\cap\mathcal{J}_2\cap\mathcal{E}_i)
      \ge \frac{|\{i : x_V[i]=0\}|}{4e^23^{k+1}n^2}  = \Omega{\left(\frac{s}{3^kn^2}\right)}.
   \]
   Applying Theorem~\ref{thm:multidrift}, the expected number of iterations of the \ea{k} until an odd cycle transversal is generated for the entire graph $G$ is $O(3^k n^2 \log n)$. Since the repair procedure of Algorithm~\ref{alg:oct-repair} costs $O(k m)$, we obtain the claimed result.
\end{proof}

\section{Experiments}
\label{sec:experiments}
To interpret the concrete running time of \ea{k} on
$k$-\textsc{VertexCover} instances as a function of both $n$ and $k$, we performed a
number of experiments on different instances of
$k$-\textsc{VertexCover}. In order to maintain experimental control
over both $n$ and $k$, we created three graph classes: random planted,
clique/anticlique and biclique. In the random planted class, instances
are randomly generated by drawing each graph from a planted version of
the standard Erd\H{o}s-R\'{e}nyi random graph model in which a
$k$-vertex cover is ``planted'' into the graph and edges are selected for
inclusion with fixed probability $p$ subject to having an end point in
the planted cover.
In particular, each random graph on $n$ vertices with a planted $k$-vertex cover was generated by first randomly choosing $k$ vertices for the vertex cover and then iterating over each vertex pair $v_i, v_j$, such that $i \in \{1, \dots , k \} $ and $j \in \{1, \dots , n \} \setminus{\{i\}}$, and adding the edge $(v_i, v_j)$ with probability $p$.

In the limiting case of $p=1$ we obtain a (nonrandom) clique/anticlique instance comprised of a $k$-clique fully connected to an $n-k$ anticlique, that is, every vertex in the $k$-clique is connected to every other vertex in the graph, and the remaining vertices are connected to every vertex in the $k$-clique (but not each other). Finally, we also investigate bicliques $K_{k,n-k}$, that is,  complete bipartite graphs with $k$ vertices in one of the partitions.

We generated graph instances using values of $n \in \{20,30,\ldots,100\}$ and $k \in \{3,4,\ldots,8\}$. For the nonrandom graphs (clique/anticlique and biclique) we generated an instance for each $n$ and $k$, resulting in 54 instances each of the two classes.
For the random graphs, we also controlled for edge density using $p \in \{0.1,0.25,0.5,0.75\}$, and for each value of $n$, $k$ and $p$, we generated 10 separate random graph instances, resulting in 2160 total random graph instances.

On each instance, we measured the mean (and standard deviation) of the iterations required for the \ea{k} to find a $k$-vertex cover for the entire graph over 100 runs per instance. In Figure~\ref{fig:random-by-n}, we plot the mean running time as a function of $n$ for the random planted instances grouped by edge density $p$. The shaded bands represent the standard deviation from the mean. The runs of the \ea{k} on the nonrandom graphs are displayed in Figure~\ref{fig:nonrandom-by-n}.

\pgfplotsset{%
  cycle list={%    
    {blue,mark=*}, {red,mark=square*}, {brown,mark=o}, {black,mark=x},
    {blue,dashed,mark=diamond*}, {red,dashed,mark=+},
    {brown,dashed,mark=diamond} },
  width=3.15in}

%% RANDOM PLANTED
\begin{figure}
   \begin{subfigure}[b]{0.5\textwidth}
      \begin{tikzpicture}
         \begin{axis}[%
               legend pos=north west,
               xlabel=$n$,
               ylabel={mean run time},
               ymax=80000,
               enlargelimits=false]

            \pgfplotstableread{data/eajar.random_planted.edge_prob0.1-by-n-k8.dat}{\data}

            \addplot+[mark size=1.5pt] table[x=n,y=mean] {\data};
            \addplot[name path=upper,draw=none,forget plot] table[x=n,y   expr=\thisrow{mean}+\thisrow{sd}] {\data};
            \addplot[name path=lower,draw=none,forget plot] table[x=n,y expr=\thisrow{mean}-\thisrow{sd}] {\data};
            \pgfplotsset{cycle list shift=-1}
            \addplot+[fill opacity=0.1,forget plot] fill between[of=upper and lower];
            \pgfplotsset{cycle list shift=0}

            \pgfplotstableread{data/eajar.random_planted.edge_prob0.1-by-n-k6.dat}{\data}
            \addplot+[mark size=1.5pt] table[x=n,y=mean] {\data};
            \addplot[name path=upper,draw=none,forget plot] table[x=n,y   expr=\thisrow{mean}+\thisrow{sd}] {\data};
            \addplot[name path=lower,draw=none,forget plot] table[x=n,y expr=\thisrow{mean}-\thisrow{sd}] {\data};
            \pgfplotsset{cycle list shift=-1}
            \addplot+[fill opacity=0.1,forget plot] fill between[of=upper and lower];
            \pgfplotsset{cycle list shift=0}

            \pgfplotstableread{data/eajar.random_planted.edge_prob0.1-by-n-k4.dat}{\data}
            \addplot+[mark size=1.5pt] table[x=n,y=mean] {\data};
            \addplot[name path=upper,draw=none,forget plot] table[x=n,y   expr=\thisrow{mean}+\thisrow{sd}] {\data};
            \addplot[name path=lower,draw=none,forget plot] table[x=n,y expr=\thisrow{mean}-\thisrow{sd}] {\data};
            \pgfplotsset{cycle list shift=-1}
            \addplot+[fill opacity=0.1,forget plot] fill between[of=upper and lower];
            \pgfplotsset{cycle list shift=0}

            \legend{$k=8$,$k=6$,$k=4$}
         \end{axis}
      \end{tikzpicture}%
      \caption{$p=0.1$}
   \end{subfigure}
   \hfill
   \begin{subfigure}[b]{0.5\textwidth}
      %% p=0.25
      \begin{tikzpicture}
         \begin{axis}[%
               legend pos=north west,
               xlabel=$n$,
               ylabel={mean run time},
               ymax=80000,
               enlargelimits=false]

            \pgfplotstableread{data/eajar.random_planted.edge_prob0.25-by-n-k8.dat}{\data}

            \addplot+[mark size=1.5pt] table[x=n,y=mean] {\data};
            \addplot[name path=upper,draw=none,forget plot] table[x=n,y   expr=\thisrow{mean}+\thisrow{sd}] {\data};
            \addplot[name path=lower,draw=none,forget plot] table[x=n,y expr=\thisrow{mean}-\thisrow{sd}] {\data};
            \pgfplotsset{cycle list shift=-1}
            \addplot+[fill opacity=0.1,forget plot] fill between[of=upper and lower];
            \pgfplotsset{cycle list shift=0}

            \pgfplotstableread{data/eajar.random_planted.edge_prob0.25-by-n-k6.dat}{\data}
            \addplot+[mark size=1.5pt] table[x=n,y=mean] {\data};
            \addplot[name path=upper,draw=none,forget plot] table[x=n,y   expr=\thisrow{mean}+\thisrow{sd}] {\data};
            \addplot[name path=lower,draw=none,forget plot] table[x=n,y expr=\thisrow{mean}-\thisrow{sd}] {\data};
            \pgfplotsset{cycle list shift=-1}
            \addplot+[fill opacity=0.1,forget plot] fill between[of=upper and lower];
            \pgfplotsset{cycle list shift=0}

            \pgfplotstableread{data/eajar.random_planted.edge_prob0.25-by-n-k4.dat}{\data}
            \addplot+[mark size=1.5pt] table[x=n,y=mean] {\data};
            \addplot[name path=upper,draw=none,forget plot] table[x=n,y   expr=\thisrow{mean}+\thisrow{sd}] {\data};
            \addplot[name path=lower,draw=none,forget plot] table[x=n,y expr=\thisrow{mean}-\thisrow{sd}] {\data};
            \pgfplotsset{cycle list shift=-1}
            \addplot+[fill opacity=0.1,forget plot] fill between[of=upper and lower];
            \pgfplotsset{cycle list shift=0}

            \legend{$k=8$,$k=6$,$k=4$}
         \end{axis}
      \end{tikzpicture}%
      \caption{$p=0.25$}
   \end{subfigure}\\
   %% p=0.5
   \begin{subfigure}[b]{0.5\textwidth}
      \begin{tikzpicture}
         \begin{axis}[%
               legend pos=north west,
               xlabel=$n$,
               ylabel={mean run time},
               ymax=80000,
               enlargelimits=false]

            \pgfplotstableread{data/eajar.random_planted-by-n-k8.dat}{\data}

            \addplot+[mark size=1.5pt] table[x=n,y=mean] {\data};
            \addplot[name path=upper,draw=none,forget plot] table[x=n,y   expr=\thisrow{mean}+\thisrow{sd}] {\data};
            \addplot[name path=lower,draw=none,forget plot] table[x=n,y expr=\thisrow{mean}-\thisrow{sd}] {\data};
            \pgfplotsset{cycle list shift=-1}
            \addplot+[fill opacity=0.1,forget plot] fill between[of=upper and lower];
            \pgfplotsset{cycle list shift=0}

            \pgfplotstableread{data/eajar.random_planted-by-n-k6.dat}{\data}
            \addplot+[mark size=1.5pt] table[x=n,y=mean] {\data};
            \addplot[name path=upper,draw=none,forget plot] table[x=n,y   expr=\thisrow{mean}+\thisrow{sd}] {\data};
            \addplot[name path=lower,draw=none,forget plot] table[x=n,y expr=\thisrow{mean}-\thisrow{sd}] {\data};
            \pgfplotsset{cycle list shift=-1}
            \addplot+[fill opacity=0.1,forget plot] fill between[of=upper and lower];
            \pgfplotsset{cycle list shift=0}

            \pgfplotstableread{data/eajar.random_planted-by-n-k4.dat}{\data}
            \addplot+[mark size=1.5pt] table[x=n,y=mean] {\data};
            \addplot[name path=upper,draw=none,forget plot] table[x=n,y   expr=\thisrow{mean}+\thisrow{sd}] {\data};
            \addplot[name path=lower,draw=none,forget plot] table[x=n,y expr=\thisrow{mean}-\thisrow{sd}] {\data};
            \pgfplotsset{cycle list shift=-1}
            \addplot+[fill opacity=0.1,forget plot] fill between[of=upper and lower];
            \pgfplotsset{cycle list shift=0}

            \legend{$k=8$,$k=6$,$k=4$}
         \end{axis}
      \end{tikzpicture}%
      \caption{$p=0.5$}
   \end{subfigure}
   \hfill
   %%p=0.75
   \begin{subfigure}[b]{0.5\textwidth}
      \begin{tikzpicture}
         \begin{axis}[%
               legend pos=north west,
               xlabel=$n$,
               ylabel={mean run time},
               ymax=80000,
               enlargelimits=false]
            \pgfplotstableread{data/eajar.random_planted.edge_prob0.75-by-n-k8.dat}{\data}

            \addplot+[mark size=1.5pt] table[x=n,y=mean] {\data};
            \addplot[name path=upper,draw=none,forget plot] table[x=n,y   expr=\thisrow{mean}+\thisrow{sd}] {\data};
            \addplot[name path=lower,draw=none,forget plot] table[x=n,y expr=\thisrow{mean}-\thisrow{sd}] {\data};
            \pgfplotsset{cycle list shift=-1}
            \addplot+[fill opacity=0.1,forget plot] fill between[of=upper and lower];
            \pgfplotsset{cycle list shift=0}

            \pgfplotstableread{data/eajar.random_planted.edge_prob0.75-by-n-k6.dat}{\data}
            \addplot+[mark size=1.5pt] table[x=n,y=mean] {\data};
            \addplot[name path=upper,draw=none,forget plot] table[x=n,y   expr=\thisrow{mean}+\thisrow{sd}] {\data};
            \addplot[name path=lower,draw=none,forget plot] table[x=n,y expr=\thisrow{mean}-\thisrow{sd}] {\data};
            \pgfplotsset{cycle list shift=-1}
            \addplot+[fill opacity=0.1,forget plot] fill between[of=upper and lower];
            \pgfplotsset{cycle list shift=0}

            \pgfplotstableread{data/eajar.random_planted.edge_prob0.75-by-n-k4.dat}{\data}
            \addplot+[mark size=1.5pt] table[x=n,y=mean] {\data};
            \addplot[name path=upper,draw=none,forget plot] table[x=n,y   expr=\thisrow{mean}+\thisrow{sd}] {\data};
            \addplot[name path=lower,draw=none,forget plot] table[x=n,y expr=\thisrow{mean}-\thisrow{sd}] {\data};
            \pgfplotsset{cycle list shift=-1}
            \addplot+[fill opacity=0.1,forget plot] fill between[of=upper and lower];
            \pgfplotsset{cycle list shift=0}

            \legend{$k=8$,$k=6$,$k=4$}
         \end{axis}
      \end{tikzpicture}%
      \caption{$p=0.75$}
   \end{subfigure}
   \caption{\label{fig:random-by-n} Mean running time of \ea{k} on random planted vertex cover instances for various edge probabilities $p$. Shaded bands represent standard deviation.}
\end{figure}
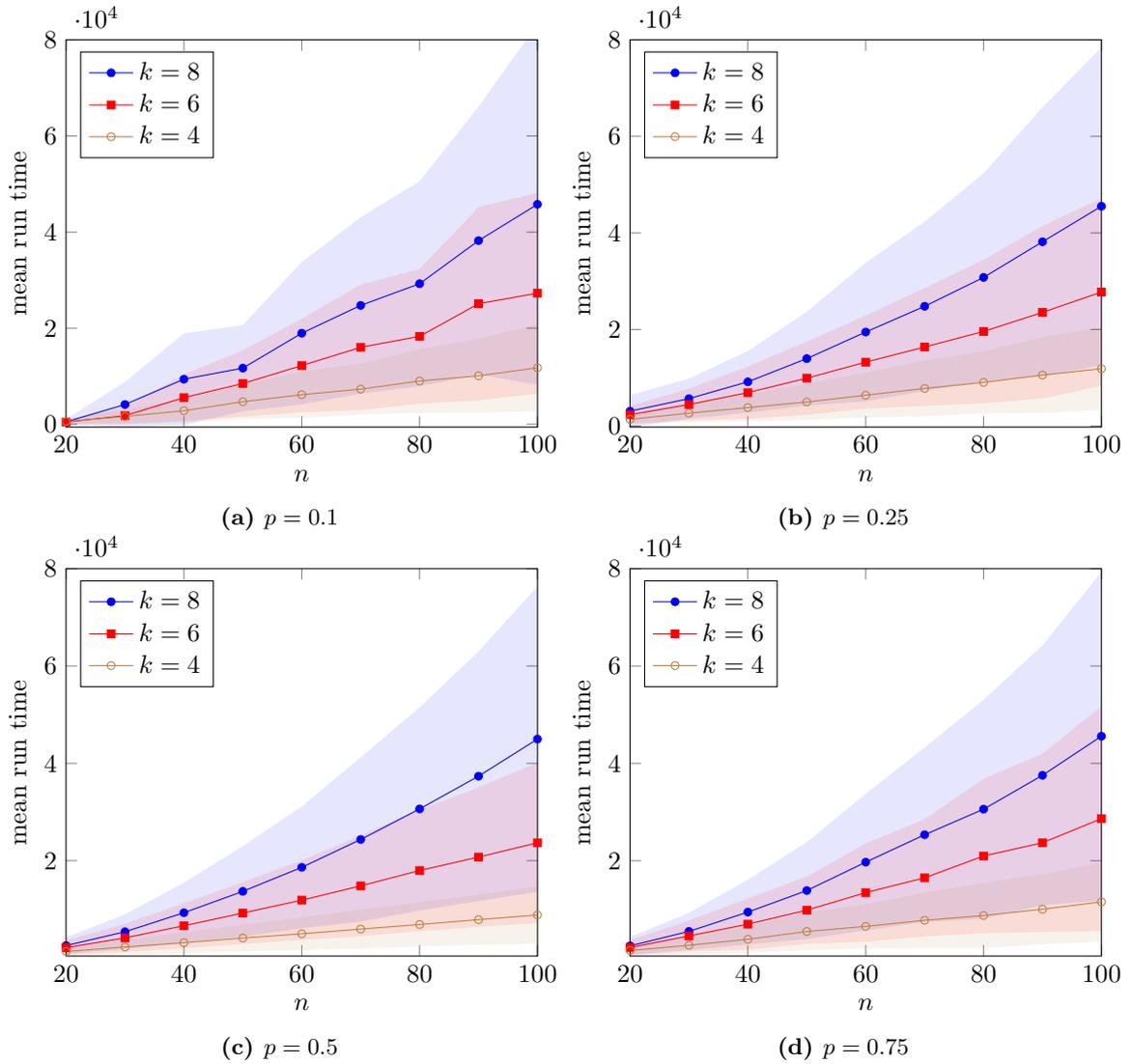

\begin{figure}
   \begin{subfigure}[b]{0.5\textwidth}
      \begin{tikzpicture}
         \begin{axis}[%
               legend pos=north west,
               xlabel=$n$,
               ylabel={mean run time},
               ymax=80000,
               enlargelimits=false]

            \pgfplotstableread{data/eajar.clique-anticlique-by-n-k8.dat}{\data}

            \addplot+[mark size=1.5pt] table[x=n,y=mean] {\data};
            \addplot[name path=upper,draw=none,forget plot] table[x=n,y   expr=\thisrow{mean}+\thisrow{sd}] {\data};
            \addplot[name path=lower,draw=none,forget plot] table[x=n,y expr=\thisrow{mean}-\thisrow{sd}] {\data};
            \pgfplotsset{cycle list shift=-1}
            \addplot+[fill opacity=0.1,forget plot] fill between[of=upper and lower];
            \pgfplotsset{cycle list shift=0}

            \pgfplotstableread{data/eajar.clique-anticlique-by-n-k6.dat}{\data}
            \addplot+[mark size=1.5pt] table[x=n,y=mean] {\data};
            \addplot[name path=upper,draw=none,forget plot] table[x=n,y   expr=\thisrow{mean}+\thisrow{sd}] {\data};
            \addplot[name path=lower,draw=none,forget plot] table[x=n,y expr=\thisrow{mean}-\thisrow{sd}] {\data};
            \pgfplotsset{cycle list shift=-1}
            \addplot+[fill opacity=0.1,forget plot] fill between[of=upper and lower];
            \pgfplotsset{cycle list shift=0}

            \pgfplotstableread{data/eajar.clique-anticlique-by-n-k4.dat}{\data}
            \addplot+[mark size=1.5pt] table[x=n,y=mean] {\data};
            \addplot[name path=upper,draw=none,forget plot] table[x=n,y   expr=\thisrow{mean}+\thisrow{sd}] {\data};
            \addplot[name path=lower,draw=none,forget plot] table[x=n,y expr=\thisrow{mean}-\thisrow{sd}] {\data};
            \pgfplotsset{cycle list shift=-1}
            \addplot+[fill opacity=0.1,forget plot] fill between[of=upper and lower];
            \pgfplotsset{cycle list shift=0}

            \legend{$k=8$,$k=6$,$k=4$}
         \end{axis}
      \end{tikzpicture}%

      \caption[]{clique/anticlique}
   \end{subfigure}
   \hfill
   %%% biclique
   \begin{subfigure}[b]{0.5\textwidth}
      \begin{tikzpicture}
         \begin{axis}[%
               legend pos=north west,
               xlabel=$n$,
               ylabel={mean run time},
               ymax=80000,
               enlargelimits=false]

            \pgfplotstableread{data/eajar.biclique-by-n-k8.dat}{\data}

            \addplot+[mark size=1.5pt] table[x=n,y=mean] {\data};
            \addplot[name path=upper,draw=none,forget plot] table[x=n,y   expr=\thisrow{mean}+\thisrow{sd}] {\data};
            \addplot[name path=lower,draw=none,forget plot] table[x=n,y expr=\thisrow{mean}-\thisrow{sd}] {\data};
            \pgfplotsset{cycle list shift=-1}
            \addplot+[fill opacity=0.1,forget plot] fill between[of=upper and lower];
            \pgfplotsset{cycle list shift=0}

            \pgfplotstableread{data/eajar.biclique-by-n-k6.dat}{\data}
            \addplot+[mark size=1.5pt] table[x=n,y=mean] {\data};
            \addplot[name path=upper,draw=none,forget plot] table[x=n,y   expr=\thisrow{mean}+\thisrow{sd}] {\data};
            \addplot[name path=lower,draw=none,forget plot] table[x=n,y expr=\thisrow{mean}-\thisrow{sd}] {\data};
            \pgfplotsset{cycle list shift=-1}
            \addplot+[fill opacity=0.1,forget plot] fill between[of=upper and lower];
            \pgfplotsset{cycle list shift=0}

            \pgfplotstableread{data/eajar.biclique-by-n-k4.dat}{\data}
            \addplot+[mark size=1.5pt] table[x=n,y=mean] {\data};
            \addplot[name path=upper,draw=none,forget plot] table[x=n,y   expr=\thisrow{mean}+\thisrow{sd}] {\data};
            \addplot[name path=lower,draw=none,forget plot] table[x=n,y expr=\thisrow{mean}-\thisrow{sd}] {\data};
            \pgfplotsset{cycle list shift=-1}
            \addplot+[fill opacity=0.1,forget plot] fill between[of=upper and lower];
            \pgfplotsset{cycle list shift=0}

            \legend{$k=8$,$k=6$,$k=4$}
         \end{axis}
      \end{tikzpicture}%
      \caption[]{biclique}
   \end{subfigure}
   \caption{\label{fig:nonrandom-by-n} Mean running time of \ea{k} on nonrandom vertex cover instances.
      Shaded bands represent standard deviation.}
\end{figure}
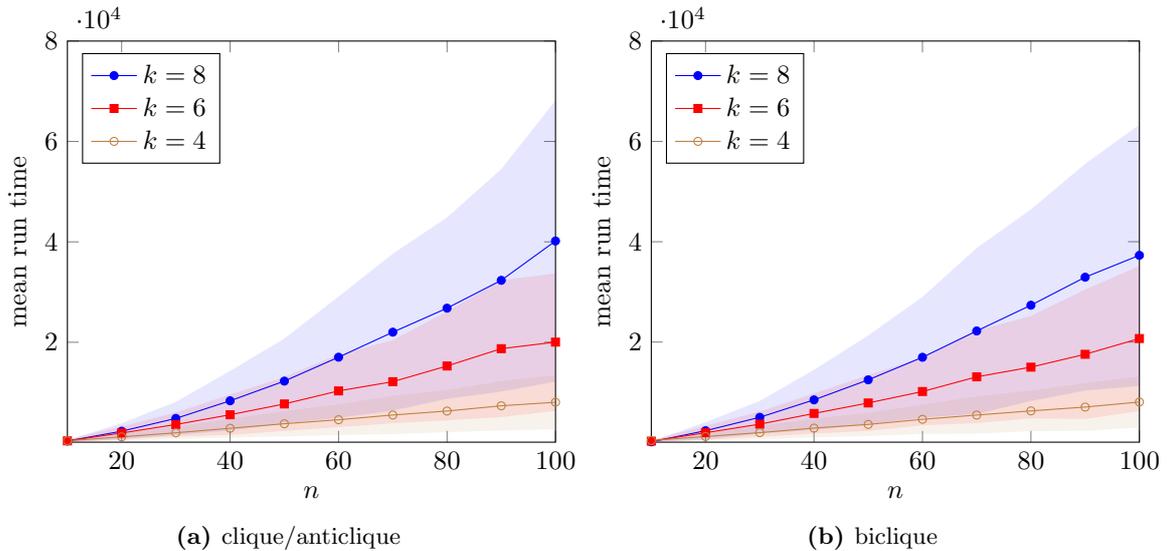

%%%%%%%%%%%%%%%%%%%%%%%%%%%%%%%%%%%%%%%%
%% DATA NORMALIZED BY A FUNCTION OF n %%
%%%%%%%%%%%%%%%%%%%%%%%%%%%%%%%%%%%%%%%%

%% Normalization function:
% \def\norm{(\thisrow{n}*ln(\thisrow{n}*ln(\thisrow{n}))}
%\def\norm{(\thisrow{n}*ln(\thisrow{n}*ln(\thisrow{n}))}
%\def\norm{(\thisrow{n}*ln(\thisrow{n})}
%\def\norm{(\thisrow{n}*sqrt(\thisrow{n})*ln(\thisrow{n})}
%\def\norm{(\thisrow{n}*\thisrow{n})}
\def\norm{(\thisrow{n}*\thisrow{n}*ln(\thisrow{n})}

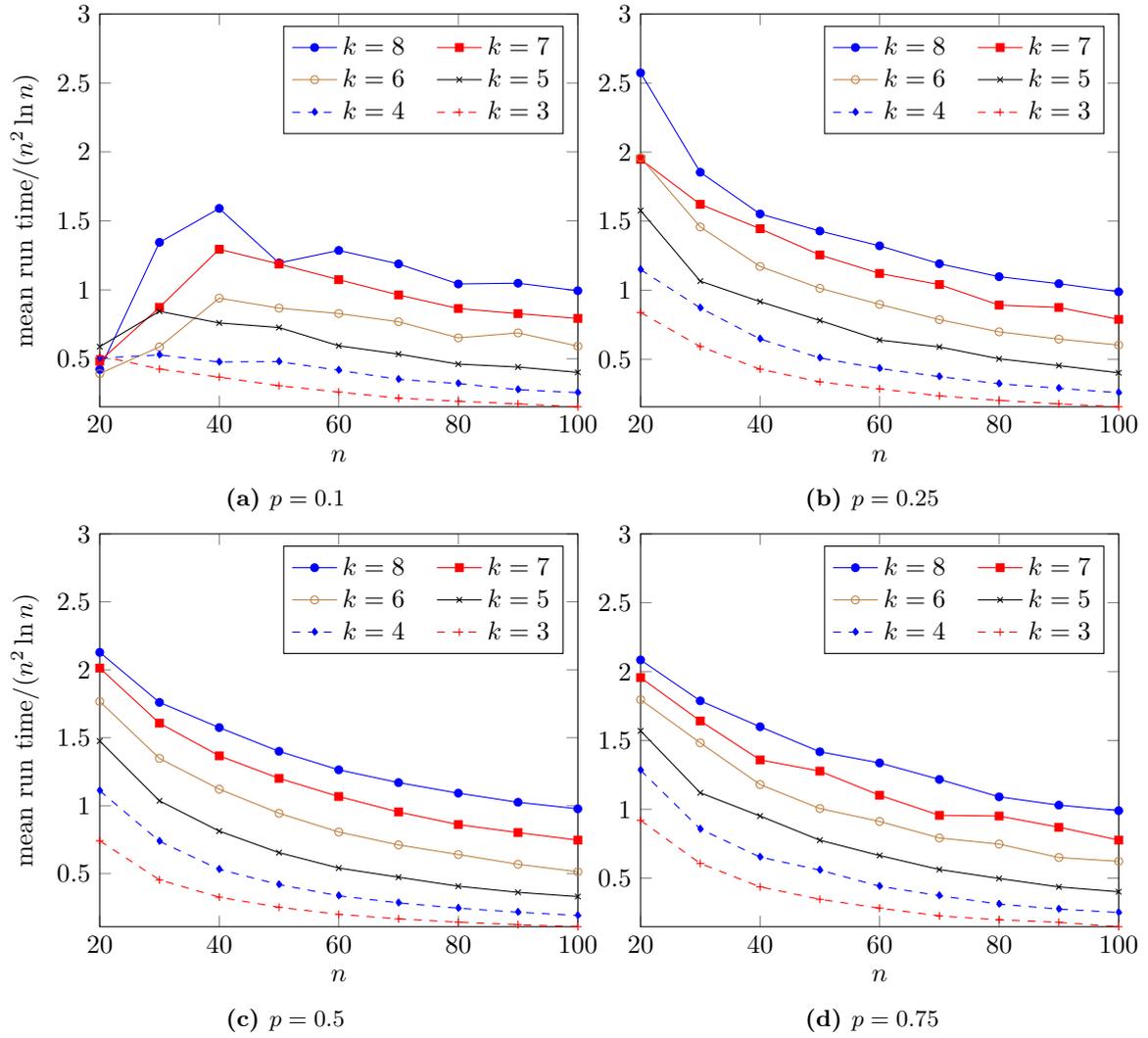
\begin{figure}
   \begin{subfigure}{0.5\textwidth}
      \begin{tikzpicture}
         \begin{axis}[%
           legend pos=north east,
           legend columns=2,
           legend style={
           /tikz/column 2/.style={
                column sep=10pt,
              },
              },
               xlabel=$n$,
               ylabel near ticks,
               ylabel={mean run time$/(n^2 \ln n)$},
               ymax=3,
               enlargelimits=false]
            \pgfplotstableread{data/eajar.random_planted.edge_prob0.1-by-n-k8.dat}{\data}
            \addplot+[mark size=1.5pt] table[x=n,y expr=\thisrow{mean}/\norm] {\data};

            \pgfplotstableread{data/eajar.random_planted.edge_prob0.1-by-n-k7.dat}{\data}
            \addplot+[mark size=1.5pt] table[x=n,y expr=\thisrow{mean}/\norm] {\data};

            \pgfplotstableread{data/eajar.random_planted.edge_prob0.1-by-n-k6.dat}{\data}
            \addplot+[mark size=1.5pt] table[x=n,y expr=\thisrow{mean}/\norm] {\data};

            \pgfplotstableread{data/eajar.random_planted.edge_prob0.1-by-n-k5.dat}{\data}
            \addplot+[mark size=1.5pt] table[x=n,y expr=\thisrow{mean}/\norm] {\data};

            \pgfplotstableread{data/eajar.random_planted.edge_prob0.1-by-n-k4.dat}{\data}
            \addplot+[mark size=1.5pt] table[x=n,y expr=\thisrow{mean}/\norm] {\data};

            \pgfplotstableread{data/eajar.random_planted.edge_prob0.1-by-n-k3.dat}{\data}
            \addplot+[mark size=1.5pt] table[x=n,y expr=\thisrow{mean}/\norm] {\data};
            \legend{$k=8$,$k=7$,$k=6$,$k=5$,$k=4$,$k=3$}
         \end{axis}
      \end{tikzpicture}%
      \caption[]{$p=0.1$}
    \end{subfigure}
    % \hfill
    \hspace*{2mm}%
    \hspace*{\fill}%
    \begin{subfigure}{0.5\linewidth}
      \begin{tikzpicture}
         \begin{axis}[%
           legend pos=north east,
                      legend columns=2,
           legend style={
           /tikz/column 2/.style={
                column sep=10pt,
              },
              },
               xlabel=$n$,
               ymax=3,
               enlargelimits=false]
            \pgfplotstableread{data/eajar.random_planted.edge_prob0.25-by-n-k8.dat}{\data}
            \addplot+[mark size=1.5pt] table[x=n,y expr=\thisrow{mean}/\norm] {\data};

            \pgfplotstableread{data/eajar.random_planted.edge_prob0.25-by-n-k7.dat}{\data}
            \addplot+[mark size=1.5pt] table[x=n,y expr=\thisrow{mean}/\norm] {\data};

            \pgfplotstableread{data/eajar.random_planted.edge_prob0.25-by-n-k6.dat}{\data}
            \addplot+[mark size=1.5pt] table[x=n,y expr=\thisrow{mean}/\norm] {\data};

            \pgfplotstableread{data/eajar.random_planted.edge_prob0.25-by-n-k5.dat}{\data}
            \addplot+[mark size=1.5pt] table[x=n,y expr=\thisrow{mean}/\norm] {\data};

            \pgfplotstableread{data/eajar.random_planted.edge_prob0.25-by-n-k4.dat}{\data}
            \addplot+[mark size=1.5pt] table[x=n,y expr=\thisrow{mean}/\norm] {\data};

            \pgfplotstableread{data/eajar.random_planted.edge_prob0.25-by-n-k3.dat}{\data}
            \addplot+[mark size=1.5pt] table[x=n,y expr=\thisrow{mean}/\norm] {\data};
            \legend{$k=8$,$k=7$,$k=6$,$k=5$,$k=4$,$k=3$}
         \end{axis}
      \end{tikzpicture}%
      \caption[]{$p=0.25$}
   \end{subfigure}\\
   \begin{subfigure}{0.5\textwidth}
      \begin{tikzpicture}
         \begin{axis}[%
               legend pos=north east,
               xlabel=$n$,
                          legend columns=2,
           legend style={
           /tikz/column 2/.style={
                column sep=10pt,
              },
              },
               ylabel={mean run time$/(n^2 \ln n)$},
               ymax=3,
               enlargelimits=false]
            \pgfplotstableread{data/eajar.random_planted-by-n-k8.dat}{\data}
            \addplot+[mark size=1.5pt] table[x=n,y expr=\thisrow{mean}/\norm] {\data};

            \pgfplotstableread{data/eajar.random_planted-by-n-k7.dat}{\data}
            \addplot+[mark size=1.5pt] table[x=n,y expr=\thisrow{mean}/\norm] {\data};

            \pgfplotstableread{data/eajar.random_planted-by-n-k6.dat}{\data}
            \addplot+[mark size=1.5pt] table[x=n,y expr=\thisrow{mean}/\norm] {\data};

            \pgfplotstableread{data/eajar.random_planted-by-n-k5.dat}{\data}
            \addplot+[mark size=1.5pt] table[x=n,y expr=\thisrow{mean}/\norm] {\data};

            \pgfplotstableread{data/eajar.random_planted-by-n-k4.dat}{\data}
            \addplot+[mark size=1.5pt] table[x=n,y expr=\thisrow{mean}/\norm] {\data};

            \pgfplotstableread{data/eajar.random_planted-by-n-k3.dat}{\data}
            \addplot+[mark size=1.5pt] table[x=n,y expr=\thisrow{mean}/\norm] {\data};

            \legend{$k=8$,$k=7$,$k=6$,$k=5$,$k=4$,$k=3$}
         \end{axis}
      \end{tikzpicture}%
      \caption[]{$p=0.5$}
   \end{subfigure}
   % \hfill
   \hspace*{2mm}%
   \begin{subfigure}{0.5\textwidth}
      \begin{tikzpicture}
         \begin{axis}[%
               legend pos=north east,
               xlabel=$n$,
                          legend columns=2,
           legend style={
           /tikz/column 2/.style={
                column sep=10pt,
              },
              },
              ymax=3,
              enlargelimits=false]
            \pgfplotstableread{data/eajar.random_planted.edge_prob0.75-by-n-k8.dat}{\data}
            \addplot+[mark size=1.5pt] table[x=n,y expr=\thisrow{mean}/\norm] {\data};

            \pgfplotstableread{data/eajar.random_planted.edge_prob0.75-by-n-k7.dat}{\data}
            \addplot+[mark size=1.5pt] table[x=n,y expr=\thisrow{mean}/\norm] {\data};

            \pgfplotstableread{data/eajar.random_planted.edge_prob0.75-by-n-k6.dat}{\data}
            \addplot+[mark size=1.5pt] table[x=n,y expr=\thisrow{mean}/\norm] {\data};

            \pgfplotstableread{data/eajar.random_planted.edge_prob0.75-by-n-k5.dat}{\data}
            \addplot+[mark size=1.5pt] table[x=n,y expr=\thisrow{mean}/\norm] {\data};

            \pgfplotstableread{data/eajar.random_planted.edge_prob0.75-by-n-k4.dat}{\data}
            \addplot+[mark size=1.5pt] table[x=n,y expr=\thisrow{mean}/\norm] {\data};

            \pgfplotstableread{data/eajar.random_planted.edge_prob0.75-by-n-k3.dat}{\data}
            \addplot+[mark size=1.5pt] table[x=n,y expr=\thisrow{mean}/\norm] {\data};

            \legend{$k=8$,$k=7$,$k=6$,$k=5$,$k=4$,$k=3$}
         \end{axis}
      \end{tikzpicture}%
      \caption[]{$p=0.75$}
   \end{subfigure}
   \caption{\label{fig:random-by-n-normalized} Mean running time of \ea{k} divided by $n^2 \ln n$ on random planted vertex cover instances for various edge probabilities $p$.}
\end{figure}

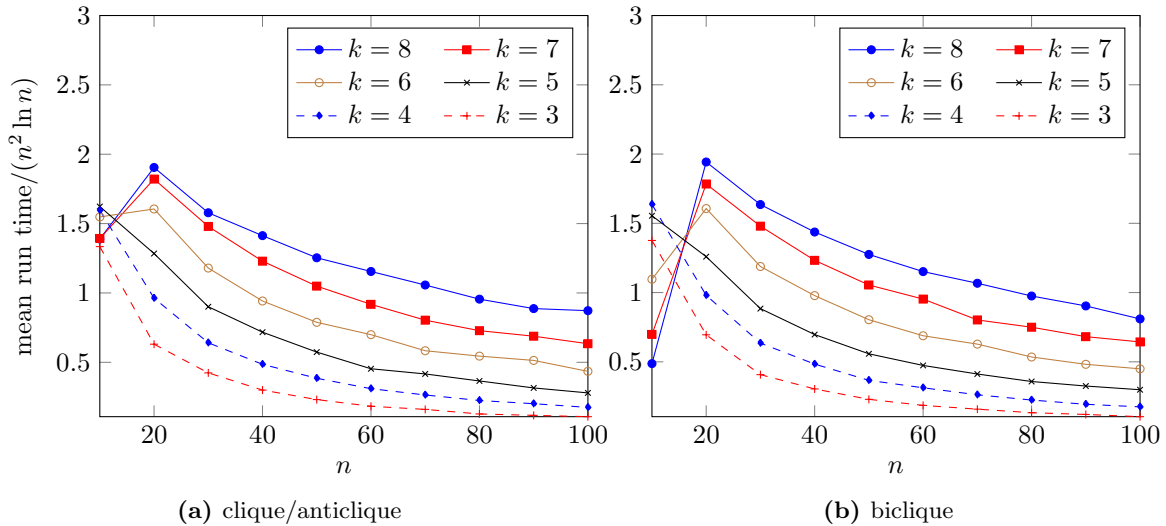
\begin{figure}
   \begin{subfigure}{0.5\linewidth}
      \begin{tikzpicture}
         \begin{axis}[%
               legend pos=north east,
               xlabel=$n$,
                          legend columns=2,
           legend style={
           /tikz/column 2/.style={
                column sep=10pt,
              },
              },
               ylabel={mean run time$/(n^2 \ln n)$},
               ymax=3,
               enlargelimits=false]

            \pgfplotstableread{data/eajar.clique-anticlique-by-n-k8.dat}{\data}
            \addplot+[mark size=1.5pt] table[x=n,y expr=\thisrow{mean}/\norm] {\data};

            \pgfplotstableread{data/eajar.clique-anticlique-by-n-k7.dat}{\data}
            \addplot+[mark size=1.5pt] table[x=n,y expr=\thisrow{mean}/\norm] {\data};

            \pgfplotstableread{data/eajar.clique-anticlique-by-n-k6.dat}{\data}
            \addplot+[mark size=1.5pt] table[x=n,y expr=\thisrow{mean}/\norm] {\data};

            \pgfplotstableread{data/eajar.clique-anticlique-by-n-k5.dat}{\data}
            \addplot+[mark size=1.5pt] table[x=n,y expr=\thisrow{mean}/\norm] {\data};

            \pgfplotstableread{data/eajar.clique-anticlique-by-n-k4.dat}{\data}
            \addplot+[mark size=1.5pt] table[x=n,y expr=\thisrow{mean}/\norm] {\data};

            \pgfplotstableread{data/eajar.clique-anticlique-by-n-k3.dat}{\data}
            \addplot+[mark size=1.5pt] table[x=n,y expr=\thisrow{mean}/\norm] {\data};

            \legend{$k=8$,$k=7$,$k=6$,$k=5$,$k=4$,$k=3$}
         \end{axis}
      \end{tikzpicture}%
      \caption[]{clique/anticlique}
   \end{subfigure}
   % \hfill
   \hspace*{2mm}%
   \begin{subfigure}{0.5\linewidth}
      \begin{tikzpicture}
         \begin{axis}[%
               legend pos=north east,
               xlabel=$n$,
                          legend columns=2,
           legend style={
           /tikz/column 2/.style={
                column sep=10pt,
              },
              },
              ymax=3,
              enlargelimits=false]

            \pgfplotstableread{data/eajar.biclique-by-n-k8.dat}{\data}
            \addplot+[mark size=1.5pt] table[x=n,y expr=\thisrow{mean}/\norm] {\data};

            \pgfplotstableread{data/eajar.biclique-by-n-k7.dat}{\data}
            \addplot+[mark size=1.5pt] table[x=n,y expr=\thisrow{mean}/\norm] {\data};

            \pgfplotstableread{data/eajar.biclique-by-n-k6.dat}{\data}
            \addplot+[mark size=1.5pt] table[x=n,y expr=\thisrow{mean}/\norm] {\data};

            \pgfplotstableread{data/eajar.biclique-by-n-k5.dat}{\data}
            \addplot+[mark size=1.5pt] table[x=n,y expr=\thisrow{mean}/\norm] {\data};

            \pgfplotstableread{data/eajar.biclique-by-n-k4.dat}{\data}
            \addplot+[mark size=1.5pt] table[x=n,y expr=\thisrow{mean}/\norm] {\data};

            \pgfplotstableread{data/eajar.biclique-by-n-k3.dat}{\data}
            \addplot+[mark size=1.5pt] table[x=n,y expr=\thisrow{mean}/\norm] {\data};

            \legend{$k=8$,$k=7$,$k=6$,$k=5$,$k=4$,$k=3$}
         \end{axis}
      \end{tikzpicture}%
      \caption[]{biclique}

   \end{subfigure}
   \caption[]{\label{fig:nonrandom-by-n-normalized} Mean run time divided by $n^2 \ln n$ of \ea{k} on nonrandom vertex cover instances.}

\end{figure}

To compare the empirical running time to the bound proved in Theorem~\ref{thm:vc} and assess the magnitude of hidden constants, we plot the mean running time normalized by $n^2 \ln n$ in Figures~\ref{fig:random-by-n-normalized} and~\ref{fig:nonrandom-by-n-normalized}. The results here suggest that not only is the normalized running time bounded above by a fixed value that depends only on $k$, as predicted, but also that the bound in Theorem~\ref{thm:vc}
seems to be too large for random planted instances, clique/anticlique and biclique. While this may suggest that the upper bound could be further tightened, we point out that Theorem~\ref{thm:vc} establishes a worst-case running time bound, and the studied instance classes, while easier to manage and control, are also likely to be simpler to solve.

In order to observe the dependence of the running time on $k$ for fixed $n$, we plot the mean running time (again normalized by $n^2 \ln n$) as a function of $k$ in Figures~\ref{fig:random-by-k} and~\ref{fig:nonrandom-by-k}. We compare these results with a plot of $2^k/100$ and observe that the growth with $k$ appears to be even subexponential. The constant $1/100$ is used as a scaling factor for the comparison.

%%% dependence on k
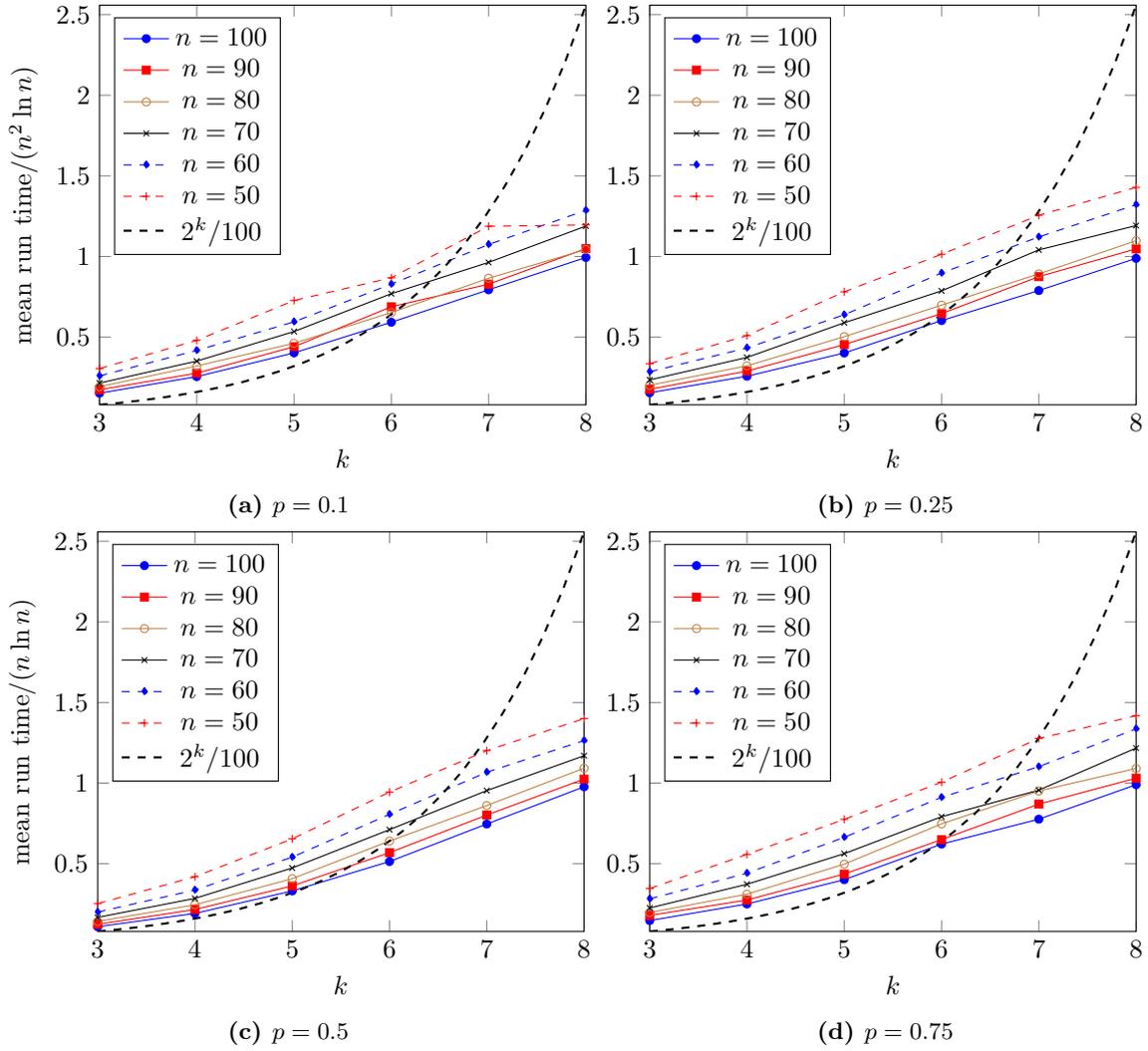
\begin{figure}
   \begin{subfigure}[b]{0.5\linewidth}
      \begin{tikzpicture}
         \begin{axis}[%
               legend pos=north west,
               xlabel=$k$,
               ylabel={mean run time$/(n^2 \ln n)$},
               enlargelimits=false]
            \pgfplotstableread{data/eajar.random_planted.edge_prob0.1-by-k-n100.dat}{\data}
            \addplot+[mark size=1.5pt] table[x=k,y expr=\thisrow{mean}/\norm] {\data};

            \pgfplotstableread{data/eajar.random_planted.edge_prob0.1-by-k-n90.dat}{\data}
            \addplot+[mark size=1.5pt] table[x=k,y expr=\thisrow{mean}/\norm] {\data};

            \pgfplotstableread{data/eajar.random_planted.edge_prob0.1-by-k-n80.dat}{\data}
            \addplot+[mark size=1.5pt] table[x=k,y expr=\thisrow{mean}/\norm] {\data};

            \pgfplotstableread{data/eajar.random_planted.edge_prob0.1-by-k-n70.dat}{\data}
            \addplot+[mark size=1.5pt] table[x=k,y expr=\thisrow{mean}/\norm] {\data};

            \pgfplotstableread{data/eajar.random_planted.edge_prob0.1-by-k-n60.dat}{\data}
            \addplot+[mark size=1.5pt] table[x=k,y expr=\thisrow{mean}/\norm] {\data};

            \pgfplotstableread{data/eajar.random_planted.edge_prob0.1-by-k-n50.dat}{\data}
            \addplot+[mark size=1.5pt] table[x=k,y expr=\thisrow{mean}/\norm] {\data};

            \addplot[thick,dashed,black,domain=3:8] {2^x/100};

            \legend{$n=100$,$n=90$,$n=80$,$n=70$,$n=60$,$n=50$,$2^k/100$}
         \end{axis}
      \end{tikzpicture}
      \caption{$p=0.1$}
   \end{subfigure}
   %\hfill
   \hspace*{2mm}%
   \begin{subfigure}[b]{0.5\linewidth}
      \begin{tikzpicture}
         \begin{axis}[%
               legend pos=north west,
               xlabel=$k$,
               enlargelimits=false]
            \pgfplotstableread{data/eajar.random_planted.edge_prob0.25-by-k-n100.dat}{\data}
            \addplot+[mark size=1.5pt] table[x=k,y expr=\thisrow{mean}/\norm] {\data};

            \pgfplotstableread{data/eajar.random_planted.edge_prob0.25-by-k-n90.dat}{\data}
            \addplot+[mark size=1.5pt] table[x=k,y expr=\thisrow{mean}/\norm] {\data};

            \pgfplotstableread{data/eajar.random_planted.edge_prob0.25-by-k-n80.dat}{\data}
            \addplot+[mark size=1.5pt] table[x=k,y expr=\thisrow{mean}/\norm] {\data};

            \pgfplotstableread{data/eajar.random_planted.edge_prob0.25-by-k-n70.dat}{\data}
            \addplot+[mark size=1.5pt] table[x=k,y expr=\thisrow{mean}/\norm] {\data};

            \pgfplotstableread{data/eajar.random_planted.edge_prob0.25-by-k-n60.dat}{\data}
            \addplot+[mark size=1.5pt] table[x=k,y expr=\thisrow{mean}/\norm] {\data};

            \pgfplotstableread{data/eajar.random_planted.edge_prob0.25-by-k-n50.dat}{\data}
            \addplot+[mark size=1.5pt] table[x=k,y expr=\thisrow{mean}/\norm] {\data};

            \addplot[thick,dashed,black,domain=3:8] {2^x/100};

            \legend{$n=100$,$n=90$,$n=80$,$n=70$,$n=60$,$n=50$,$2^k/100$}
         \end{axis}
      \end{tikzpicture}
      \caption{$p=0.25$}
   \end{subfigure}\\
   \begin{subfigure}[b]{0.5\linewidth}
      \begin{tikzpicture}
         \begin{axis}[%
               legend pos=north west,
               xlabel=$k$,
               ylabel={mean run time$/(n \ln n)$},
               enlargelimits=false]
            \pgfplotstableread{data/eajar.random_planted-by-k-n100.dat}{\data}
            \addplot+[mark size=1.5pt] table[x=k,y expr=\thisrow{mean}/\norm] {\data};
            \pgfplotstableread{data/eajar.random_planted-by-k-n90.dat}{\data}
            \addplot+[mark size=1.5pt] table[x=k,y expr=\thisrow{mean}/\norm] {\data};
            \pgfplotstableread{data/eajar.random_planted-by-k-n80.dat}{\data}
            \addplot+[mark size=1.5pt] table[x=k,y expr=\thisrow{mean}/\norm] {\data};
            \pgfplotstableread{data/eajar.random_planted-by-k-n70.dat}{\data}
            \addplot+[mark size=1.5pt] table[x=k,y expr=\thisrow{mean}/\norm] {\data};
            \pgfplotstableread{data/eajar.random_planted-by-k-n60.dat}{\data}
            \addplot+[mark size=1.5pt] table[x=k,y expr=\thisrow{mean}/\norm] {\data};
            \pgfplotstableread{data/eajar.random_planted-by-k-n50.dat}{\data}
            \addplot+[mark size=1.5pt] table[x=k,y expr=\thisrow{mean}/\norm] {\data};

            \addplot[thick,dashed,black,domain=3:8] {2^x/100};

            \legend{$n=100$,$n=90$,$n=80$,$n=70$,$n=60$,$n=50$,$2^k/100$}
         \end{axis}
      \end{tikzpicture}
      \caption{$p=0.5$}
   \end{subfigure}
   %\hfill
   \hspace*{2mm}%
   \begin{subfigure}[b]{0.5\linewidth}
      \begin{tikzpicture}
         \begin{axis}[%
               legend pos=north west,
               xlabel=$k$,
               enlargelimits=false]
            \pgfplotstableread{data/eajar.random_planted.edge_prob0.75-by-k-n100.dat}{\data}
            \addplot+[mark size=1.5pt] table[x=k,y expr=\thisrow{mean}/\norm] {\data};

            \pgfplotstableread{data/eajar.random_planted.edge_prob0.75-by-k-n90.dat}{\data}
            \addplot+[mark size=1.5pt] table[x=k,y expr=\thisrow{mean}/\norm] {\data};

            \pgfplotstableread{data/eajar.random_planted.edge_prob0.75-by-k-n80.dat}{\data}
            \addplot+[mark size=1.5pt] table[x=k,y expr=\thisrow{mean}/\norm] {\data};

            \pgfplotstableread{data/eajar.random_planted.edge_prob0.75-by-k-n70.dat}{\data}
            \addplot+[mark size=1.5pt] table[x=k,y expr=\thisrow{mean}/\norm] {\data};

            \pgfplotstableread{data/eajar.random_planted.edge_prob0.75-by-k-n60.dat}{\data}
            \addplot+[mark size=1.5pt] table[x=k,y expr=\thisrow{mean}/\norm] {\data};

            \pgfplotstableread{data/eajar.random_planted.edge_prob0.75-by-k-n50.dat}{\data}
            \addplot+[mark size=1.5pt] table[x=k,y expr=\thisrow{mean}/\norm] {\data};

            \addplot[thick,dashed,black,domain=3:8] {2^x/100};

            \legend{$n=100$,$n=90$,$n=80$,$n=70$,$n=60$,$n=50$,$2^k/100$}
         \end{axis}
      \end{tikzpicture}
      \caption{$p=0.75$}
   \end{subfigure}

   \caption[]{\label{fig:random-by-k} Mean run time of \ea{k} divided by $n^2 \ln n$ on random planted vertex cover instances as a function of $k$ for various edge probabilities $p$.}
\end{figure}

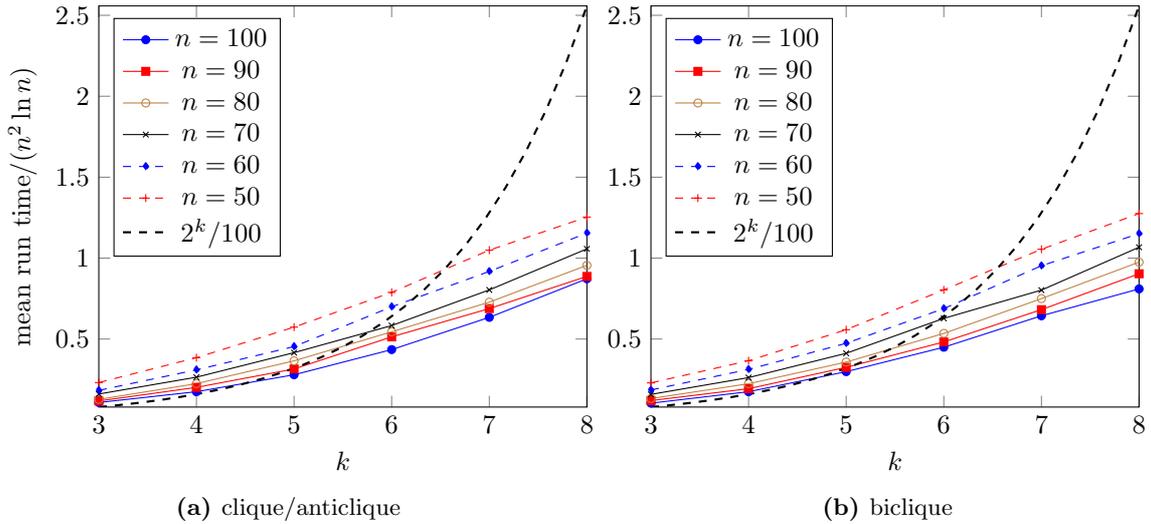
\begin{figure}
   \begin{subfigure}[b]{0.5\linewidth}
      \begin{tikzpicture}
         \begin{axis}[%
               legend pos=north west,
               xlabel=$k$,
               ylabel={mean run time$/(n^2 \ln n)$},
               enlargelimits=false]
            \pgfplotstableread{data/eajar.clique-anticlique-by-k-n100.dat}{\data}
            \addplot+[mark size=1.5pt] table[x=k,y expr=\thisrow{mean}/\norm] {\data};

            \pgfplotstableread{data/eajar.clique-anticlique-by-k-n90.dat}{\data}
            \addplot+[mark size=1.5pt] table[x=k,y expr=\thisrow{mean}/\norm] {\data};

            \pgfplotstableread{data/eajar.clique-anticlique-by-k-n80.dat}{\data}
            \addplot+[mark size=1.5pt] table[x=k,y expr=\thisrow{mean}/\norm] {\data};

            \pgfplotstableread{data/eajar.clique-anticlique-by-k-n70.dat}{\data}
            \addplot+[mark size=1.5pt] table[x=k,y expr=\thisrow{mean}/\norm] {\data};

            \pgfplotstableread{data/eajar.clique-anticlique-by-k-n60.dat}{\data}
            \addplot+[mark size=1.5pt] table[x=k,y expr=\thisrow{mean}/\norm] {\data};

            \pgfplotstableread{data/eajar.clique-anticlique-by-k-n50.dat}{\data}
            \addplot+[mark size=1.5pt] table[x=k,y expr=\thisrow{mean}/\norm] {\data};

            \addplot[thick,dashed,black,domain=3:8] {2^x/100};

            \legend{$n=100$,$n=90$,$n=80$,$n=70$,$n=60$,$n=50$,$2^k/100$}
         \end{axis}
      \end{tikzpicture}
      \caption{clique/anticlique}
   \end{subfigure}
   \hspace*{2mm}%
   \begin{subfigure}[b]{0.5\linewidth}
      \begin{tikzpicture}
         \begin{axis}[%
               legend pos=north west,
               xlabel=$k$,
               enlargelimits=false]
            \pgfplotstableread{data/eajar.biclique-by-k-n100.dat}{\data}
            \addplot+[mark size=1.5pt] table[x=k,y expr=\thisrow{mean}/\norm] {\data};

            \pgfplotstableread{data/eajar.biclique-by-k-n90.dat}{\data}
            \addplot+[mark size=1.5pt] table[x=k,y expr=\thisrow{mean}/\norm] {\data};

            \pgfplotstableread{data/eajar.biclique-by-k-n80.dat}{\data}
            \addplot+[mark size=1.5pt] table[x=k,y expr=\thisrow{mean}/\norm] {\data};

            \pgfplotstableread{data/eajar.biclique-by-k-n70.dat}{\data}
            \addplot+[mark size=1.5pt] table[x=k,y expr=\thisrow{mean}/\norm] {\data};

            \pgfplotstableread{data/eajar.biclique-by-k-n60.dat}{\data}
            \addplot+[mark size=1.5pt] table[x=k,y expr=\thisrow{mean}/\norm] {\data};

            \pgfplotstableread{data/eajar.biclique-by-k-n50.dat}{\data}
            \addplot+[mark size=1.5pt] table[x=k,y expr=\thisrow{mean}/\norm] {\data};

            \addplot[thick,dashed,black,domain=3:8] {2^x/100};

            \legend{$n=100$,$n=90$,$n=80$,$n=70$,$n=60$,$n=50$,$2^k/100$}
         \end{axis}
      \end{tikzpicture}
      \caption{biclique}
   \end{subfigure}

   \caption[]{\label{fig:nonrandom-by-k} Mean run time of \ea{k} divided by $n^2 \ln n$ on nonrandom vertex cover instances as a function of $k$.}
\end{figure}

\subsection{Comparison to Standard (1+1) EA}
\label{sec:comparison}
In this section we present experiments that compare the empirical running time of the \ea{k} to the standard (1+1)~EA with no constraint repair mechanism.
For the standard (1+1)~EA, we use the typical fitness function for
this problem that penalizes infeasible
covers~\cite{Khuri94anevolutionary,HeYaoLi2005,OlivetoHe2009vc}:
\begin{equation}
  \label{eq:fitness-standard}
  g(x) = |x| + n\cdot\Big\lvert \Big\{ (u,v) \in E \colon x[u] = x[v] = 0 \Big\}\Big\rvert.
\end{equation}
In contrast to the \ea{k} fitness function $f_k$ introduced in Equation~\eqref{eq:fitness} (which is designed to be maximized), the objective of the standard (1+1)~EA is to \emph{minimize} $g$, and a minimum feasible vertex cover of a graph $G = (V,E)$ encoded as a binary string of length $n$ would correspond to a global minimum of $g$.

Our choice of random planted, clique/anticlique and biclique graphs in the experiments reported above was motivated by the need to maintain fine-grained control over the parameters in order to understand how algorithm behavior scales with both $n$ and $k$. However, this control comes at a cost: the resulting graphs tend to be particularly easy to optimize. For instance, in the case of the nonrandom graphs, a simple greedy degree heuristic (i.e., choose vertices of the largest degree until a cover is obtained) would easily find the solution. This is because solution vertices have degree $n-1$ (in clique/anticlique graphs) or $n-k$ (in bicliques), and the remaining vertices not in the optimal cover have degree only $k$.  The penalty term in the standard fitness function stated in Equation~\eqref{eq:fitness-standard} in some sense provides a fitness signal that correlates to this greedy approach.
This can be roughly conceptualized as follows. Denote an optimal $k$-cover as $S \subseteq V$. From a uniform random
bitstring $x$, each vertex $v \in S$ where $x[v]=0$ is incident on at
least $(n-k)/2$ uncovered edges on average, whereas each non-optimal
vertex $u \in V \setminus S$ where $x[u]=0$ is incident on $k/2$
uncovered edges on average. In this situation, a mutation that results
in adding $v$ to the cover would
correspond to a much larger fitness improvement than one which adds $u$, and
would also obscure any losses obtained by simultaneously removing from
the cover some
non-optimal vertex $u' \in V \setminus S$ where $x[u']=1$. In fact, we would expect this to hold for many ``typical'' strings encountered by the (1+1)~EA and the process would tend to have a significant bias toward adding and keeping those vertices that the greedy heuristic would add. 

In the case of the \ea{k}, however, we would expect to pay a significant overhead as it first must find a feasible subgraph, and then subsequently rely on the jump-and-repair operation to iteratively move to new feasible supergraphs. We conjecture that the standard (1+1)~EA would have a significant advantage on clique/anticlique graphs and bicliques, and this may carry over to random planted instances in which the edges are chosen uniformly from an underlying clique/anticlique.

We repeat the experiments with the (1+1)~EA using the same graphs from the previous experiments.
The comparison of runtime as a function of $n$ on graphs with cover size $k=8$ in Figure~\ref{fig:eajar-ea-random-by-n} (for the random planted instances) and Figure~\ref{fig:eajar-ea-nonrandom-by-n} (for the nonrandom instances).
These results confirm our suspicions that the standard (1+1)~EA has an
advantage on these relatively simple graphs. Moreover, the runtime of
the standard EA tends to be more tightly concentrated than that of the \ea{k}.

%% RANDOM PLANTED
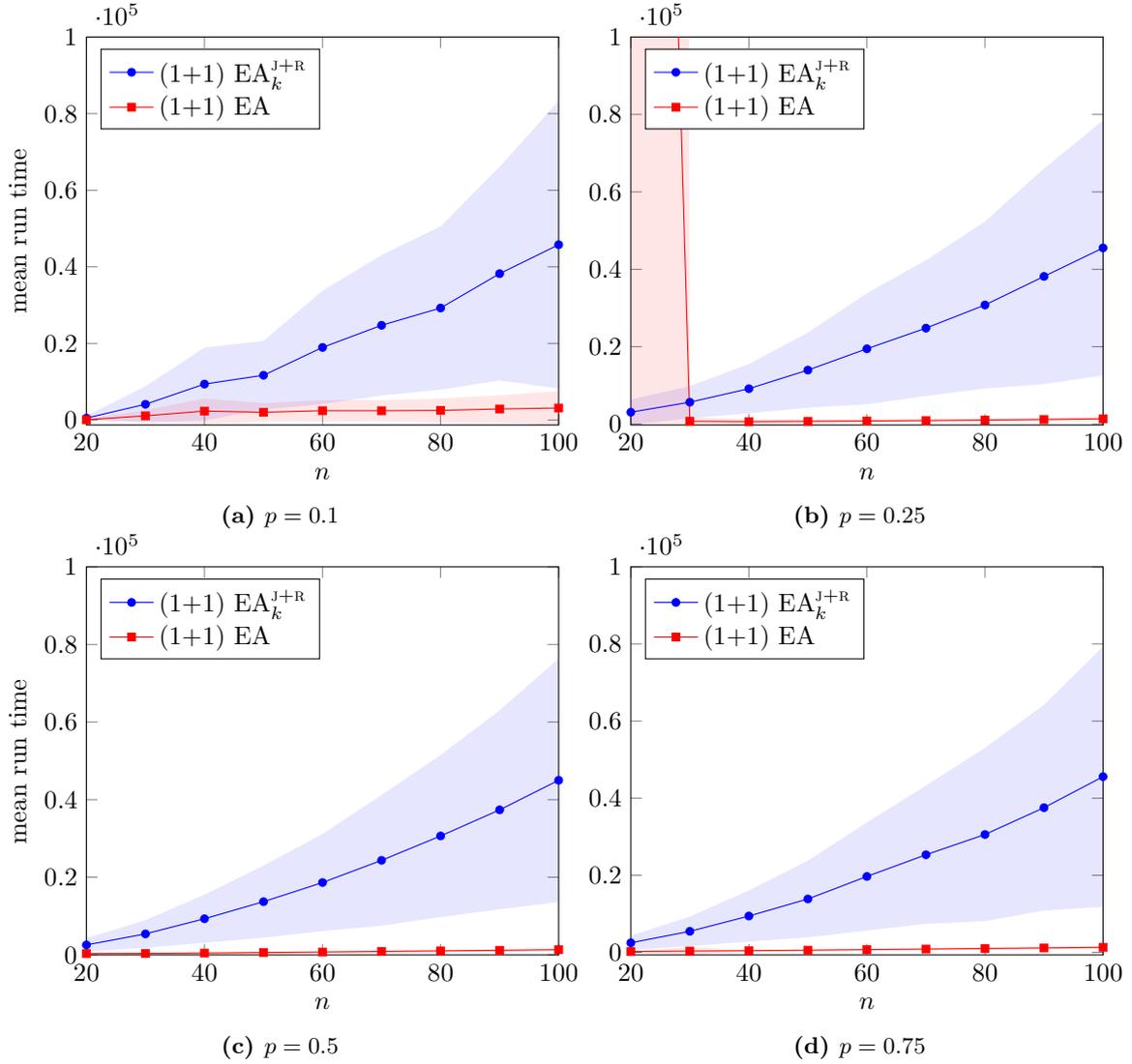
\begin{figure}
   \begin{subfigure}[b]{0.5\textwidth}
      \begin{tikzpicture}
         \begin{axis}[%
               legend pos=north west,
               xlabel=$n$, 
               ylabel={mean run time},
               ymax=100000,
               enlargelimits=false,
               legend cell align={left}]
            \pgfplotstableread{data/eajar.random_planted.edge_prob0.1-by-n-k8.dat}{\data}

            \addplot+[mark size=1.5pt] table[x=n,y=mean] {\data};

            \addplot[name path=upper,draw=none,forget plot] table[x=n,y   expr=\thisrow{mean}+\thisrow{sd}] {\data};
            \addplot[name path=lower,draw=none,forget plot] table[x=n,y expr=\thisrow{mean}-\thisrow{sd}] {\data};
            \pgfplotsset{cycle list shift=-1}
            \addplot+[fill opacity=0.1,forget plot] fill between[of=upper and lower];
            \pgfplotsset{cycle list shift=0}

            \pgfplotstableread{data/ea.random_planted.edge_prob0.1-by-n-k8.dat}{\data}

            \addplot+[mark size=1.5pt] table[x=n,y=mean] {\data};

            \addplot[name path=upper,draw=none,forget plot] table[x=n,y expr=\thisrow{mean}+\thisrow{sd}] {\data};
            \addplot[name path=lower,draw=none,forget plot]
            table[x=n,y expr=\thisrow{mean}-\thisrow{sd}] {\data};
            \pgfplotsset{cycle list shift=-1}
            \addplot+[fill opacity=0.1,forget plot] fill between[of=upper and lower];
            \pgfplotsset{cycle list shift=0}
            
            \legend{\ea{k},(1+1)~EA}
         \end{axis}
      \end{tikzpicture}%
      \caption{$p=0.1$}
   \end{subfigure}
   \hspace*{2mm}%
   \begin{subfigure}[b]{0.5\textwidth}
      %% p=0.25
      \begin{tikzpicture}
         \begin{axis}[%
               legend pos=north west,
               xlabel=$n$,
               ymax=100000,               
               enlargelimits=false,
               legend cell align={left}]

            \pgfplotstableread{data/eajar.random_planted.edge_prob0.25-by-n-k8.dat}{\data}
            \addplot+[mark size=1.5pt] table[x=n,y=mean] {\data};
            \addplot[name path=upper,draw=none,forget plot] table[x=n,y   expr=\thisrow{mean}+\thisrow{sd}] {\data};
            \addplot[name path=lower,draw=none,forget plot] table[x=n,y expr=\thisrow{mean}-\thisrow{sd}] {\data};
            \pgfplotsset{cycle list shift=-1}
            \addplot+[fill opacity=0.1,forget plot] fill between[of=upper and lower];
            \pgfplotsset{cycle list shift=0}

            \pgfplotstableread{data/ea.random_planted.edge_prob0.25-by-n-k8.dat}{\data}
            \addplot+[mark size=1.5pt] table[x=n,y=mean] {\data};
            
            \addplot[name path=upper,draw=none,forget plot]
            table[x=n,y   expr=\thisrow{mean}+\thisrow{sd}] {\data};

            %% Apply y filter to prevent negative plotting (results in
            %% "dimension too large" errors)
            \addplot[y filter/.code={\pgfmathparse{max(\pgfmathresult,0)}\pgfmathresult},
            name path=lower,draw=none,forget plot]
            table[x=n,y   expr=\thisrow{mean}-\thisrow{sd}] {\data};

            \pgfplotsset{cycle list shift=-1}
            \addplot+[fill opacity=0.1,forget plot] fill between[of=upper and lower];
             \pgfplotsset{cycle list shift=0}

            \legend{\ea{k},(1+1)~EA}
         \end{axis}
      \end{tikzpicture}%
      \caption{\label{fig:eajar-ea-random-by-n-p0.25}$p=0.25$}
   \end{subfigure}\\
   %% p=0.5
   \begin{subfigure}[b]{0.5\textwidth}
      \begin{tikzpicture}
         \begin{axis}[%
               legend pos=north west,
               xlabel=$n$,
               ylabel={mean run time},
               ymax=100000,
               enlargelimits=false,
               legend cell align={left}]

            \pgfplotstableread{data/eajar.random_planted-by-n-k8.dat}{\data}

            \addplot+[mark size=1.5pt] table[x=n,y=mean] {\data};
            \addplot[name path=upper,draw=none,forget plot] table[x=n,y   expr=\thisrow{mean}+\thisrow{sd}] {\data};
            \addplot[name path=lower,draw=none,forget plot] table[x=n,y expr=\thisrow{mean}-\thisrow{sd}] {\data};
            \pgfplotsset{cycle list shift=-1}
            \addplot+[fill opacity=0.1,forget plot] fill between[of=upper and lower];
            \pgfplotsset{cycle list shift=0}

            \pgfplotstableread{data/ea.random_planted.edge_prob0.5-by-n-k8.dat}{\data}
            \addplot+[mark size=1.5pt] table[x=n,y=mean] {\data};
            \addplot[name path=upper,draw=none,forget plot] table[x=n,y   expr=\thisrow{mean}+\thisrow{sd}] {\data};
            \addplot[name path=lower,draw=none,forget plot] table[x=n,y expr=\thisrow{mean}-\thisrow{sd}] {\data};
            \pgfplotsset{cycle list shift=-1}
            \addplot+[fill opacity=0.1,forget plot] fill between[of=upper and lower];
            \pgfplotsset{cycle list shift=0}
            
            \legend{\ea{k},(1+1)~EA}
         \end{axis}
      \end{tikzpicture}%
      \caption{$p=0.5$}
   \end{subfigure}
   \hspace*{2mm}%
   %%p=0.75
   \begin{subfigure}[b]{0.5\textwidth}
      \begin{tikzpicture}
         \begin{axis}[%
               legend pos=north west,
               xlabel=$n$,
               ymax=100000,               
               enlargelimits=false,
               legend cell align={left}]
            \pgfplotstableread{data/eajar.random_planted.edge_prob0.75-by-n-k8.dat}{\data}

            \addplot+[mark size=1.5pt] table[x=n,y=mean] {\data};
            \addplot[name path=upper,draw=none,forget plot] table[x=n,y   expr=\thisrow{mean}+\thisrow{sd}] {\data};
            \addplot[name path=lower,draw=none,forget plot] table[x=n,y expr=\thisrow{mean}-\thisrow{sd}] {\data};
            \pgfplotsset{cycle list shift=-1}
            \addplot+[fill opacity=0.1,forget plot] fill between[of=upper and lower];
            \pgfplotsset{cycle list shift=0}

            \pgfplotstableread{data/ea.random_planted.edge_prob0.75-by-n-k8.dat}{\data}
            \addplot+[mark size=1.5pt] table[x=n,y=mean] {\data};
            \addplot[name path=upper,draw=none,forget plot] table[x=n,y   expr=\thisrow{mean}+\thisrow{sd}] {\data};
            \addplot[name path=lower,draw=none,forget plot] table[x=n,y expr=\thisrow{mean}-\thisrow{sd}] {\data};
            \pgfplotsset{cycle list shift=-1}
            \addplot+[fill opacity=0.1,forget plot] fill between[of=upper and lower];
            \pgfplotsset{cycle list shift=0}

            \legend{\ea{k},(1+1)~EA}
         \end{axis}
      \end{tikzpicture}%
      \caption{$p=0.75$}
   \end{subfigure}
   \caption{\label{fig:eajar-ea-random-by-n} Comparison of mean running time of
     \ea{k} and (1+1)~EA on random planted vertex cover instances for
     various edge probabilities $p$ and $k=8$. Shaded bands represent standard deviation.}
\end{figure}

We also uncovered somewhat unexpected behavior on one random planted
instance with $20$ vertices and density $p=1/4$. This can be observed
in Figure~\ref{fig:eajar-ea-random-by-n-p0.25} where 15 out of the 100
runs of the (1+1)~EA on this graph required over $10^6$ fitness
function calls to find an optimal cover. The reason for this behavior
is that the EA can become quickly trapped in a locally optimal cover,
which is presumably difficult to escape.  This deceptive graph is
shown in Figure~\ref{fig:hard-cover-instance} where we indicate both
an optimal cover and a cover in which the (1+1)~EA became trapped
during a run. The reason for the emergence of this structure at these
particular parameters is not known, and could be a suitable course
for future work.

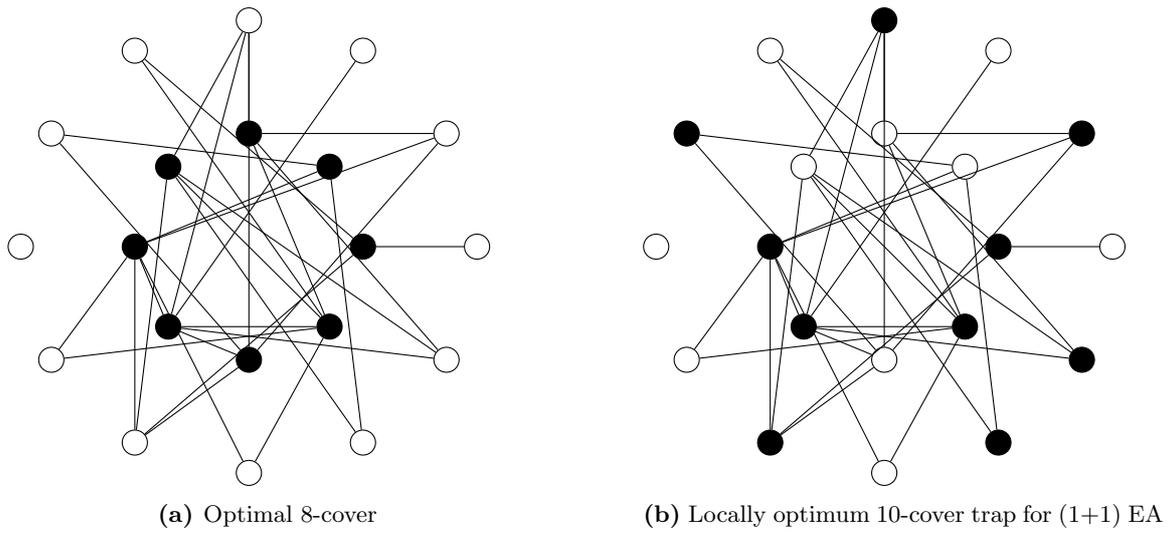
\begin{figure}
  \centering
  \begin{subfigure}{0.45\linewidth}
\begin{tikzpicture}[scale=1.5]
  
  \foreach \i in {1,...,8}
  {
  \coordinate (inner\i) at ({360/8 * (\i - 1)}:1);  
  } 
  \foreach \i in {1,...,12}
  {
  \coordinate (outer\i) at ({360/12 * (\i - 1)}:2);  
  }

  \node[smallvertex,fill=black] (0) at (inner1) {};
  \node[smallvertex,fill=black] (1) at (inner2) {};
  \node[smallvertex,fill=black] (4) at (inner3) {};
  \node[smallvertex,fill=black] (7) at (inner4) {};
  \node[smallvertex,fill=black] (10) at (inner5) {};
  \node[smallvertex,fill=black] (11) at (inner6) {};
  \node[smallvertex,fill=black] (14) at (inner7) {};
  \node[smallvertex,fill=black] (17) at (inner8) {};
  \node[smallvertex] (2) at (outer1) {};
  \node[smallvertex] (3) at (outer2) {};
  \node[smallvertex] (5) at (outer3) {};
  \node[smallvertex] (6) at (outer4) {};
  \node[smallvertex] (8) at (outer5) {};
  \node[smallvertex] (9) at (outer6) {};
  \node[smallvertex] (12) at (outer7) {};
  \node[smallvertex] (13) at (outer8) {};
  \node[smallvertex] (15) at (outer9) {};
  \node[smallvertex] (16) at (outer10) {};
  \node[smallvertex] (18) at (outer11) {};
  \node[smallvertex] (19) at (outer12) {};

  % 0 -- 2;
  \draw (0) edge (2);
  % 0 -- 8;
  \draw (0) edge (8);
  % 0 -- 15;
  \draw (0) edge (15);
  % 1 -- 9;
  \draw (1) edge (9);
  % 1 -- 10;
  \draw (1) edge (10);
  % 1 -- 18;
  \draw (1) edge (18);
  % 3 -- 4;
  \draw (3) edge (4);
  % 3 -- 10;
  \draw (3) edge (10);
  % 3 -- 14;
  \draw (3) edge (14);
  % 4 -- 6;
  \draw (4) edge (6);
  % 4 -- 17;
  \draw (4) edge (17);
  % 4 -- 19;
  \draw (4) edge (19);
  % 5 -- 11;
  \draw (5) edge (11);
  % 6 -- 7;
  \draw (6) edge (7);
  % 6 -- 11;
  \draw (6) edge (11);
  % 6 -- 14;
  \draw (6) edge (14);
  % 7 -- 15;
  \draw (7) edge (15);
  % 7 -- 17;
  \draw (7) edge (17);
  % 7 -- 18;
  \draw (7) edge (18);
  % 7 -- 19;
  \draw (7) edge (19);
  % 8 -- 17;
  \draw (8) edge (17);
  % 9 -- 14;
  \draw (9) edge (14);
  % 10 -- 11;
  \draw (10) edge (11);
  % 10 -- 13;
  \draw (10) edge (13);
  % 10 -- 15;
  \draw (10) edge (15);
  % 10 -- 16;
  \draw (10) edge (16);
  % 11 -- 14;
  \draw (11) edge (14);
  % 11 -- 17;
  \draw (11) edge (17);
  % 11 -- 19;
  \draw (11) edge (19);
  % 13 -- 17;
  \draw (13) edge (17);
  % 14 -- 15;
  \draw (14) edge (15);
  % 16 -- 17;
  \draw (16) edge (17);

\end{tikzpicture}

    \caption{Optimal 8-cover}
  \end{subfigure}\hfill
  \begin{subfigure}{0.45\linewidth}
\begin{tikzpicture}[scale=1.5]
 
  \foreach \i in {1,...,8}
  {
  \coordinate (inner\i) at ({360/8 * (\i-1)}:1);  
  } 
  \foreach \i in {1,...,12}
  {
  \coordinate (outer\i) at ({360/12 * (\i-1)}:2) {};  
  }

  \node[smallvertex,fill=black] (0) at (inner1) {};
  \node[smallvertex] (1) at (inner2) {};
  \node[smallvertex] (4) at (inner3) {};
  \node[smallvertex] (7) at (inner4) {};
  \node[smallvertex,fill=black] (10) at (inner5) {};
  \node[smallvertex,fill=black] (11) at (inner6) {};
  \node[smallvertex] (14) at (inner7) {};
  \node[smallvertex,fill=black] (17) at (inner8) {};
  \node[smallvertex] (2) at (outer1) {};
  \node[smallvertex,fill=black] (3) at (outer2) {};
  \node[smallvertex] (5) at (outer3) {};
  \node[smallvertex,fill=black] (6) at (outer4) {};
  \node[smallvertex] (8) at (outer5) {};
  \node[smallvertex,fill=black] (9) at (outer6) {};
  \node[smallvertex] (12) at (outer7) {};
  \node[smallvertex] (13) at (outer8) {};
  \node[smallvertex,fill=black] (15) at (outer9) {};
  \node[smallvertex] (16) at (outer10) {};
  \node[smallvertex,fill=black] (18) at (outer11) {};
  \node[smallvertex,fill=black] (19) at (outer12) {};

  % 0 -- 2;
  \draw (0) edge (2);
  % 0 -- 8;
  \draw (0) edge (8);
  % 0 -- 15;
  \draw (0) edge (15);
  % 1 -- 9;
  \draw (1) edge (9);
  % 1 -- 10;
  \draw (1) edge (10);
  % 1 -- 18;
  \draw (1) edge (18);
  % 3 -- 4;
  \draw (3) edge (4);
  % 3 -- 10;
  \draw (3) edge (10);
  % 3 -- 14;
  \draw (3) edge (14);
  % 4 -- 6;
  \draw (4) edge (6);
  % 4 -- 17;
  \draw (4) edge (17);
  % 4 -- 19;
  \draw (4) edge (19);
  % 5 -- 11;
  \draw (5) edge (11);
  % 6 -- 7;
  \draw (6) edge (7);
  % 6 -- 11;
  \draw (6) edge (11);
  % 6 -- 14;
  \draw (6) edge (14);
  % 7 -- 15;
  \draw (7) edge (15);
  % 7 -- 17;
  \draw (7) edge (17);
  % 7 -- 18;
  \draw (7) edge (18);
  % 7 -- 19;
  \draw (7) edge (19);
  % 8 -- 17;
  \draw (8) edge (17);
  % 9 -- 14;
  \draw (9) edge (14);
  % 10 -- 11;
  \draw (10) edge (11);
  % 10 -- 13;
  \draw (10) edge (13);
  % 10 -- 15;
  \draw (10) edge (15);
  % 10 -- 16;
  \draw (10) edge (16);
  % 11 -- 14;
  \draw (11) edge (14);
  % 11 -- 17;
  \draw (11) edge (17);
  % 11 -- 19;
  \draw (11) edge (19);
  % 13 -- 17;
  \draw (13) edge (17);
  % 14 -- 15;
  \draw (14) edge (15);
  % 16 -- 17;
  \draw (16) edge (17);

\end{tikzpicture}

    \caption{Locally optimum 10-cover trap for (1+1)~EA}
  \end{subfigure}
\caption{\label{fig:hard-cover-instance} Random planted instance with
  $n=20$, $k=8$ and $p=1/4$.}
\end{figure}

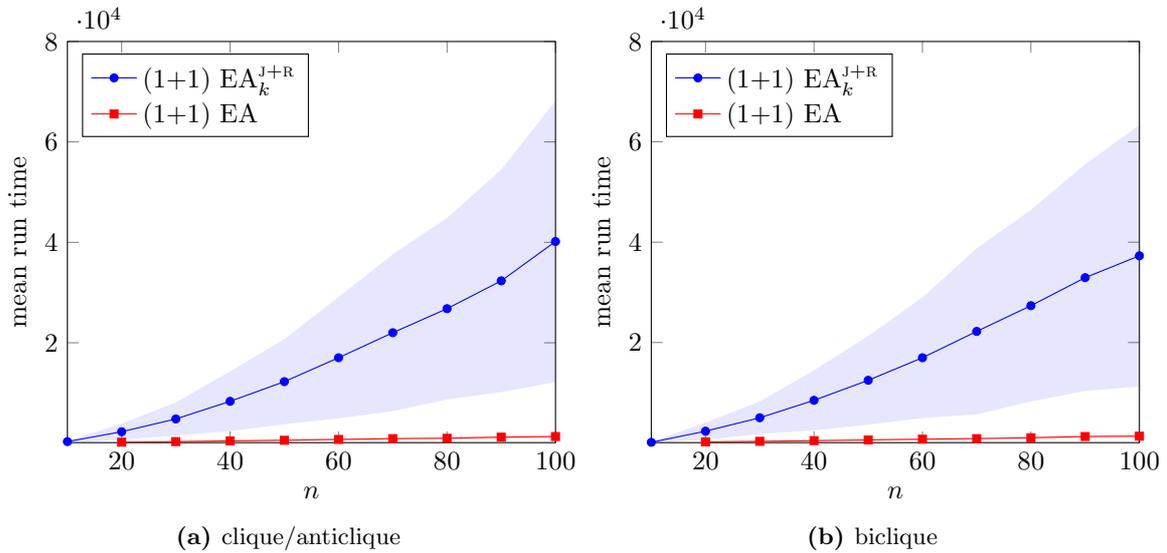
\begin{figure}
   \begin{subfigure}[b]{0.5\textwidth}
      \begin{tikzpicture}
         \begin{axis}[%
               legend pos=north west,
               xlabel=$n$,
               ylabel={mean run time},
               ymax=80000,
               enlargelimits=false,
               legend cell align={left}]

            \pgfplotstableread{data/eajar.clique-anticlique-by-n-k8.dat}{\data}

            \addplot+[mark size=1.5pt] table[x=n,y=mean] {\data};
            \addplot[name path=upper,draw=none,forget plot] table[x=n,y   expr=\thisrow{mean}+\thisrow{sd}] {\data};
            \addplot[name path=lower,draw=none,forget plot] table[x=n,y expr=\thisrow{mean}-\thisrow{sd}] {\data};
            \pgfplotsset{cycle list shift=-1}
            \addplot+[fill opacity=0.1,forget plot] fill between[of=upper and lower];
            \pgfplotsset{cycle list shift=0}

            \pgfplotstableread{data/ea.clique-anticlique-by-n-k8.dat}{\data}
            \addplot+[mark size=1.5pt] table[x=n,y=mean] {\data};
            \addplot[name path=upper,draw=none,forget plot] table[x=n,y   expr=\thisrow{mean}+\thisrow{sd}] {\data};
            \addplot[name path=lower,draw=none,forget plot] table[x=n,y expr=\thisrow{mean}-\thisrow{sd}] {\data};
            \pgfplotsset{cycle list shift=-1}
            \addplot+[fill opacity=0.1,forget plot] fill between[of=upper and lower];
            \pgfplotsset{cycle list shift=0}
            
            \legend{\ea{k},(1+1)~EA}
         \end{axis}
      \end{tikzpicture}%

      \caption[]{clique/anticlique}
   \end{subfigure}
   \hfill
   %%% biclique
   \begin{subfigure}[b]{0.5\textwidth}
      \begin{tikzpicture}
         \begin{axis}[%
               legend pos=north west,
               xlabel=$n$,
               ylabel={mean run time},
               ymax=80000,
               enlargelimits=false,
               legend cell align={left}]

            \pgfplotstableread{data/eajar.biclique-by-n-k8.dat}{\data}

            \addplot+[mark size=1.5pt] table[x=n,y=mean] {\data};
            \addplot[name path=upper,draw=none,forget plot] table[x=n,y   expr=\thisrow{mean}+\thisrow{sd}] {\data};
            \addplot[name path=lower,draw=none,forget plot] table[x=n,y expr=\thisrow{mean}-\thisrow{sd}] {\data};
            \pgfplotsset{cycle list shift=-1}
            \addplot+[fill opacity=0.1,forget plot] fill between[of=upper and lower];
            \pgfplotsset{cycle list shift=0}

            \pgfplotstableread{data/ea.biclique-by-n-k8.dat}{\data}
            \addplot+[mark size=1.5pt] table[x=n,y=mean] {\data};
            \addplot[name path=upper,draw=none,forget plot] table[x=n,y   expr=\thisrow{mean}+\thisrow{sd}] {\data};
            \addplot[name path=lower,draw=none,forget plot] table[x=n,y expr=\thisrow{mean}-\thisrow{sd}] {\data};
            \pgfplotsset{cycle list shift=-1}
            \addplot+[fill opacity=0.1,forget plot] fill between[of=upper and lower];
            \pgfplotsset{cycle list shift=0}
            
            \legend{\ea{k},(1+1)~EA}
         \end{axis}
      \end{tikzpicture}%
      \caption[]{biclique}
   \end{subfigure}
   \caption{\label{fig:eajar-ea-nonrandom-by-n} Comparison of mean running time of
     \ea{k} and (1+1)~EA nonrandom instances with $k=8$. Shaded bands represent standard deviation.}
\end{figure}

For the most part, the standard (1+1)~EA performs exceptionally well
as long as there are no pathologies in the search space that cause it
to become locally trapped. To broaden our perspective, we also
consider graphs with more intricate structure.
In particular, we compare the performance of the two algorithms on so-called Papadimitriou-Steiglitz graphs~\cite{PapadimitriouS82}, which have been the subject of both empirical and theoretical investigations of evolutionary algorithms on vertex cover~\cite{Khuri94anevolutionary,OlivetoHe2009vc}. This class of graphs was originally defined to demonstrate the failure of simple greedy degree-heuristics to approximate minimum vertex cover and consists of a complete bipartite graph $K_{\ell,\ell+2}$ to which each vertex in the size $\ell+2$ partition is attached a pendant vertex (see Figure~\ref{fig:psgraph}).

\begin{figure}
  \centering
  \begin{tikzpicture}[node distance=5mm]
    \node[smallvertex] (v1) {};
    \node[smallvertex,right=of v1] (v2) {};
    \node[smallvertex,right=of v2] (v3) {};
    \node[smallvertex,right=of v3] (v4) {};
    \node[smallvertex,right=of v4] (v5) {};

    \node[smallvertex,fill=black,below=1cm of v1] (u1) {};
    \node[smallvertex,fill=black,right=of u1] (u2) {};
    \node[smallvertex,fill=black,right=of u2] (u3) {};
    \node[smallvertex,fill=black,right=of u3] (u4) {};
    \node[smallvertex,fill=black,right=of u4] (u5) {};
    \node[smallvertex,fill=black,right=of u5] (u6) {};
    \node[smallvertex,fill=black,right=of u6] (u7) {};

    \node[smallvertex,below=of u1] (w1) {};
    \node[smallvertex,right=of w1] (w2) {};
    \node[smallvertex,right=of w2] (w3) {};
    \node[smallvertex,right=of w3] (w4) {};
    \node[smallvertex,right=of w4] (w5) {};
    \node[smallvertex,right=of w5] (w6) {};
    \node[smallvertex,right=of w6] (w7) {};

    \draw (v1) edge (u1);
    \draw (v1) edge (u2);
    \draw (v1) edge (u3);
    \draw (v1) edge (u4);
    \draw (v1) edge (u5);
    \draw (v1) edge (u6);
    \draw (v1) edge (u7);

    \draw (v2) edge (u1);
    \draw (v2) edge (u2);
    \draw (v2) edge (u3);
    \draw (v2) edge (u4);
    \draw (v2) edge (u5);
    \draw (v2) edge (u6);
    \draw (v2) edge (u7);

    \draw (v3) edge (u1);
    \draw (v3) edge (u2);
    \draw (v3) edge (u3);
    \draw (v3) edge (u4);
    \draw (v3) edge (u5);
    \draw (v3) edge (u6);
    \draw (v3) edge (u7);

    \draw (v4) edge (u1);
    \draw (v4) edge (u2);
    \draw (v4) edge (u3);
    \draw (v4) edge (u4);
    \draw (v4) edge (u5);
    \draw (v4) edge (u6);
    \draw (v4) edge (u7);

    \draw (v5) edge (u1);
    \draw (v5) edge (u2);
    \draw (v5) edge (u3);
    \draw (v5) edge (u4);
    \draw (v5) edge (u5);
    \draw (v5) edge (u6);
    \draw (v5) edge (u7);

    \draw (u1) edge (w1);
    \draw (u2) edge (w2);
    \draw (u3) edge (w3);
    \draw (u4) edge (w4);
    \draw (u5) edge (w5);
    \draw (u6) edge (w6);
    \draw (u7) edge (w7);

  \end{tikzpicture}
  \caption{\label{fig:psgraph} Papadimitriou-Steiglitz graph~\cite{PapadimitriouS82} of order $\ell = 5$ on $3\ell + 4$ vertices. Optimal size $\ell+2$ cover is colored in black.}
\end{figure}
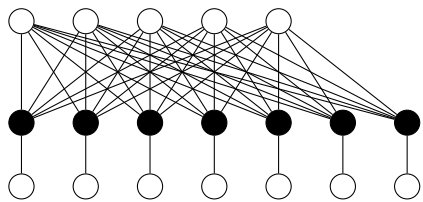

Figure~\ref{fig:ps-ecdf} compares the running times of the (1+1)~EA
and \ea{k} by plotting the empirical cumulative distribution
functions of the running time on a representative set of
Papadimitriou-Steiglitz graphs. For each graph, the EAs were run 100
times, and the number of fitness evaluations required to find the
optimal vertex cover was recorded for each run, resulting in the
distribution plots.

Clearly, for smaller graphs, the performance of the (1+1)~EA dominates that of the \ea{k}. This is not surprising, as the latter must always rebuild $G$ from induced subgraphs.  However, the \ea{k} remains relatively impervious to the effect of the suboptimal local optimum that impedes the success probability of the (1+1)~EA as the graph size increases. For example, in \ref{fig:ps-largest}, the \ea{k} is successful in all the runs by 12730889 fitness evaluations whereas the (1+1)~EA has only solved $< 50\%$ of its runs by 20711816 fitness evaluations. As $n$ increases, the (1+1)~EA exhibits polynomial runtime only with probability approaching $1/2$. This can be overcome with restarts. In fact, Oliveto, He and Yao~\cite[Corollary 2]{OlivetoHe2009vc} proved that the (1+1)~EA with multiple runs would require only $O(n \log n)$ steps to solve this class of graphs.

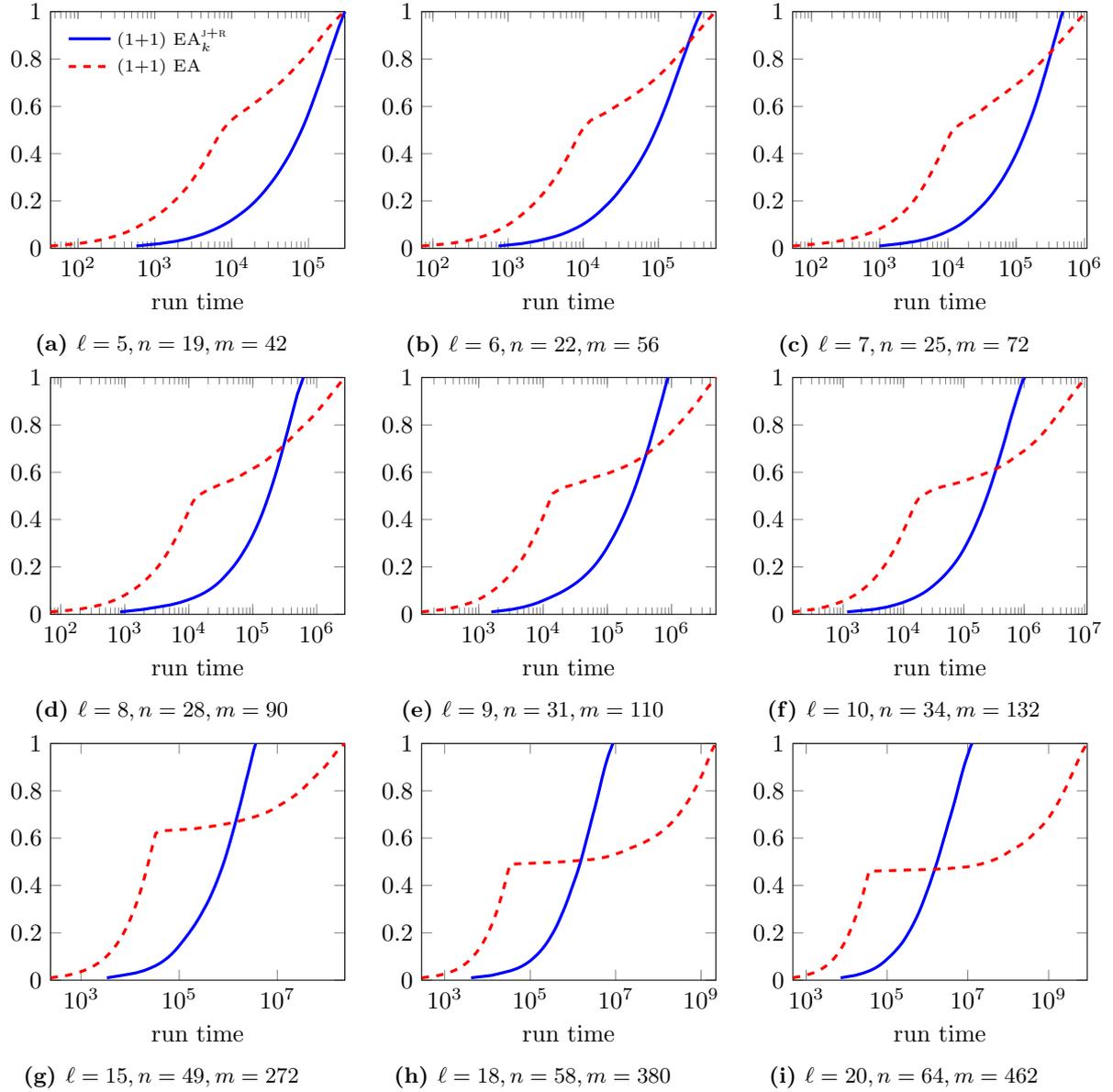
\begin{figure}
  \begin{subfigure}[b]{0.3\linewidth}
    \begin{tikzpicture}
    \begin{semilogxaxis}[%
      clip=false,
      xlabel near ticks,
      xlabel={run time},      
      legend pos=north west,
      legend cell align={left},
      legend style={draw=none,font=\scriptsize},
    enlargelimits=false,
    width=5.8cm,    
    ymin=0,ymax=1,    
    every axis plot/.style={very thick,no marks}]   
      \addplot table[x=runtime,y=freq]
      {data/eajar.papadimitriou-steiglitz-5-ecdf.dat};
      \addplot[dashed,red] table[x=runtime,y=freq]
      {data/ea.papadimitriou-steiglitz-5-ecdf.dat};
      \legend{\ea{k},(1+1)~EA}
\end{semilogxaxis}
\end{tikzpicture}
\caption{$\ell=5, n=19, m=42$}
\end{subfigure}\hfill
  \begin{subfigure}[b]{0.3\linewidth}
    \begin{tikzpicture}
    \begin{semilogxaxis}[%
      clip=false,
      xlabel near ticks,
      xlabel={run time},
    enlargelimits=false,
    width=5.8cm,    
    ymin=0,ymax=1,    
    every axis plot/.style={very thick,no marks}]   
      \addplot table[x=runtime,y=freq]
      {data/eajar.papadimitriou-steiglitz-6-ecdf.dat};
      \addplot[dashed,red] table[x=runtime,y=freq]
      {data/ea.papadimitriou-steiglitz-6-ecdf.dat};      
\end{semilogxaxis}
\end{tikzpicture}
\caption{$\ell=6, n=22, m=56$}
\end{subfigure}\hfill
  \begin{subfigure}[b]{0.3\linewidth}
    \begin{tikzpicture}
    \begin{semilogxaxis}[%
      clip=false,
      xlabel near ticks,
      xlabel={run time},      
    enlargelimits=false,
    width=5.8cm,    
    ymin=0,ymax=1,    
    every axis plot/.style={very thick,no marks}]   
      \addplot table[x=runtime,y=freq]
      {data/eajar.papadimitriou-steiglitz-7-ecdf.dat};
      \addplot[dashed,red] table[x=runtime,y=freq]
      {data/ea.papadimitriou-steiglitz-7-ecdf.dat};      
\end{semilogxaxis}
\end{tikzpicture}
\caption{$\ell=7, n=25, m=72$}
\end{subfigure}

  \begin{subfigure}[b]{0.3\linewidth}
    \begin{tikzpicture}
    \begin{semilogxaxis}[%
      clip=false,
      xlabel near ticks,
      xlabel={run time},
          enlargelimits=false,
    width=5.8cm,    
    ymin=0,ymax=1,    
    every axis plot/.style={very thick,no marks}]   
      \addplot table[x=runtime,y=freq]
      {data/eajar.papadimitriou-steiglitz-8-ecdf.dat};
      \addplot[dashed,red] table[x=runtime,y=freq]
      {data/ea.papadimitriou-steiglitz-8-ecdf.dat};      
\end{semilogxaxis}
\end{tikzpicture}
\caption{$\ell=8, n=28, m=90$}
\end{subfigure}\hfill
  \begin{subfigure}[b]{0.3\linewidth}
    \begin{tikzpicture}
    \begin{semilogxaxis}[%
      clip=false,
      xlabel near ticks,
      xlabel={run time},      
    enlargelimits=false,
    width=5.8cm,    
    ymin=0,ymax=1,    
    every axis plot/.style={very thick,no marks}]   
      \addplot table[x=runtime,y=freq]
      {data/eajar.papadimitriou-steiglitz-9-ecdf.dat};
      \addplot[dashed,red] table[x=runtime,y=freq]
      {data/ea.papadimitriou-steiglitz-9-ecdf.dat};     
\end{semilogxaxis}
\end{tikzpicture}
\caption{$\ell=9, n=31, m=110$}
\end{subfigure}\hfill
  \begin{subfigure}[b]{0.3\linewidth}
    \begin{tikzpicture}
    \begin{semilogxaxis}[%
      clip=false,
      xlabel near ticks,
      xlabel={run time},      
    enlargelimits=false,
    width=5.8cm,    
    ymin=0,ymax=1,    
    every axis plot/.style={very thick,no marks}]   
      \addplot table[x=runtime,y=freq]
      {data/eajar.papadimitriou-steiglitz-10-ecdf.dat};
      \addplot[dashed,red] table[x=runtime,y=freq]
      {data/ea.papadimitriou-steiglitz-10-ecdf.dat};      
\end{semilogxaxis}
\end{tikzpicture}
\caption{$\ell=10, n=34, m=132$}
\end{subfigure}

  \begin{subfigure}[b]{0.3\linewidth}
    \begin{tikzpicture}
    \begin{semilogxaxis}[%
      clip=false,
      xlabel near ticks,
      xlabel={run time},      
    enlargelimits=false,
    width=5.8cm,    
    ymin=0,ymax=1,    
    every axis plot/.style={very thick,no marks}]   
      \addplot table[x=runtime,y=freq]
      {data/eajar.papadimitriou-steiglitz-15-ecdf.dat};
      \addplot[dashed,red] table[x=runtime,y=freq]
      {data/ea.papadimitriou-steiglitz-15-ecdf.dat};      
\end{semilogxaxis}
\end{tikzpicture}
\caption{$\ell=15, n=49, m=272$}
\end{subfigure}\hfill
  \begin{subfigure}[b]{0.3\linewidth}
    \begin{tikzpicture}
    \begin{semilogxaxis}[%
      clip=false,
      xlabel near ticks,
      xlabel={run time},      
    enlargelimits=false,
    width=5.8cm,    
    ymin=0,ymax=1,    
    every axis plot/.style={very thick,no marks}]   
      \addplot table[x=runtime,y=freq]
      {data/eajar.papadimitriou-steiglitz-18-ecdf.dat};
      \addplot[dashed,red] table[x=runtime,y=freq]
      {data/ea.papadimitriou-steiglitz-18-ecdf.dat};      
\end{semilogxaxis}
\end{tikzpicture}
\caption{$\ell=18, n=58, m=380$}
\end{subfigure}\hfill
\begin{subfigure}[b]{0.3\linewidth}
    \begin{tikzpicture}
    \begin{semilogxaxis}[%
    clip=false,   
    width=5.8cm,
    xlabel near ticks,
    xlabel={run time},
    enlargelimits=false,
    ymin=0,ymax=1,    
    every axis plot/.style={very thick,no marks}]   
      \addplot table[x=runtime,y=freq]
      {data/eajar.papadimitriou-steiglitz-20-ecdf.dat};
      \addplot[dashed,red] table[x=runtime,y=freq] {data/ea.papadimitriou-steiglitz-20-ecdf.dat};
\end{semilogxaxis}
\end{tikzpicture}
\caption{\label{fig:ps-largest}$\ell=20,n=64,m=462$}
  \end{subfigure}
\caption{\label{fig:ps-ecdf}Empirical cumulative distribution functions for running time
  of the (1+1)~EA and the \ea{k} on Papadimitriou-Steiglitz
  instances. The $y$-axes denote the proportion of instances solved by
a particular run time.}
\end{figure}

A more compelling contrast between the (1+1)~EA and the \ea{k} is
possible in graphs with more detailed structure where the arrangement
of local optima is more complex.  In the same paper, Oliveto, He and
Yao presented a class of graphs that are difficult for the (1+1)~EA to
find an approximation within a factor of $2(1-\epsilon) - o(1)$. These
graphs~\cite[Section 7]{OlivetoHe2009vc}, which we will refer to as
Oliveto-He-Yao graphs, are constructed as a chain of complete
bipartite graphs with unbalanced partitions. In particular, for a
given $n \in \N$ and constant $0 < \epsilon < 1/2$, we take $\sqrt{n}$
copies of the complete bipartite graph
$K_{(1-\epsilon)(\sqrt{n}-1),1+\epsilon(\sqrt{n}-1)}$ and connect each
``block'' by making a vertex in the small partition of one block
adjacent to a vertex in the small partition of the next block (see
Figure~\ref{fig:ohygraph}. This $\epsilon$ parameter controls how
unbalanced the partitions are in each complete bipartite graph block.
This class of graphs is more difficult for the (1+1)~EA to optimize
because it must simultaneously synchronize all the blocks to the
correct local cover.

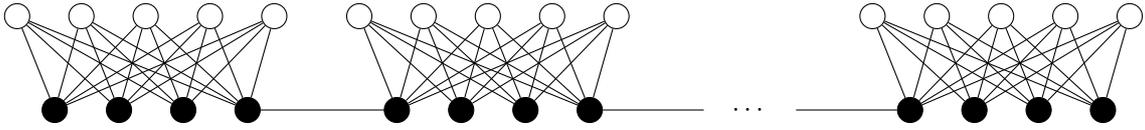
\begin{figure}
  \centering
  \begin{tikzpicture}[node distance=5mm]

    \newcommand{\drawbipartite}[1]{
      \node[smallvertex,fill=black] (v#1-1) {};
      \node[smallvertex,fill=black,right=of v#1-1] (v#1-2) {};
      \node[smallvertex,fill=black,right=of v#1-2] (v#1-3) {};
      \node[smallvertex,fill=black,right=of v#1-3] (v#1-4) {};

      \node[smallvertex,above left=1cm and 2.5mm of v#1-1] (u#1-1) {};
      \node[smallvertex,right=of u#1-1] (u#1-2) {};
      \node[smallvertex,right=of u#1-2] (u#1-3) {};
      \node[smallvertex,right=of u#1-3] (u#1-4) {};
      \node[smallvertex,right=of u#1-4] (u#1-5) {};

      \foreach \i in {1,2,3,4}{
        \foreach \j in {1,2,3,4,5}{
          \draw (v#1-\i) edge (u#1-\j);          
        }
      }      
    }
    \drawbipartite{1}    
    
    \begin{scope}[xshift=4.5cm]
      \drawbipartite{2}    
    \end{scope}

    \begin{scope}[xshift=11.25cm]
      \drawbipartite{3}    
    \end{scope}

    \draw (v1-4) edge (v2-1);
    \draw (v2-4) edge ($(v2-4) + (1.5,0)$);
    \draw (v3-1) edge ($(v3-1) - (1.5,0)$);
    \path (v2-4) -- (v3-1) node[midway] {$\cdots$};
  \end{tikzpicture}
  \caption[]{\label{fig:ohygraph} Oliveto-He-Yao graph. Optimal vertex
  cover is colored in black.}
\end{figure}

We repeated the comparison on these graphs, varying both $n$ and
$\epsilon$. In particular, for $\sqrt{n} \in \{5,6,7,8,9,10\}$ and
$\epsilon \in \{1/10,1/4,1/3\}$ we constructed an Oliveto-He-Yao graph
for $n$ and $\epsilon$. In the case that $\epsilon(\sqrt{n}-1)$ is not
an integer, each bipartite block is taken to be $K_{a,b}$ where $a$
and $b$ are the nearest integers to $(1-\epsilon)(\sqrt{n}-1)$ and
$1+\epsilon(\sqrt{n}-1)$, respectively. For each such graph, we ran
the (1+1)~EA and the \ea{k} 100 times and collected the number of
fitness evaluations required until the optimal vertex cover was
found. A budget of 96 hours was allocated to each run, and if the
optimum vertex cover was not found within this time budget, the
process was marked as a failure. The resulting empirical cumulative
distribution functions are plotted in
Figures~\ref{fig:ohy-ecdf-eps0.1}, \ref{fig:ohy-ecdf-eps0.25} and
\ref{fig:ohy-ecdf-eps0.3333}.

Smaller $\epsilon$ values correspond to graphs that are comparatively
easier to solve, since vertices belonging to the optimal cover are
incident to significantly more edges than those that are not.  For
example, when $n=25$ and $\epsilon=1/10$, the instance is a chain of 5
copies of the star graph $K_{1,4}$ connected by their central vertices
(see Figure~\ref{fig:ohy-easy}).
Clearly, on this graph, the (1+1)~EA can quickly make progress without
much difficulty. For example, once a feasible cover has been found,
there is always a reasonable chance to make a strictly improving
mutation if the central vertex of a suboptimal block is already in the
cover (i.e., the black vertices of Figure~\ref{fig:ohy-easy}),
otherwise it can easily add the central block vertex and remove any
vertex from a doubly-covered edge resulting in a cover of equal
fitness that can be improved with a single mutation as above. We leave
a more precise analysis of the (1+1)~EA on this graph as an exercise.
In any case, the performance of the (1+1)~EA clearly dominates the
\ea{k} for this particular graph as can be seen in Figure~\ref{fig:ohy-ecdf-n25-eps0.1}.

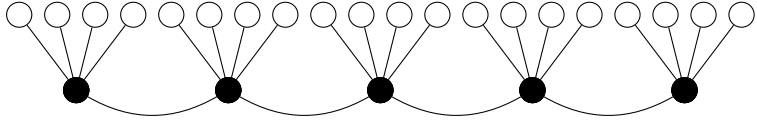
\begin{figure}
  \centering
  \begin{tikzpicture}[node distance=5mm,x=2cm]
    \def \noderadius {1.5cm}
    \def \starangle {30}
      \foreach \i in {1,2,3,4,5}
      {
        \foreach[count=\j] \x in {-0.375,-0.125,0.125,0.375}{
          \node[smallvertex,fill=black] (v\i0) at (\i,0) {};
          \node[smallvertex] (v\i\j) at ($(\i,1)+(\x,0)$) {};
          \draw (v\i0) edge (v\i\j);
        }
      }            
      \draw (v10) edge[bend right] (v20);
      \draw (v20) edge[bend right] (v30);
      \draw (v30) edge[bend right] (v40);
      \draw (v40) edge[bend right] (v50);      
    \end{tikzpicture}
\caption{\label{fig:ohy-easy}Easy Oliveto-He-Yao graph: $n=25$,
  $\epsilon=1/10$.}
\end{figure}

Another slight aberration is the Oliveto-He-Yao graph with $n=36$ and
$\epsilon=1/3$. In this instance, the values of
$(1-\epsilon)(\sqrt{n}-1)$ and $1+\epsilon\sqrt{n}$ rounded to the
nearest integer are both 3, thus this instance consists of a chain of
6 copies of $K_{3,3}$, and so therefore has multiple global
optima. This is reflected in the slightly improved behavior of the
(1+1)~EA seen in Figure~\ref{fig:ohy-ecdf-n36-eps0.333}. This is the only instance with this property.

On larger graphs, runs that failed to find an optimal solution during
the allotted budget are reflected in the plots where the rightmost
limit of the CDF falls short of one. We point out that this instance
class is not particularly easy from an FPT perspective, since the
optimal cover size is scaling linearly with the problem size: there
are $\sqrt{n}$ blocks of size $1+\epsilon\left(\sqrt{n}-1\right)$, meaning the
size of the optimal vertex cover is, up to rounding factors,
$\epsilon n + (1-\epsilon)\sqrt{n} = \Theta(n)$. Nevertheless, the
instance class provides a descriptive comparison between the
(1+1)~EA and the \ea{k}.

\begin{figure}
  \begin{subfigure}[b]{0.3\linewidth}
    \begin{tikzpicture}
    \begin{semilogxaxis}[%
      clip=false,
      xlabel={run time},
      xlabel near ticks,      
      legend pos=north west,
      legend cell align={left},
      legend style={draw=none,font=\scriptsize},
    % jump mark left,
    enlargelimits=false,
    width=5.8cm,    
    ymin=0,ymax=1,    
    every axis plot/.style={very thick,no marks}]   
      \addplot table[x=runtime,y=freq]
      {data/eajar.oliveto-he-yao-5-0.1-ecdf.dat};
      \addplot[dashed,red] table[x=runtime,y=freq]
      {data/ea.oliveto-he-yao-5-0.1-ecdf.dat};
      \legend{\ea{k},(1+1)~EA}
\end{semilogxaxis}
\end{tikzpicture}
\caption{\label{fig:ohy-ecdf-n25-eps0.1}$n=25$, $m=24$, $k^{*}=5$}
  \end{subfigure}\hfill
\begin{subfigure}[b]{0.3\linewidth}
    \begin{tikzpicture}
    \begin{semilogxaxis}[%
      clip=false,
      xlabel={run time},
      xlabel near ticks,
    % jump mark left,
    width=5.8cm,
    enlargelimits=false,
    ymin=0,ymax=1,    
    every axis plot/.style={very thick,no marks}]   
      \addplot table[x=runtime,y=freq]
      {data/eajar.oliveto-he-yao-6-0.1-ecdf.dat};
      \addplot[dashed,red] table[x=runtime,y=freq] {data/ea.oliveto-he-yao-6-0.1-ecdf.dat};
\end{semilogxaxis}
\end{tikzpicture}
\caption{$n=36, m=65, k^{*}=12$}
  \end{subfigure}\hfill
\begin{subfigure}[b]{0.3\linewidth}
    \begin{tikzpicture}
    \begin{semilogxaxis}[%
    clip=false,
    % jump mark left,
    width=5.8cm,
    xlabel={run time},
    xlabel near ticks,
    enlargelimits=false,
    ymin=0,ymax=1,    
    every axis plot/.style={very thick,no marks}]   
      \addplot table[x=runtime,y=freq]
      {data/eajar.oliveto-he-yao-7-0.1-ecdf.dat};
      \addplot[dashed,red] table[x=runtime,y=freq] {data/ea.oliveto-he-yao-7-0.1-ecdf.dat};
\end{semilogxaxis}
\end{tikzpicture}
\caption{$n=49, m=76, k^{*}=14$}
  \end{subfigure}
\begin{subfigure}[b]{0.3\linewidth}
    \begin{tikzpicture}
    \begin{semilogxaxis}[%
      clip=false,
      xlabel={run time},
      xlabel near ticks,
    % jump mark left,
    enlargelimits=false,
    width=5.8cm,    
    ymin=0,ymax=1,    
    every axis plot/.style={very thick,no marks}]   
      \addplot table[x=runtime,y=freq]
      {data/eajar.oliveto-he-yao-8-0.1-ecdf.dat};
      \addplot[dashed,red] table[x=runtime,y=freq]
      {data/ea.oliveto-he-yao-8-0.1-ecdf.dat};     
\end{semilogxaxis}
\end{tikzpicture}
\caption{$n=64, m=103, k^{*}=16$}
  \end{subfigure}\hfill
\begin{subfigure}[b]{0.3\linewidth}
    \begin{tikzpicture}
    \begin{semilogxaxis}[%
      clip=false,
      xlabel={run time},
      xlabel near ticks,
    % jump mark left,
    width=5.8cm,
    enlargelimits=false,
    ymin=0,ymax=1,    
    every axis plot/.style={very thick,no marks}]   
      \addplot table[x=runtime,y=freq]
      {data/eajar.oliveto-he-yao-9-0.1-ecdf.dat};
      \addplot[dashed,red] table[x=runtime,y=freq] {data/ea.oliveto-he-yao-9-0.1-ecdf.dat};
\end{semilogxaxis}
\end{tikzpicture}
\caption{$n=81, m=134, k^{*}=18$}
  \end{subfigure}\hfill
\begin{subfigure}[b]{0.3\linewidth}
    \begin{tikzpicture}
    \begin{semilogxaxis}[%
    clip=false,
    % jump mark left,
    width=5.8cm,
    xlabel={run time},
    xlabel near ticks,
    enlargelimits=false,
    ymin=0,ymax=1,    
    every axis plot/.style={very thick,no marks}]   
      \addplot table[x=runtime,y=freq]
      {data/eajar.oliveto-he-yao-10-0.1-ecdf.dat};
      \addplot[dashed,red] table[x=runtime,y=freq] {data/ea.oliveto-he-yao-10-0.1-ecdf.dat};
\end{semilogxaxis}
\end{tikzpicture}
\caption{$n=100, m=169, k^{*}=20$}
  \end{subfigure}

  \caption{\label{fig:ohy-ecdf-eps0.1}Empirical cumulative distribution functions for running time
  of the (1+1)~EA and the \ea{k} on Oliveto-He-Yao instances
  ($\epsilon=1/10$). The optimal cover size is denoted by $k^{*}$. The $y$-axes denote the proportion of instances solved by
a particular run time.}
\end{figure}
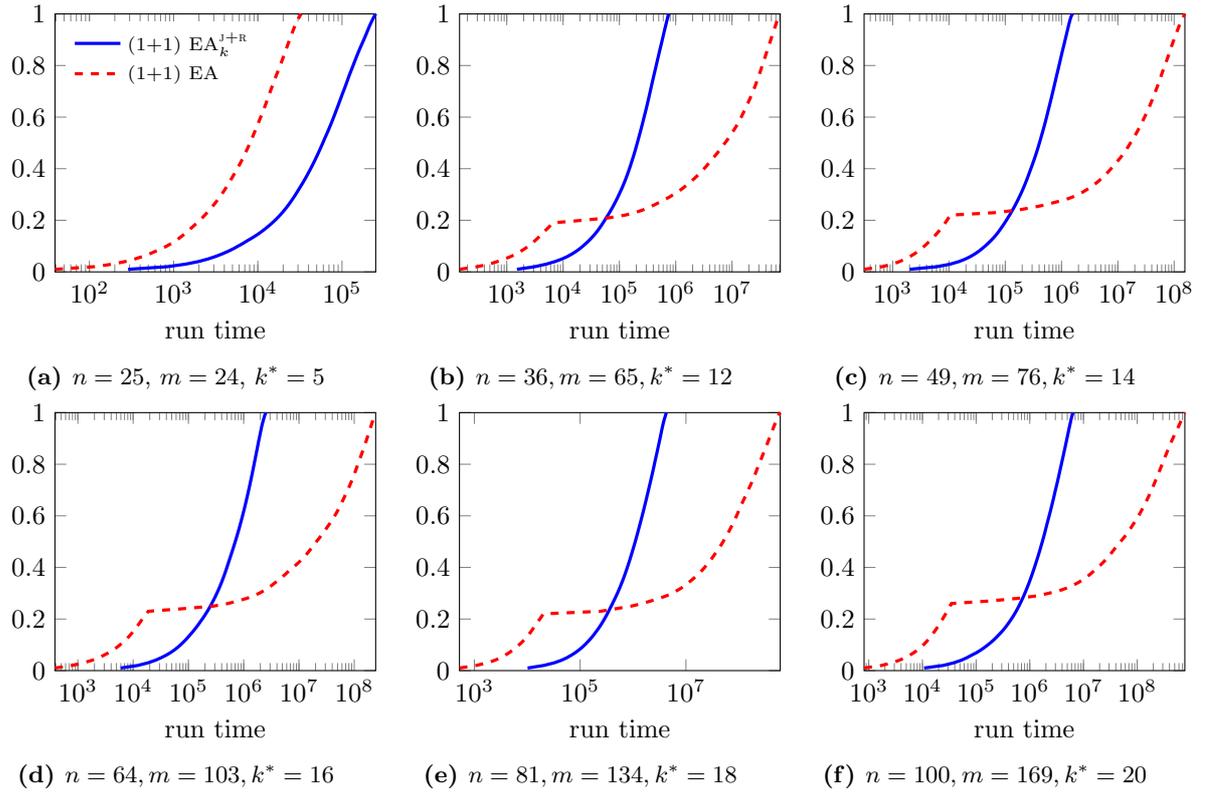

\begin{figure}
  \begin{subfigure}[b]{0.3\linewidth}
    \begin{tikzpicture}
    \begin{semilogxaxis}[%
      clip=false,
      xlabel={run time},
      xlabel near ticks,      
      legend pos=north west,
      legend cell align={left},
      legend style={draw=none,fill=none,font=\scriptsize},
    % jump mark left,
    enlargelimits=false,
    width=5.8cm,    
    ymin=0,ymax=1,    
    every axis plot/.style={very thick,no marks},
    ]%   
      \addplot table[x=runtime,y=freq]
      {data/eajar.oliveto-he-yao-5-0.25-ecdf.dat};
      \addplot[dashed,red] table[x=runtime,y=freq]
      {data/ea.oliveto-he-yao-5-0.25-ecdf.dat};
      \legend{\ea{k},(1+1)~EA}
\end{semilogxaxis}
\end{tikzpicture}
\caption{$n=25$, $m=34$, $k^{*}=10$}
  \end{subfigure}\hfill
\begin{subfigure}[b]{0.3\linewidth}
    \begin{tikzpicture}
    \begin{semilogxaxis}[%
      clip=false,
      xlabel={run time},
      xlabel near ticks,
    % jump mark left,
    width=5.8cm,
    enlargelimits=false,
    ymin=0,ymax=1,    
    every axis plot/.style={very thick,no marks},
    ]%   
      \addplot table[x=runtime,y=freq]
      {data/eajar.oliveto-he-yao-6-0.25-ecdf.dat};
      \addplot[dashed,red] table[x=runtime,y=freq] {data/ea.oliveto-he-yao-6-0.25-ecdf.dat};
\end{semilogxaxis}
\end{tikzpicture}
\caption{$n=36, m=53, k^{*}=12$}
  \end{subfigure}\hfill
\begin{subfigure}[b]{0.3\linewidth}
    \begin{tikzpicture}
    \begin{semilogxaxis}[%
    clip=false,
    % jump mark left,
    width=5.8cm,
    xlabel={run time},
    xlabel near ticks,
    enlargelimits=false,
    ymin=0,ymax=1,    
    every axis plot/.style={very thick,no marks}]   
      \addplot table[x=runtime,y=freq]
      {data/eajar.oliveto-he-yao-7-0.25-ecdf.dat};
      \addplot[dashed,red] table[x=runtime,y=freq] {data/ea.oliveto-he-yao-7-0.25-ecdf.dat};
\end{semilogxaxis}
\end{tikzpicture}
\caption{$n=49, m=111, k^{*}=21$}
  \end{subfigure}
\begin{subfigure}[b]{0.3\linewidth}
    \begin{tikzpicture}
    \begin{semilogxaxis}[%
      clip=false,
      xlabel={run time},
      xlabel near ticks,
    % jump mark left,
    enlargelimits=false,
    width=5.8cm,    
    ymin=0,ymax=1,    
    every axis plot/.style={very thick,no marks}]   
      \addplot table[x=runtime,y=freq]
      {data/eajar.oliveto-he-yao-8-0.25-ecdf.dat};
      \addplot[dashed,red] table[x=runtime,y=freq]
      {data/ea.oliveto-he-yao-8-0.25-ecdf.dat};     
\end{semilogxaxis}
\end{tikzpicture}
\caption{$n=64, m=127, k^{*}=24$}
  \end{subfigure}\hfill
\begin{subfigure}[b]{0.3\linewidth}
    \begin{tikzpicture}
    \begin{semilogxaxis}[%
      clip=false,
      xlabel={run time},
      xlabel near ticks,
    % jump mark left,
    width=5.8cm,
    enlargelimits=false,
    ymin=0,ymax=1,    
    every axis plot/.style={very thick,no marks}]   
      \addplot table[x=runtime,y=freq]
      {data/eajar.oliveto-he-yao-9-0.25-ecdf.dat};
      \addplot[dashed,red] table[x=runtime,y=freq] {data/ea.oliveto-he-yao-9-0.25-ecdf.dat};
\end{semilogxaxis}
\end{tikzpicture}
\caption{$n=81, m=170, k^{*}=27$}
  \end{subfigure}\hfill
\begin{subfigure}[b]{0.3\linewidth}
    \begin{tikzpicture}
    \begin{semilogxaxis}[%
    clip=false,
    % jump mark left,
    width=5.8cm,
    xlabel={run time},
    xlabel near ticks,
    enlargelimits=false,
    ymin=0,ymax=1,    
    every axis plot/.style={very thick,no marks}]   
      \addplot table[x=runtime,y=freq]
      {data/eajar.oliveto-he-yao-10-0.25-ecdf.dat};
      \addplot[dashed,red] table[x=runtime,y=freq] {data/ea.oliveto-he-yao-10-0.25-ecdf.dat};
\end{semilogxaxis}
\end{tikzpicture}
\caption{$n=100, m=219, k^{*}=30$}
  \end{subfigure}

  \caption{\label{fig:ohy-ecdf-eps0.25}Empirical cumulative distribution functions for running time
  of the (1+1)~EA and the \ea{k} on Oliveto-He-Yao instances
  ($\epsilon=1/4$). The optimal cover size is denoted by $k^{*}$. The $y$-axes denote the proportion of instances solved by
a particular run time.}
\end{figure}
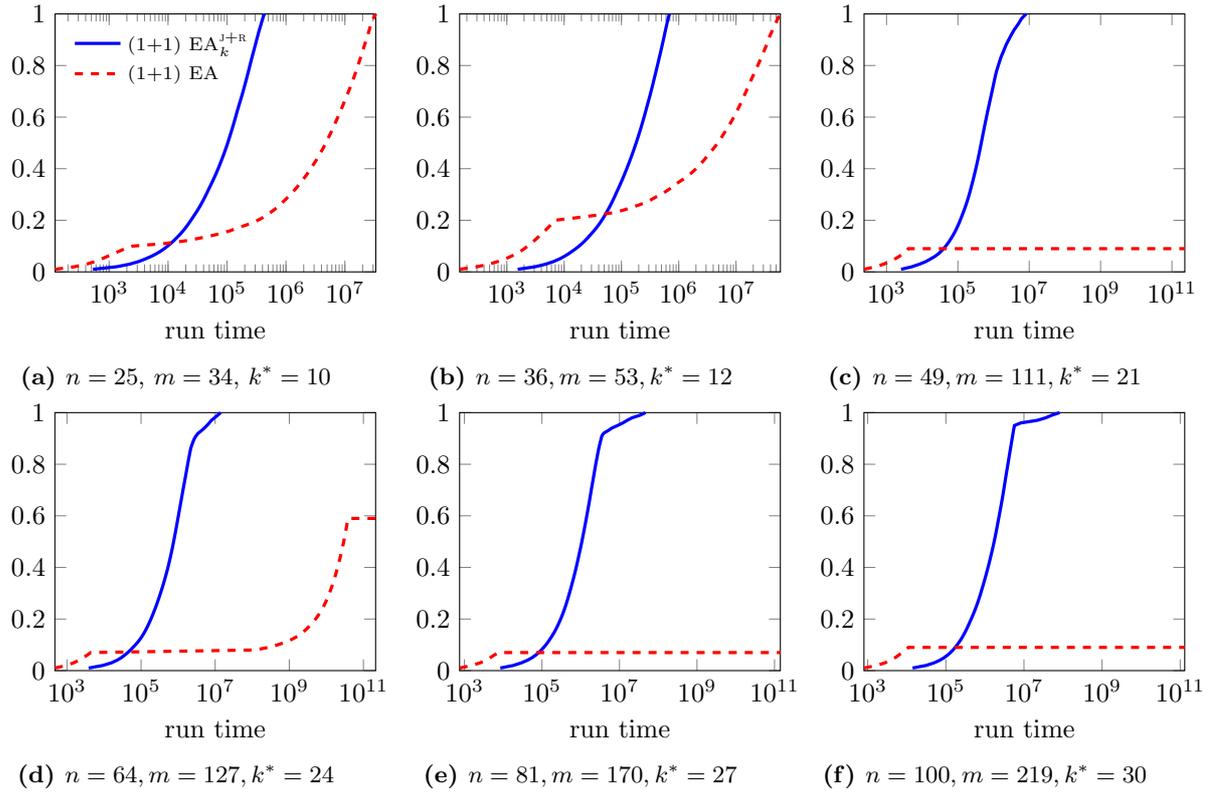

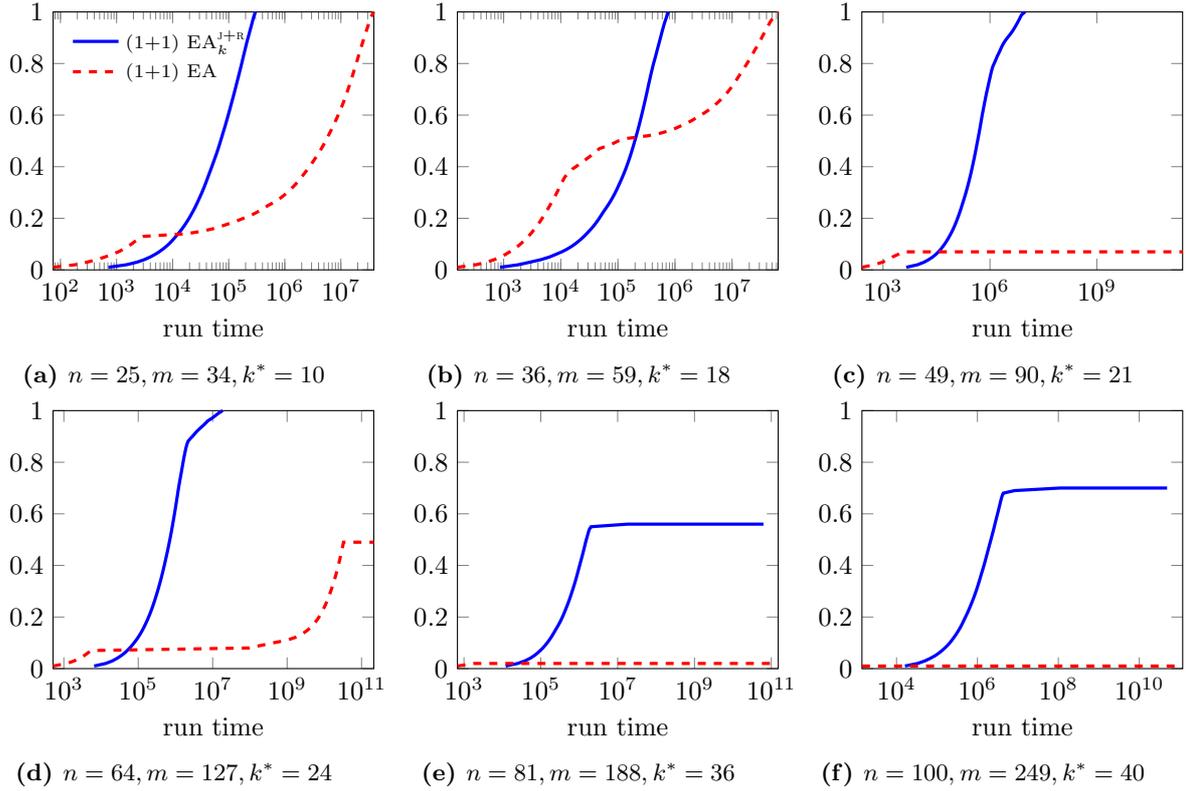
\begin{figure}
\begin{subfigure}[b]{0.3\linewidth}
    \begin{tikzpicture}
    \begin{semilogxaxis}[%
      clip=false,
      xlabel={run time},
      xlabel near ticks,
      legend pos=north west,
      legend cell align={left},
      legend style={draw=none,fill=none,font=\scriptsize},
    % jump mark left,
    width=5.8cm,
    enlargelimits=false,
    ymin=0,ymax=1,    
    every axis plot/.style={very thick,no marks}]   
      \addplot table[x=runtime,y=freq]
      {data/eajar.oliveto-he-yao-5-0.3333-ecdf.dat};
      \addplot[dashed,red] table[x=runtime,y=freq]
      {data/ea.oliveto-he-yao-5-0.3333-ecdf.dat};
      \legend{\ea{k},(1+1)~EA}
\end{semilogxaxis}
\end{tikzpicture}
\caption{$n=25, m=34, k^{*}=10$}
  \end{subfigure}\hfill
\begin{subfigure}[b]{0.3\linewidth}
    \begin{tikzpicture}
    \begin{semilogxaxis}[%
    clip=false,
    % jump mark left,
    width=5.8cm,
    xlabel={run time},
    xlabel near ticks,
    enlargelimits=false,
    ymin=0,ymax=1,    
    every axis plot/.style={very thick,no marks}]   
      \addplot table[x=runtime,y=freq]
      {data/eajar.oliveto-he-yao-6-0.3333-ecdf.dat};
      \addplot[dashed,red] table[x=runtime,y=freq] {data/ea.oliveto-he-yao-6-0.3333-ecdf.dat};
\end{semilogxaxis}
\end{tikzpicture}
\caption{\label{fig:ohy-ecdf-n36-eps0.333}$n=36, m=59, k^{*}=18$}
  \end{subfigure}\hfill
\begin{subfigure}[b]{0.3\linewidth}
    \begin{tikzpicture}
    \begin{semilogxaxis}[%
      clip=false,
      xlabel={run time},
      xlabel near ticks,
    % jump mark left,
    enlargelimits=false,
    width=5.8cm,    
    ymin=0,ymax=1,    
    every axis plot/.style={very thick,no marks}]   
      \addplot table[x=runtime,y=freq]
      {data/eajar.oliveto-he-yao-7-0.3333-ecdf.dat};
      \addplot[dashed,red] table[x=runtime,y=freq]
      {data/ea.oliveto-he-yao-7-0.3333-ecdf.dat};     
\end{semilogxaxis}
\end{tikzpicture}
\caption{$n=49, m=90, k^{*}=21$}
  \end{subfigure}
\begin{subfigure}[b]{0.3\linewidth}
    \begin{tikzpicture}
    \begin{semilogxaxis}[%
      clip=false,
      xlabel={run time},
      xlabel near ticks,
    % jump mark left,
    width=5.8cm,
    enlargelimits=false,
    ymin=0,ymax=1,    
    every axis plot/.style={very thick,no marks}]   
      \addplot table[x=runtime,y=freq]
      {data/eajar.oliveto-he-yao-8-0.3333-ecdf.dat};
      \addplot[dashed,red] table[x=runtime,y=freq] {data/ea.oliveto-he-yao-8-0.3333-ecdf.dat};
\end{semilogxaxis}
\end{tikzpicture}
\caption{$n=64, m=127, k^{*}=24$}
  \end{subfigure}\hfill
\begin{subfigure}[b]{0.3\linewidth}
    \begin{tikzpicture}
    \begin{semilogxaxis}[%
    clip=false,
    % jump mark left,
    width=5.8cm,
    xlabel={run time},
    xlabel near ticks,
    enlargelimits=false,
    ymin=0,ymax=1,    
    every axis plot/.style={very thick,no marks}]   
      \addplot table[x=runtime,y=freq]
      {data/eajar.oliveto-he-yao-9-0.3333-ecdf.dat};
      \addplot[dashed,red] table[x=runtime,y=freq] {data/ea.oliveto-he-yao-9-0.3333-ecdf.dat};
\end{semilogxaxis}
\end{tikzpicture}
\caption{$n=81, m=188, k^{*}=36$}
  \end{subfigure}\hfill
\begin{subfigure}[b]{0.3\linewidth}
    \begin{tikzpicture}
    \begin{semilogxaxis}[%
    clip=false,
    % jump mark left,
    width=5.8cm,
    xlabel={run time},
    xlabel near ticks,
    enlargelimits=false,
    ymin=0,ymax=1,    
    every axis plot/.style={very thick,no marks}]   
      \addplot table[x=runtime,y=freq]
      {data/eajar.oliveto-he-yao-10-0.3333-ecdf.dat};
      \addplot[dashed,red] table[x=runtime,y=freq] {data/ea.oliveto-he-yao-10-0.3333-ecdf.dat};
\end{semilogxaxis}
\end{tikzpicture}
\caption{$n=100, m=249, k^{*}=40$}
  \end{subfigure}

  \caption{\label{fig:ohy-ecdf-eps0.3333}Empirical cumulative distribution functions for running time
  of the (1+1)~EA and the \ea{k} on Oliveto-He-Yao instances
  ($\epsilon=1/3$). The optimal cover size is denoted by $k^{*}$. The $y$-axes denote the proportion of instances solved by
a particular run time.}
\end{figure}

\section{Conclusion}
\label{sec:conclusion}

In this paper, we have presented a variant of the (1+1)~EA that employs a focused jump-and-repair operation for solving parameterized variants of graph problems that are closed under induced subgraphs in which we must find a feasible vertex set of size $k$. Rather than searching the space of vertex sets, we search in parallel for both a vertex set along with an induced subgraph for which that vertex set is feasible. Offspring with vertex sets that violate the cardinality constraint have the opportunity to be probabilistically repaired by the focused jump-and-repair step. This step can be successful when the result happens to be both cardinality feasible and solution feasible for the induced subgraph. We prove that this approach, the \ea{k}, is an FPT Monte-Carlo algorithm for $k$-\textsc{VertexCover} the $k$-\textsc{FVST}, and $k$-\textsc{OddCycleTranversal}. Moreover, we show that a simple restarting framework for the \ea{k} solves the minimum \textsc{VertexCover} problem in time $O(2^{OPT}n^2 \log n)$ where $OPT$ is the size of the optimal vertex cover. Ignoring polynomial factors, this upper bound is smaller by an exponential factor in $OPT$ than the best-known bound for FPT evolutionary algorithms on minimum \textsc{VertexCover}. 

One drawback to our proposed approach is that the jump-and-repair operation must be problem-tailored. In the case of the $k$-\textsc{VertexCover} problem, this repair operation is rather natural, as we are simply ensuring that the edges uncovered by removing elements from the current vertex cover would be covered again by neighborhood of the deleted elements. For $k$-\textsc{FVST}, and $k$-\textsc{OddCycleTransversal} the repair operation is somewhat more involved, as it requires a subroutine for solving the \textsc{LongestCommonSubsequence} problem or computing a maximum flow. However, we point out that the repair operation is the only problem-specific component needed by the \ea{k} apart from the fitness function, and we are confident that the general framework we have presented will make it easier for designing repair operators on other parameterized problems.

Several directions for future work remain open. So far, lower bounds
are missing on FPT evolutionary algorithms for
\textsc{VertexCover}. This makes it difficult to directly compare
algorithms, and it is crucial to establish lower bounds in order to
understand the complete picture. Further experiments that compare
empirical run times for both traditional EAs and Global SEMO would also be valuable.
Furthermore, benchmark instances that are harder to solve than the ostensibly easy random planted class would be valuable.
Since the proposed approach only requires the construction of a procedure that repairs infeasible solutions, it could be extended to other problems for which an effective repair operator can be designed. Finally, it may be relevant to determine to what extent a crossover operator could effectively simulate the jump operation.

\section*{Acknowledgements}
The authors acknowledge the Minnesota Supercomputing Institute at the University of Minnesota for providing resources that contributed to the research results reported within this paper. \url{http://www.msi.umn.edu}

\bibliographystyle{plain}
\bibliography{references}

\end{document}